\documentclass[11pt]{article}
\usepackage[left=1in,top=1in,right=1in,bottom=1in,letterpaper]{geometry}
\input{macro}

\title{What Exactly Does Guidance Do in Masked Discrete Diffusion Models}
\date{\vspace{-5ex}}

\usepackage{times}
\author{
 Ye He \thanks{Georgia Institute of Technology. \texttt{yhe367@gatech.edu}}
 \and
 Kevin Rojas \thanks{Georgia Institute of Technology. \texttt{kevin.rojas@gatech.edu}}
  \and
 Molei Tao \thanks{Georgia Institute of Technology. \texttt{mtao@gatech.edu}}
}
\begin{document}

\maketitle

\begin{abstract}
   We study masked discrete diffusion models with classifier-free guidance (CFG). Assuming no score error nor discretization error, we derive an explicit solution to the guided reverse dynamics, so that how guidance influences the sampling behavior can be precisely characterized. When the full data distribution is a mixture over classes and the goal is to sample from a specific class, guidance amplifies class-specific regions while suppresses regions shared with other classes. This effect depends on the guidance strength $w$ and induces distinct covariance structures in the sampled distribution. Notably, we observe quantitatively different behaviors in $1$D and $2$D.  We also show that for large $w$, the decay rate of the total variation ($\TV$) along the reverse dynamics is double-exponential in $w$ for both $1$D and $2$D. These findings highlight the role of guidance, not just in shaping the output distribution, but also in controlling the dynamics of the sampling trajectory. Our theoretical analysis is supported by experiments that illustrate the geometric effects of guidance and its impact on convergence.
\end{abstract}

\section{Introduction}\label{sec:intro}
Diffusion models have become an influential tool for generative modeling, offering a flexible framework that performs well across a range of data types including images, audio, and text~\citep{dhariwal2021diffusion,kongdiffwave,li2022diffusion,ho2022video}. Originally formulated in continuous state spaces~\citep{ho2020denoising,songscore}, these models simulate a forward noising process—typically modeled by a stochastic differential equation and learn a reverse process to denoise and reconstruct the original data. More recently, the diffusion framework has been extended to discrete state spaces~\citep{campbell2022continuous,lou2023discrete}, where the forward process is defined via a continuous-time Markov chain over a finite state space. This has enabled generative modeling for discrete domains such as language modeling, molecule generation, and protein design~\citep{lou2023discrete,nie2025large,huang2023conditional,gruver2023protein}.

A key innovation that has enhanced the performance and flexibility of diffusion models is guidance, which introduces an auxiliary parameter to steer the reverse process toward desired outputs. In the continuous setting, classifier guidance~\citep{dhariwal2021diffusion} and classifier-free guidance~\citep{ho2021classifier,nichol2022glide} are widely used for conditional generation based on class labels or text prompts, significantly improving sample quality and alignment with conditioning signals. This technique has been critical to the success of models such as GLIDE~\citep{nichol2022glide} and Imagen~\citep{saharia2022photorealistic}. Theoretical analyses of guided diffusion models in continuous state spaces have examined how guidance modifies the reverse dynamics, most of which focus on simple settings such as low-dimensional and mixture of Gaussian models~\citep{bradley2024classifier,wu2024theoretical,chidambaram2024does}.

Classifier-free guidance (CFG) has also been recently introduced to discrete diffusion models, for applications such as text generation and controlled molecule design~\citep{huang2023conditional,nisonoff2024unlocking}. On the surface, the methodology appears to be very similar to the continuous diffusion case, but are guidance mechanism in continuous and discrete cases really similar? A closer look can actually reveal inherent differences: the lack of gradients and smooth geometry requires alternative strategies such as modifying transition probabilities or reweighting proposal distributions~\citep{nisonoff2024unlocking,sahoo2024simple}. While empirical results demonstrate that guidance improves sample quality and controllability~\citep{sahoo2024simple,xiong2025guide}, the theoretical understanding of how guidance affects the dynamics of the diffusion process in discrete state spaces remains limited.

In this paper, we provide a rigorous and quantitative framework for analyzing the effects of CFG on the discrete diffusion generative process~\citep{nisonoff2024unlocking}. We focus on masked discrete diffusion models~\citep{campbell2022continuous,shi2024simplified,sahoo2024simple,ou2024your}—a common subclass of discrete diffusion models. Assuming there is no errors from the score approximation and the numerical integration, we address the following two questions in the low-dimensional setting ($1$D and $2$D).
\begin{question}
    \centering
    How does guidance affect the distribution of the generated samples?
\end{question}
Towards this question, we derive explicit formulas for the sampled distributions. Assuming the full distribution is a mixture of different class distributions, i.e., the data distribution $p$ satisfies Assumption \ref{assup:full distribution}, the sampled distribution for a single class amplifies the probability mass in the area that is only supported in that class, and reduces the probability mass in the area that overlaps with other classes, eventually to zero as guidance strength $w$ goes to $\infty$. The strength of the amplification/reduction depends on $w$. The covariance structure of the sampled distribution varies for $1$D and $2$D.  
\begin{question}
    \centering
    How does guidance affect the rate of convergence of the reverse sampling dynamics?
\end{question}
To answer this question, we quantify $\TV$ between the distribution along the reverse sampling dynamics and the sampled distribution. Our results reflect that for both $1$D and $2$D, the decay rates of $\TV$ along the reverse sampling dynamics exhibit a double-exponential dependency on the guidance strength $w$ for $w\gg 1$. 

By characterizing the above influence of guidance on sampling trajectories, distributional shifts and convergence rates, our work bridges the gap between practical heuristics and theoretical understanding in guided discrete diffusion generation. 
\begin{assumption}\label{assup:full distribution} Let $\{z_k\}_{k=1}^M$ be the set of $M$ labels, each of which is associated with a class distribution $p(\cdot|z_k)$ supported on $\mc{X}_k\subsetneq S$. The full data distribution $p$ is a mixture of distributions $\{p(\cdot|z_k)\}_{k=1}^M$ with weights $\{a_k\}_{k=1}^M$, i.e., $p(\cdot)=\sum_{k=1}^M a_k p(\cdot|z_k)$. 
\end{assumption}
\noindent\textbf{Paper Organization.}
The remainder of the paper is organized as follows. Section~\ref{sec:prelim} introduces preliminaries on diffusion models relevant to our analysis. Section~\ref{sec:without guidance} quantifies the density evolution in masked discrete diffusion without guidance, which serves as a foundation for the guided case. Section~\ref{sec:with guidance} presents our theoretical analysis of the guided diffusion process, and Section~\ref{sec:experiment} provides numerical examples supporting our findings. Conclusions are discussed in Section~\ref{sec:conclusion}, and additional related work is presented in Appendix~\ref{append:related work}.
\vspace{-0.1in}
\section{Preliminaries}\label{sec:prelim}
\vspace{-.05in}
\subsection{Notations}
\vspace{-0.1in}
In this paper, for any $x\in \mb{R}^D$ and $A\subset \{1,2,\cdots,D\}$, we use $x_A\in \mb{R}^{|A|}$ to denote the vector by preserving dimensions whose indices are in $A$. $\setminus i$ is used to denote $\{1,2,\cdots,N\}\setminus \{i\}$. 
For any distribution $p$, $p(x_A)$ denotes the $A$-marginal density evaluated at $x_A$. For functions $f,g$, we use $f(w)\sim g(w)$ to represent $\lim_{w\to\infty} {f(w)}/{g(w)}=1$ and $f(w)=\Theta(g(w))$ to indicate that $c_1g(w)\le f(w)\le c_2g(w)$ for some $c_1,c_2,w_0>0$ and all $w>w_0$. 
\vspace{-0.05in}
\subsection{Discrete Diffusion Models}
\vspace{-0.1in}
We consider the probability state space $S=\{1,2,\cdots, N\}^D$. The data distribution $p$ is represented as a vector in $\mb{R}^{N^D}$ that sums up to $1$. The discrete diffusion process is defined as a continuous-time Markov process~\citep{campbell2022continuous,lou2023discrete}, given by the differential equation
\begin{align}\label{eq:forward process}
    \frac{\dee p_t}{\dee t} = Q_t p_t,\quad p_0= p,
\end{align}
where $Q_t\in \mb{R}^{N^D\times N^D}$ are the transition rate matrices for all $t\ge 0$ s.t. (1) $Q(y,x)\ge 0$ for all $x,y\in S$ and $x\neq y$; (2) $\sum_{y\in S}Q_t(y,x)=0$ for all $x\in S$. In this paper, we focus on a widely used effective forward process, the absorbing forward process~\citep{austin2021structured,lou2023discrete,shi2024simplified,ou2024your} 
, which independently transforms all the states to the masked state across different dimensions. The explicit expression of the transition rate matrices and their properties will be discussed in Section \ref{sec:without guidance}. The process \eqref{eq:forward process} has a reverse process defined by
\begin{align}\label{eq:reverse process}
    \frac{\dee q_t}{\dee t} = \Bar{Q}_{T-t}q_t,\quad q_0=p_T,
\end{align}
where the $\{\bQ_t\}_{0\le t\le T}$ is a sequence of reverse transition rate matrices given by 
\begin{align}\label{eq:reverse diffusion matrix}
    \bQ_t(y,x)=\left\{
    \begin{aligned}
       & \frac{p_t(y)}{p_t(x)} Q_t(x,y),\quad & y\neq x ,\\
       & -\sum_{ s\neq x} \bQ_t(s,x), & y=x.
    \end{aligned}
    \right.
\end{align}
It is well-known that \eqref{eq:reverse process} is the reverse of \eqref{eq:forward process}, i.e., $q_t=p_{T-t}$ for all $t\in [0,T]$. The ratios $\tfrac{p_t(y)}{p_t(x)}$ are known as the concrete scores~\citep{meng2022concrete} which generalize the typical score function $\nabla\log p_t(x)$ in continuous diffusion models. If the concrete scores $\tfrac{p_t(y)}{p_t(x)}$ are learned efficiently, we can generate samples from the data distribution $p$ by simulating the reverse processes \eqref{eq:reverse process}. In practice, people usually learn the concrete score via denoising entropy matching~\citep{lou2023discrete}: minimizing the following denoising score entropy:
\begin{align}\label{eq:denosiing score entropy}
    \mc{L}_{\mathrm{DSE}} = \mb{E}_{x_0\sim p} \mb{E}_{x\sim p_{t|0}(\cdot|x_0)} \big[ \sum_{y\neq x} s_t^\theta(x,y)- \frac{p_{t|0}(y|x_0)}{p_{t|0}(x|x_0)}\log s_t^\theta(x,y) \big],
\end{align}
where $s_t^\theta(x,y)$ is the parametrized score to approximate $\tfrac{p_t(y)}{p_t(x)}$. Last, new samples from the data distribution are generated by simulating the following reverse process: 
\begin{align}\label{eq:approximate reverse process}
    \frac{\dee q_t^\theta}{\dee t} = \Bar{Q}_{T-t}^\theta q_t^\theta,\quad q_0^\theta=\delta_{[M]},
\end{align}
where $\bQ^\theta_t$ is obtained from $\bQ_t$ by replacing $\tfrac{p_t(y)}{p_t(x)}$ with $s_t^\theta(x,y)$. The initial condition is a point mass at the masked state $[M]\coloneqq (N,\cdots, N)^\intercal$. Various numerical methods can be used to simulate \eqref{eq:approximate reverse process}, such as the Gillespie’s Algorithm~\citep{gillespie1976general}, Tau-leaping~\citep{gillespie2001approximate,campbell2022continuous} 
and uniformization~\citep{grassmann1977transient,chen2024convergence}, etc. Throughout the rest of the paper, we assume that there is no score approximation error and numerical error. We focus on the generation ability along the continuous-time reverse sampling dynamics. To understand the effect of score approximation and numerical schemes on the generation ability will be left as future work.

\subsection{Discrete Diffusion Models with CFG} 

To generate high quality samples conditioned on a specific label class $z$, \citet{nisonoff2024unlocking} introduced discrete diffusion process with CFG. One way to understand CFG intuitively is to think of sampling from a distribution, $p^{z,w}$, that is the full data distribution tilted by the conditional likelihood
\begin{align}\label{eq:tilt}
    p^{z,w}(\cdot) \propto p(\cdot) p(z|\cdot)^{1+w} \propto p(\cdot)^{-w} p(\cdot|z)^{1+w},
\end{align}
where the guidance parameter $w\ge -1$ and the second equation follows from the Bayesian rule. When $w=-1$, the tilted distribution recovers the full data distribution. When $w=0$, the tilted distribution becomes the conditional distribution on class $z$. By varying $w$, we change the emphasize of the likelihood, hence adjust the quality and diversity of the generated samples. To implement this idea in discrete diffusion models, \citet{nisonoff2024unlocking} proposed to tilt the reverse transition rate matrix $\Bar{Q}_t$ in \eqref{eq:reverse diffusion matrix} accordingly. First, define another forward process that evolves the conditional distribution $p(\cdot|z)$ with the same transition rate matrix $Q_t$ as used in \eqref{eq:forward process}: 
\begin{align}\label{eq:conditional forward process}
    \frac{\dee p_t(\cdot|z)}{\dee t} = Q_t p_t(\cdot|z),\quad p_0= p(\cdot|z).
\end{align}
Since the reverse transition rate matrices depend on both the forward transition rate matrices and the distributions along the forward process, the associated reverse transition rate matrices to \eqref{eq:conditional forward process}, denoted as $ \bQ^{z}_t$,  are different from those defined in \eqref{eq:reverse diffusion matrix}. $ \bQ^{z}_t$ is given by
\begin{align}\label{eq:conditional reverse diffusion matrix}
     \bQ^{z}_t(y,x)=\left\{
    \begin{aligned}
       & \frac{p_t(y|z)}{p_t(x|z)} Q_t(x,y),\quad & y\neq x ,\\
       & -\sum_{ s\neq x} \bQ^z_t(s,x), & y=x.
    \end{aligned}
    \right.
\end{align}
Using the tilting strategy in \eqref{eq:tilt}, the CFG reverse discrete process is given by
\begin{align}\label{eq:guided reverse process}
    \frac{\dee q^{z,w}_t}{\dee t} = \hat{Q}^{z,w}_{T-t} q^{z,w}_t,\quad q_0^{z,w}= \delta_{[M]},
\end{align}
where the initial condition is a point mass at the masked state $[M]\coloneqq (N,\cdots, N)^\intercal$. The reverse transition rate matrix is defined as 
\begin{align}\label{eq:guided reverse diffusion matrix}
     \hat{Q}^{z,w}_t(y,x)=\left\{
    \begin{aligned}
       & \bQ_t(y,x)^{-w} \bQ^z_t(y,x)^{1+w},\quad & y\neq x ,\\
       & -\sum_{ s\neq x} \hat{Q}^{z,w}_t(s,x), & y=x.
    \end{aligned}
    \right.
\end{align}
When $w=-1$, the CFG rate matrix $\hat{Q}_t^{z,w}$ is the unguided rate matrix $\bQ_t$ in \eqref{eq:reverse diffusion matrix}. When $w=0$, the CFG rate matrix $\hat{Q}_t^{z,w}$ is nothing but the conditional rate matrix $\bQ_t^z$ in \eqref{eq:conditional reverse diffusion matrix}.

\section{Analysis of Masked Discrete Diffusion Models without Guidance}\label{sec:without guidance} 
This section analyzes the behavior of masked discrete diffusion in the absence of guidance. By quantifying the density evolution of the sampling process, we establish a baseline understanding of the unguided dynamics. These results provide essential groundwork for the theoretical analysis of discrete diffusion with CFG in the Section \ref{sec:with guidance}.
\subsection{Density evolution along the forward process} The forward process in the masked discrete diffusion process gradually absorbs all the mass to the masked state $N$. In practice~\citep{campbell2022continuous,lou2023discrete}, the forward transition rate matrix is parametrized by $Q_t=\sigma(t) \big(\sum_{d=1}^D I_N \otimes \cdots \underbrace{ Q }_{d^{th}}\cdots \otimes I_N  \big) $ where 
\begin{align}\label{eq:absorbing transition rate matrix}
    Q = \begin{pmatrix}
        -1  & \cdots & 0 & 0 \\
        \vdots & \ddots & \vdots & \vdots  \\
        0  & \cdots & -1 & 0 \\
        1 &  \cdots &  1 & 0
    \end{pmatrix}_{N\times N}
\end{align} 
For simplicity, we consider $\sigma(t)\equiv 1$ in the paper. According \eqref{eq:absorbing transition rate matrix}, we express the densities along the forward process in the following proposition whose proof is deferred to Appendix \ref{append:forward process}. 
\begin{proposition}\label{prop:forward density evolution} Let $\mu_t$ be the solution to $\tfrac{\dee}{\dee t}\mu_t = Q_t \mu_t$ with initial distribution $\mu_0=\mu$ and $Q_t$ given above. Then
    \begin{align}\label{eq:forward density absorbing}
        \mu_t & =\begin{pmatrix}
        e^{-t} & 0 & \cdots & 0 & 0 \\
        0 & e^{-t} & \cdots & 0 & 0 \\
        \vdots & \vdots & \ddots & \vdots & \vdots  \\
        0 & 0 & \cdots & e^{-t} & 0 \\
        1-e^{-t} & 1-e^{-t} & \cdots &  1-e^{-t} & 1
    \end{pmatrix}_{N\times N}^{\otimes D} \mu\coloneqq A_t^{\otimes D} \mu.
    \end{align}
    As a consequence, for any $x\in S$ with $\UM = \UM(x) \coloneqq |\{i: x_i <N\}|$,
    \begin{align*}
        \mu_t(x) = e^{ -|\UM| t} (1-e^{-t})^{D- |\UM|} \sum_{y: y_{\UM}=x_{\UM}} \mu(y).
    \end{align*}
\end{proposition}

\subsection{Density evolution along the reverse process} For any distribution $\mu$ on $S$, the forward process \eqref{eq:forward process} initiated at $\mu$ induces a reverse process, whose transition rate matrix is denoted as $\Bar{Q}_t[\mu]$. An explicit expression of $\Bar{Q}_t[\mu]$ can be derived from the property in \eqref{eq:reverse diffusion matrix} and the forward densities in Proposition \ref{prop:forward density evolution}.

\begin{proposition}\label{prop:reverse transition rate} The sequence of reverse transition rate matrices associated with $\mu$ (initialization of the forward process) satisfies that for all $0\le t\le T$,
    \begin{align*}
        \Bar{Q}_t[\mu] (y,x) & = \left\{
        \begin{aligned}
            & \frac{e^{-t}}{1-e^{-t}} \frac{\sum_{u:u_{\UM}=y_{\UM}} \mu(u) }{\sum_{u:u_{\UM}=x_{\UM}} \mu(u)}      ,\quad & x_i\neq N =y_i, x_{\setminus i}=y_{\setminus i} , \\
            & - \sum_{u\in \mc{N}(x)} \Bar{Q}_t[\mu] (u,x), \quad & y=x,\\
            & 0 , \quad &\text{otherwise}.
        \end{aligned}
        \right. 
    \end{align*}
\end{proposition}
With the above expression of the reverse transition rate matrix, in the low-dimensional setting of $D=1$, we derive the density formulas along the reverse sampling dynamics in the following theorem, with the proof deferred to Appendix \ref{append:forward process}. 
\begin{theorem}\label{thm:1d reverse no guidance} In the discrete diffusion models on $S$, if $D=1$ and $q_{t}$ satisfies the sampling dynamics $\tfrac{\dee}{\dee t}q_t = \Bar{Q}_{T-t}[p] q_t$ with initial condition $q_0=\delta_{N}$, we have that for all $0\le t\le T$,
    \begin{align}\label{eq:1d reverse density}
        q_t(x) =\left\{
    \begin{aligned}
        & \big( 1- \frac{1-e^{-(T-t)}}{1-e^{-T}}\big) p(x) , \quad & x=1,2,\cdots, N-1, \\
        & \frac{1-e^{-(T-t)}}{1-e^{-T}}, \quad & x=N.
    \end{aligned}
    \right.
    \end{align}
\end{theorem}
\begin{remark}[No initialization error]\label{rem:1d generation} Unlike other diffusion processes, the absorbing discrete diffusion does not induce any initialization error. Even though we approximate the initialization in \eqref{eq:reverse process} by the point mass at the masked state, the sampled distribution recovers the data distribution, i.e., $q_T=p$. Same property also holds for masked discrete diffusion with CFG, as shown in Section \ref{sec:with guidance}. 
\end{remark}

\section{Analysis of Masked Discrete Diffusion Models with CFG}\label{sec:with guidance}

In this section, we analyze the reverse sampling dynamics in the masked discrete diffusion model with CFG and provide quantitative understandings to the sampled distributions and the convergence rates in the generation process. Our analysis follows the approach in Section \ref{sec:without guidance}: representing the reverse transition rate matrix in \eqref{eq:guided reverse diffusion matrix}, then solving the reverse process \eqref{eq:guided reverse process}. WLOG, we consider the task of sampling from the label class $z_1$ (denoted as $z$ for simplicity). Due to the different behaviors of the sampling dynamics in low-dimensional settings, we state our results in two different settings: $D=1$ that corresponds to single-token generation and $D=2$ that illustrates multiple-token generation. 
\subsection{$D=1$: single-token generation}\label{sec:1d generation}

 For $D=1$, the reverse transition rate matrix defined in \eqref{eq:guided reverse diffusion matrix} is exactly the reverse transition rate matrix induced by the tilted distribution $p^{z,w}(\cdot)\propto p(\cdot)^{-w}p(\cdot|z)^{1+w}$: 
\begin{align*}
    \hat{Q}_t^{z,w}(y,x) = \left\{
    \begin{aligned}
        &\frac{e^{-t}}{1-e^{-t}} p(x)^{-w}p(x|z)^{1+w} , \quad & x=N\neq y \\
        & -\frac{e^{-t}}{1-e^{-t}}\sum_{x=1}^{N-1} p(x)^{-w}p(x|z)^{1+w} , \quad & x=y=N \\
        & 0  ,\quad &\text{otherwise},
    \end{aligned}
    \right.
\end{align*}
i.e.,  $\hat{Q}_t^{z,w}=\mc{Z}^{z,w} \Bar{Q}_t[p^{z,w}]$, where $\mc{Z}^{z,w}\coloneqq \sum_{x=1}^{N-1} p(x)^{-w}p(x|z)^{1+w}$ and $\Bar{Q}_t[p^{z,w}]$ was derived in Proposition \ref{prop:reverse transition rate}. As a consequence, the reverse sampling dynamics with CFG in \eqref{eq:guided reverse process} is adapted from the reverse dynamics of \eqref{eq:forward process}, by replacing the initial distribution by $p^{z,w}$ and scaling the velocity by the factor $\mc{Z}^{z,w}$. Similar to Theorem \ref{thm:1d reverse no guidance}, we derive the following theorem that provides explicit formulas for the densities along the reverse sampling dynamics with CFG.
\begin{theorem}\label{thm:1d reverse} In the discrete diffusion models on $S$ with CFG, if $D=1$ and $q_t^{z,w}$ satisfies the sampling dynamics \eqref{eq:guided reverse process}, we have that for all $0\le t\le T$,
    \begin{align}\label{eq:1d reverse density}
        q_t^{z,w}(x) =\left\{
    \begin{aligned}
        & \big( 1- \big(\frac{1-e^{-(T-t)}}{1-e^{-T}}\big)^{\mc{Z}} \big) p^{z,w}(x) , \quad & x=1,2,\cdots, N-1, \\
        & \big(\frac{1-e^{-(T-t)}}{1-e^{-T}}\big)^{\mc{Z}}, \quad & x=N,
    \end{aligned}
    \right.
    \end{align}
where $\mc{Z}= \mc{Z}^{z,w}=\sum_{x=1}^{N-1} p(x)^{-w}p(x|z)^{1+w}$.
\end{theorem}
The convergence rate of the reverse sampling dynamics and the properties of the sampled distributions can be instantly derived from the above expression. We state the results in the following two Propositions with the proofs deferred to Appendix~\ref{append:1d density}. 
\begin{proposition}\label{prop:convergence rate} Under the assumptions in Theorem \ref{thm:1d reverse}, we have that for all $0\le t\le T$ and $w>0$, $\TV(q_t^{z,w}, p^{z,w}) = \big(\tfrac{1-e^{-(T-t)}}{1-e^{-T}}\big)^{\mc{Z}}$ where $\mc{Z}= \mc{Z}^{z,w}=\sum_{x=1}^{N-1} p(x)^{-w}p(x|z)^{1+w}$.
\end{proposition}
\begin{remark}[Double exponential dependency on $w$]\label{rem:1d convergence rate} 
The $\TV$ exponentially decays along the sampling dynamics \eqref{eq:guided reverse process}. The exponential rate $\mc{Z}$ is the normalization constant appearing in the construction of the tilted distribution $p^{z,w}$. For all $w>0$, alternatively we can represent  $\mc{Z}=\exp(w\mc{D}_{1+w}(p(\cdot|z)|p))$, where $\D_\alpha(\mu_1|\mu_2) \coloneqq \frac{1}{\alpha-1}\log\big( \sum_x \tfrac{\mu_1(x)^\alpha}{\mu_2(x)^{\alpha-1}} \big) $ is the $\alpha$-divergence from $\mu_1$ to $\mu_2$ for all $\alpha\in(0,\infty)\setminus \{1\}$. According to the property of $\alpha$-divergence, we immediately get the following properties for $\mc{Z}$: (a) $\mc{Z}^{z,w}\ge 1$; (b) $w\mapsto \mc{Z}^{z,w}$ is monotone increasing; (c) $\log(\mc{Z}^{z,w}) \sim  w\sup_x \tfrac{p(x|z)}{p(x)}$ for $w\gg 1$. Therefore, for $w\gg 1$, the \textbf{exponential} decay rate of $\TV$ is \textbf{exponential} in $w$.
\end{remark}
 
\begin{proposition}\label{prop:1d sampled distribution property} Assume the full distribution satisfies Assumption \ref{assup:full distribution}, depending on the support of $p(\cdot|z_1)$, the sampled distribution $q_T^{z_1,w}$ admits the following different behaviors:
\begin{itemize}
    \item [(1)] if $\mc{X}_1\cap \mc{X}_k = \emptyset$  for all $k=2,\cdots,M$, the sampled distribution $q_T^{z_1,w}=p(\cdot|z_1)$ for all $w\ge 0$.
    \item [(2)] if $S_1\coloneqq \mc{X}_1 \cap \big( \cup_{k=2}^M \mc{X}_k \big)\neq \emptyset$ and $I_1\coloneqq \{ k: \mc{X}_k\cap \mc{X}_1\neq \emptyset \}$, we have
    \begin{align*}
        q_T^{z_1,w}(x)\propto \left\{
        \begin{aligned}
            &  p(x|z_1), \quad &  x\in \mc{X}_1\setminus S_1 \\
            & \big( \frac{a_1 p(x|z_1)}{\sum_{k\in I_1} a_k p(x|z_k) } \big)^{w} p(x|z_1) \quad & x\in  S_1 \\
            & 0 ,\quad & \text{otherwise}
        \end{aligned}
        \right.
    \end{align*}
    As a consequence, as $w\to\infty$, $q_T^{z_1,w}\to p_{\mc{X}_1 \setminus S_1}(\cdot|z_1)$ pointwisely. $p_{\mc{X}_1 \setminus S_1}(\cdot|z_1)$ is the restriction of $p(\cdot|z_1)$ to the set $\mc{X}_1 \setminus S_1$.
\end{itemize}    
\end{proposition}
\begin{remark}[Local mean/variance preservation]\label{rem:local variance preservation 1d}  Proposition \ref{prop:1d sampled distribution property} suggests that is obtained from the class-1 distribution $p(\cdot|z_1)$ by transforming the probability mass from the overlapping region ($S_1$) to the unique region of class $z_1$ ($\mc{X}_1\setminus S_1$). In particular, this transformation preserve the local mean and variance within the unique region. Please refer to Appendix~\ref{append:1d density} for the detailed argument. 
\end{remark}
\subsection{$D=2$: multiple-token generation} Unlike the case for $D=1$, the reverse transition rate matrix defined in \eqref{eq:guided reverse process} deviates from capturing the geometry of the tilted distribution $p^{z,w}$, i.e., $\hat{Q}^{z,w}_t \neq C \Bar{Q}_t[p^{z,w}] $ for any constant $C$ in general. The explicit expression for $\hat{Q}_t^{z,w}$, which is derived based on the construction of the guided reverse transition rate matrix in \eqref{eq:guided reverse diffusion matrix} and Proposition \ref{prop:forward density evolution}, is stated in the following Proposition \ref{prop:reverse rate matrix 2d}.
\begin{proposition}\label{prop:reverse rate matrix 2d} When $D=2$, denote $\mc{Z}=\mc{Z}^{z,w}=\sum_{x\in S} p(x)^{-w}p(x|z)^{1+w}$. Then the guided reverse transition rate matrix is given by $\hat{Q}_t^{z,w}=\tfrac{e^{-t}}{1-e^{-t}} \hat{Q}^{z,w}$ s.t.,
\begin{align*}
    \hat{Q}^{z,w}(y,x)=\left\{
    \begin{aligned}
        & \frac{\mc{Z} p^{z,w}(y)}{p(y_i)^{-w}p(y_i|z)^{1+w}} , \quad & x_i=y_i\neq N, x_{\setminus i}=N\neq y_{\setminus i} \\  
                 & p(y_i)^{-w}p(y_i|z)^{1+w} , \quad & x_i=N\neq y_i,x_{\setminus i}=y_{\setminus i}= N \\
        &  -\frac{\mc{Z}p^{z,w}(y_i)}{p(y_i)^{-w}p(y_i|z)^{1+w}} , \quad & x_i=y_i\neq N, x_{\setminus i}=y_{\setminus i}=N \\
        & -\sum_{l=1}^2\sum_{u_l=1}^{N-1}p(u_1)^{-w}p(u_l|z)^{1+w} , \quad & x=y=(N,N)\\
        & 0  ,\quad &\text{otherwise}.
    \end{aligned}
    \right.
\end{align*}
\end{proposition}
 As a consequence, the sampling dynamics in \eqref{eq:guided reverse process} no longer generates samples from the tilted distribution $p^{z,w}$. In fact, the convergence behavior for \eqref{eq:guided reverse process} is much more complicated than the one for $D=1$. In the rest of this section, we first explicitly express the solution of \eqref{eq:guided reverse process} for $D=2$. Then we interpret the results from the perspectives of sampled distributions and the convergence rates, highlighting the differences to those for $D=1$. All the proofs are included in Appendix \ref{append:2d density}. 

Before we state the main results, we define the following quantities: for all $x_1,x_2=1,2,\cdots, N-1$,
\begin{align}\label{eq:coefficient c d definition}
    & c_{x_1} \coloneqq  \frac{\sum_l p(x_1,l)^{-w}p(x_1,l|z)^{1+w}}{p(x_1)^{-w} p(x_1|z)^{1+w}}, \quad d_{x_2}=\frac{\sum_l p(l,x_2)^{-w}p(l,x_2|z)^{1+w}}{p(x_2)^{-w}p(x_2|z)^{1+w}},  \\
    & c_N \coloneqq \frac{\sum_{l_1,l_2} p(l_1,l_2)^{-w}p(l_1,l_2|z)^{1+w}}{\sum_{l_1}p(l_1)^{-w}p(l_1|z)^{1+w}} ,\quad d_N\coloneqq \frac{\sum_{l_1,l_2} p(l_1,l_2)^{-w}p(l_1,l_2|z)^{1+w}}{\sum_{l_2}p(l_2)^{-w}p(l_2|z)^{1+w}}. 
\end{align}
It is easy to see that for all $l=1,2\cdots, N$, $c_l\ge 1$ and $d_l\ge 1$, and $c_l=d_l=1$ when $D=1$. When $D=2$, $\{c_x,d_x\}_{x=1}^N$ encodes the information of $p, p(\cdot|z)$ and the guidance $w$ into the sampling dynamics \eqref{eq:guided reverse process}, and hence affects the sampled distributions and the convergence rates.   
\begin{theorem}\label{thm:2d reverse} In the discrete guided diffusion models on $S$, if $D=2$ and $q_t^{z,w}$ satisfies the sampling dynamics \eqref{eq:guided reverse process}, we have that for all $0\le t\le T$,
   {\small \begin{align}\label{eq:2d reverse density}
        q_t^{z,w}(x) =\left\{
    \begin{aligned}
        & \alpha_t(x) \mc{Z} p^{z,w}(x) , \quad & x_1,x_2\neq N, \\
        & \alpha_t(x) \mc{Z} p^{z,w}(x_i), \quad & x_i=N\neq x_{\setminus i},\\
        & \alpha_t(x)  ,\quad & x_1=x_2=N.
    \end{aligned}
    \right.
    \end{align}}
In \eqref{eq:2d reverse density}, $\mc{Z}= \mc{Z}^{z,w}=\sum_{x\in S} p(x)^{-w}p(x|z)^{1+w}$ and
\begin{align*}
  {\small  \alpha_t(x) = \left\{
    \begin{aligned}
        & -\frac{1}{c_{x_1}(\lambda_{NN}^{z,w}+c_{x_1})}\big( 1-(\frac{1-e^{-(T-t)}}{1-e^{-T}})^{c_{x_1}} \big)-\frac{1}{d_{x_2}(\lambda_{NN}^{z,w}+d_{x_2})}\big( 1-(\frac{1-e^{-(T-t)}}{1-e^{-T}})^{d_{x_2}} \big)  \\
    & \quad \quad-\frac{1}{\lambda_{NN}^{z,w}}\big( \frac{1}{\lambda_{NN}^{z,w}+c_{x_1}} + \frac{1}{\lambda_{NN}^{z,w}+d_{x_2}} \big) \big( 1-(\frac{1-e^{-(T-t)}}{1-e^{-T}})^{-\lambda_{NN}^{z,w}} \big) ,\quad\quad x_1,x_2\neq N\\
    & -\frac{1}{c_{x_1}(\lambda_{NN}^{z,w}+c_{x_1})} \big( (\frac{1-e^{-(T-t)}}{1-e^{-T}})^{c_{x_1}} -(\frac{1-e^{-(T-t)}}{1-e^{-T}})^{-\lambda_{NN}^{z,w}} \big),\quad\quad\quad\  x_1\neq N=x_2\\
    & -\frac{1}{d_{x_2}(\lambda_{NN}^{z,w}+d_{x_2})}  \big( (\frac{1-e^{-(T-t)}}{1-e^{-T}})^{d_{x_2}}- (\frac{1-e^{-(T-t)}}{1-e^{-T}})^{-\lambda_{NN}^{z,w}} \big),\quad\quad\quad\   x_2\neq N=x_1\\
    &(\frac{1-e^{-(T-t)}}{1-e^{-T}})^{-\lambda_{NN}^{z,w}} ,\qquad\qquad\qquad\qquad\qquad\qquad\qquad\qquad\qquad\qquad\quad  x_1=x_2=N,
    \end{aligned}
    \right.}
\end{align*}
where $ \lambda_{NN}^{z,w}\coloneqq -\mc{Z}(1/c_N+1/d_N)$.
\end{theorem}
\begin{remark}\label{rem:sampled distribution 2d}(Sampled distribution) Unlike the $1$D setting, the sampled distribution in $2$D is not the tilted distribution $p^{z,w}$. According to Theorem \ref{thm:2d reverse}, the sampled distribution $q_T^{z,w}$ is given by
    \begin{align}\label{eq:sampled distribution 2d}
        q_T^{z,w}(x) = \frac{1/c_{x_1}+1/d_{x_2}}{1/c_N+1/d_N} p^{z,w}(x),\quad \forall x\in \{1,2,\cdots, N-1\}^2.
    \end{align}
\end{remark}
\begin{proposition}\label{prop:convergence rate 2d} Under the assumptions in Theorem \ref{thm:2d reverse}, we have that for all $0\le t\le T$ and $w\gg 1$, $-\ln(\TV(q_t^{z,w}, q_T^{z,w}))=\exp(\Theta(w)) \ln \big(\tfrac{1-e^{-T}}{1-e^{-(T-t)}}\big) $.
\end{proposition} 
Similar to the $1$D setting, the \textbf{exponential} decay rate of $\TV$ in $2$D is also \textbf{exponential} in $w$ for $w\gg 1$. As we observed in Section \ref{sec:experiment}, these double-exponential dependency on $w$ may cause numerical scheme less stable because the reverse sampling dynamics has a significant sharper transition in time as $w$ increases. 
\begin{definition}\label{def:marginal support} For any distribution $\mu$ on $S=\{1,2,\cdots, N\}^D$ with support $\mc{X}$, its $d$-marginal support, denoted as $\mc{X}_d$, is defined as $\mc{X}_d\coloneqq \{ x_d : x\in \mc{X} \}$.  
\end{definition}
\begin{proposition}\label{prop:2d sampled distribution property} Assume the full distribution satisfies Assumption \ref{assup:full distribution}, depending on the marginal supports of $p(\cdot|z_1)$, the sampled distribution $q_T^{z_1,w}$ admits the following different behaviors:
\begin{itemize}
    \item [(1)] If $\mc{X}_{1,d}\cap \mc{X}_{k,d} =\emptyset$ for all $d=1,2$ and $k=2,\cdots M$, we have $q_T^{z_1,w}=p(\cdot|z_1)$.
    \item [(2)] If $S_{1,d}\coloneqq \mc{X}_{1,d} \cap \big( \cup_{k=2}^M \mc{X}_{k,d} \big)\neq \emptyset$ for some $d=1,2$. Let $S_{1}\coloneqq \mc{X}_{1} \cap \big( \cup_{k=2}^M \mc{X}_{k} \big)$, $I_1\coloneqq \{ k: \mc{X}_k\cap \mc{X}_1\neq \emptyset \}$ and $I_{1,d}\coloneqq \{ k: \mc{X}_{k,d}\cap \mc{X}_{1,d}\neq \emptyset \}$. We have
    {\small
    \begin{align*}
        q_T^{z_1,w}(x)\propto \left\{
        \begin{aligned}
            &  2 p(x|z_1), \qquad\qquad\qquad\qquad\qquad\qquad   x\in \mc{X}_1, x_1\in \mc{X}_{1,1}\setminus S_{1,1}, x_2\in \mc{X}_{1,2}\setminus S_{1,2} \\
            & \big( (\frac{a_1 p(x_i|z_1)}{\sum_{k\in I_{1,i}} a_k p(x_i|z_k) })^{w} +1 \big) p(x|z_1), \qquad  x\in \mc{X}_1, x_i\in S_{1,i}, x_{\setminus i}\in \mc{X}_{1,\setminus i}\setminus S_{1,\setminus i} \\
            & \big( \sum_{i=1}^2(\frac{a_1 p(x_i|z_1)}{\sum_{k\in I_{1,i}} a_k p(x_i|z_k) })^w\big) p(x|z_1), \qquad\qquad  x\in \mc{X}_1\setminus S_1, x_1\in S_{1,1}, x_2\in S_{1,2}  \\
             & \big( \sum_{i=1}^2(\frac{a_1 p(x_i|z_1)}{\sum_{k\in I_{1,i}} a_k p(x_i|z_k) } )^w\big) \big(\frac{a_1 p(x|z_1)}{\sum_{k\in I_{1}} a_k p(x|z_k)}\big)^wp(x|z_1), \qquad\qquad\qquad   x\in S_1  \\
            & 0 ,\qquad\qquad\qquad\qquad\qquad\qquad\qquad\qquad\qquad\qquad\qquad\qquad\qquad\qquad\quad \text{otherwise}.
        \end{aligned}
        \right.
    \end{align*}}
    As a consequence, as $w\to\infty$, $q_T^{z_1,w}\to q^{z_1,\infty}(\cdot|z_1)$ pointwisely. $q^{z_1,\infty}(\cdot|z_1)$ satisfies that $\mathrm{Supp}(q^{z_1,\infty}(\cdot|z_1))\subset \mc{X}_1\setminus S_1$.
\end{itemize}    
\end{proposition}
\begin{remark}[Effect of guidance on sampled distributions]\label{rem:guidance effect 2d} In $2$D, the sampled distribution $q_T^{z_1,w}$ adapts the conditional distribution $p(\cdot|z_1)$ by adjusting the weights of mass on different sets: (a) for $x$ in the set that is disjoint with other classes' supports and the marginals of the set are disjoint to the marginal supports of other classes, the sampled distribution $q_T^{z_1,w}$ admits the largest weight and preserve the local covariance in the set; (b) for $x$ in the set that is disjoint with other classes' supports and one of the set marginals overlaps with marginal supports of other classes, the sampled distribution $q_T^{z_1,w}$ put weights that depend on the conditional marginals as described in cases $2,3$ in Proposition \ref{prop:2d sampled distribution property}-(2); (c) for $x$ in the overlapping set with other classes' supports, the sampled distribution $q_T^{z_1,w}$ puts the smallest weights that also depends on the conditional marginals as shown in case $4$ in Proposition \ref{prop:2d sampled distribution property}-(2). 
\end{remark}
\begin{remark}[Discussion on $q^{z_1,\infty}(\cdot|z_1)$] \label{eq:asymp distribution w inf 2d} Under Assumption \ref{assup:full distribution}, $q^{z_1,\infty}(\cdot|z_1)$ has zero mass on overlapping region between class $z_1$ and other classes. In the non-overlapping region, explicit formula for $q^{z_1,\infty}(\cdot|z_1)$ can be derived from the expression of $q^{z_1,w}(\cdot|z_1)$ in Proposition \ref{prop:2d sampled distribution property}. However, it depends on the nullities of the regions in Proposition \ref{prop:2d sampled distribution property}-(2). We refer the readers to Appendix \ref{append:explicit expression asymp w 2d} for a detailed discussion.    
\end{remark}
\vspace{-.2in}
\section{Numerical Examples}\label{sec:experiment}
\vspace{-.1in}
In this section we present numerical results for better illustrating of our theoretical results. Unless otherwise stated we train our models using a small transformer and use Tau-leaping with $50$ steps as the numerical scheme and $10$K samples. Experiments are run on a NVIDIA GeForce RTX™ 4070 Laptop GPU.

\noindent\textbf{Experiments in $1$D.} We consider two setups for the one dimensional experiments: classes with and without intersections to demonstrate how guidance works differently in these cases. In Figure \ref{fig:1d-threeplots}-(a)(b), we can observe how adding a region of intersection dramatically affects the generation. We observe that even with score and discretization errors, our empirical sampled distribution closely approximate the tilted distribution in Proposition \ref{prop:1d sampled distribution property}. We also plot $\TV$ as a function of $w$ in Figure \ref{fig:1d-threeplots}-(c) for a fixed time $t=.5$, we observe that the empirical version closely follows our theory in Proposition \ref{prop:convergence rate} for small $w$. For large $w$, we observe a flat/increasing region in the plot. We conjecture that this is mainly due to the sharp transition of the reverse sampling dynamics for large $w$ (as shown in Remark \ref{rem:1d convergence rate}), which makes the Tau-leaping scheme less efficient and less stable. 

\begin{figure}[htbp]
    \centering
    \begin{subfigure}[t]{0.3\textwidth}    \includegraphics[width=\linewidth]{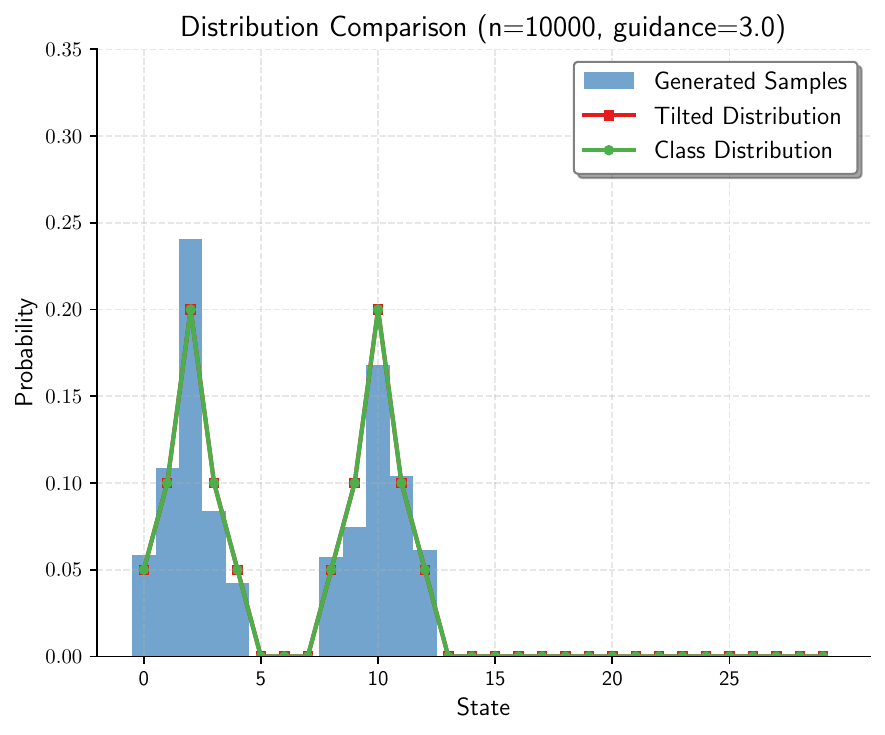}
        \caption{No effect in disjoint support.}
    \end{subfigure}
    \hfill
    \begin{subfigure}[t]{0.3\textwidth}
        \includegraphics[width=\linewidth]{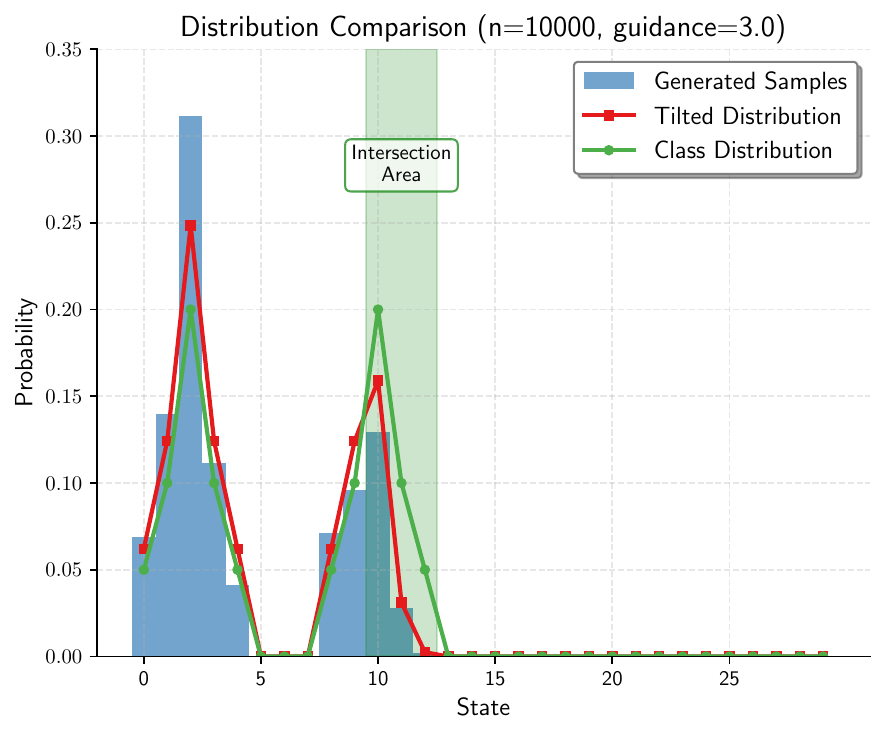}
        \caption{Mass is shifted away from the intersecting region.}
    \end{subfigure}
    \hfill
    \begin{subfigure}[t]{0.3\textwidth}
        \includegraphics[height=3.5cm,width=\linewidth]{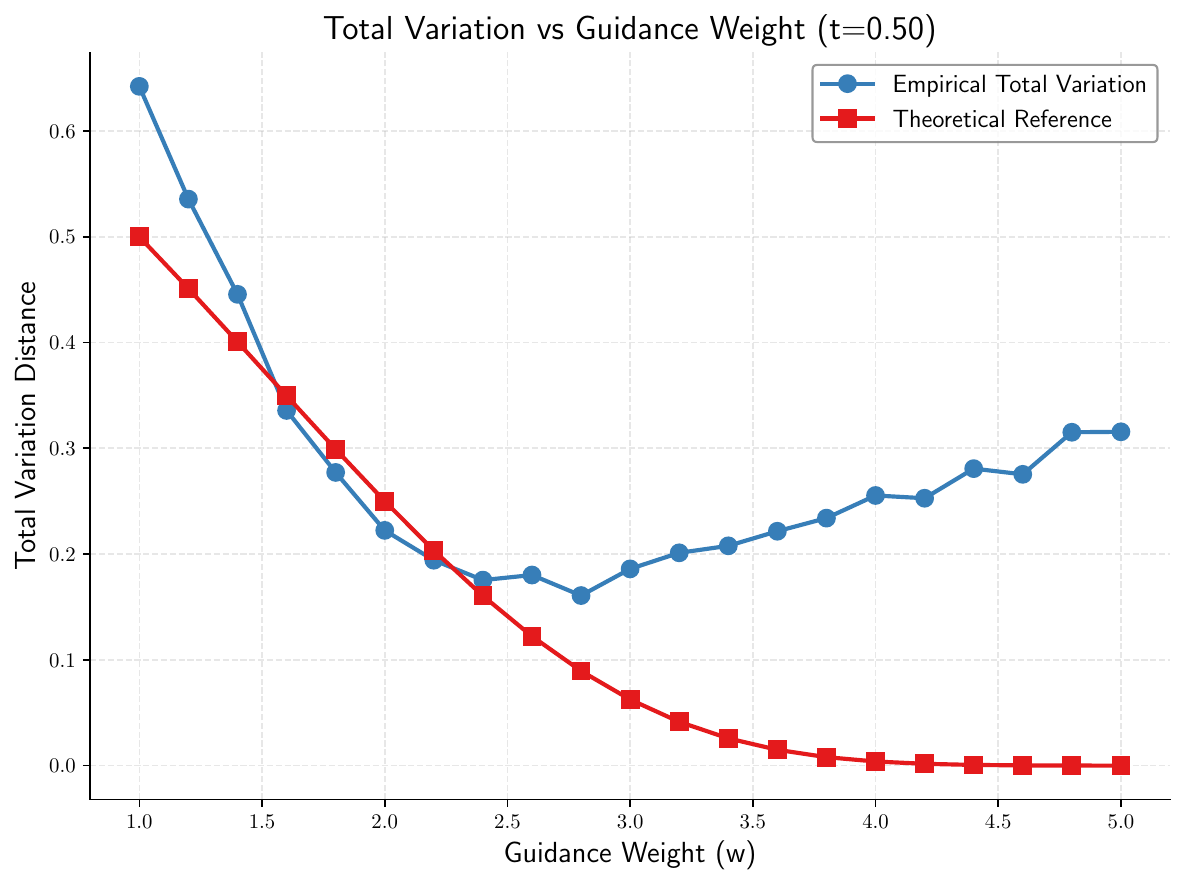}
        \caption{Total variation as a function of $w$.}
    \end{subfigure}
    
    \caption{
        In the first two plots, we illustrate how the effect of guidance differs depending on whether the class overlaps with the rest of the distribution: in disjoint regions, guidance has no effect, while in overlapping regions it redistributes mass. 
        The third plot shows how $\TV$ evolves with $w$, closely matching the result of Proposition~\ref{prop:convergence rate} for small $w$.
        \vspace{-0.2in}
    }
    \label{fig:1d-threeplots}
\end{figure}

\noindent\textbf{Experiments in $2$D.} We consider a setup analogous to the $1$D case. Our distributions contain two diamond shaped distributions. One of the diamonds remains in the corner, while another is in the center. We consider the case where the center mode is disjoint or intersecting the other classes. For more visualizations of the data distribution we refer the reader to the appendix. We observe in Figure \ref{fig:2d-threeplots}-(a)(b) that a similar phenomenon where the intersection vanishes under guidance occurs in $2$D. In Figure \ref{fig:2d-threeplots}-(c), the plots has a flat region for large $w$, which may be caused by the inefficiency of Tau-leaping in simulating sharp transition in the sampling dynamics. 

\begin{figure}[htbp]
    \centering
    
    \begin{subfigure}[t]{0.35\textwidth}
        \includegraphics[width=\linewidth]{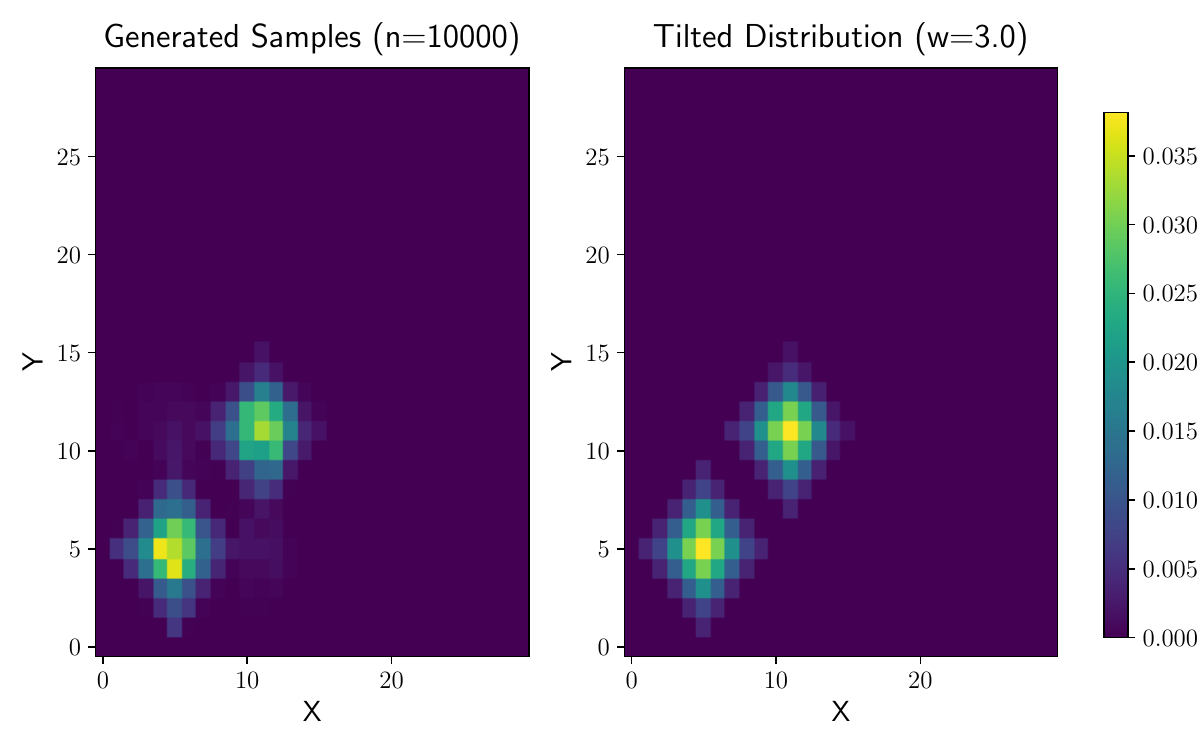}
        \caption{No effect in disjoint support.}
    \end{subfigure}
    \hfill
    \begin{subfigure}[t]{0.35\textwidth}
        \includegraphics[width=\linewidth]{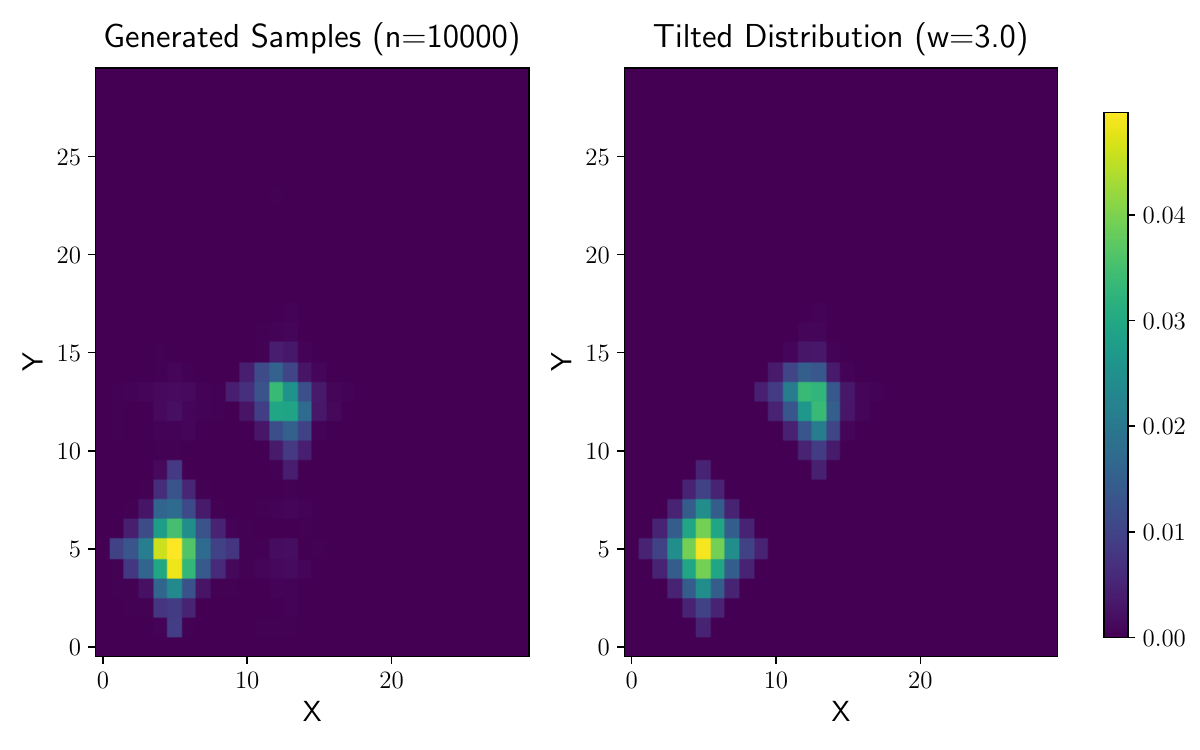}
        \caption{Mass is shifted away from intersecting region.}
    \end{subfigure}
    \hfill
    \begin{subfigure}[t]{0.26\textwidth}
        \includegraphics[width=\linewidth]{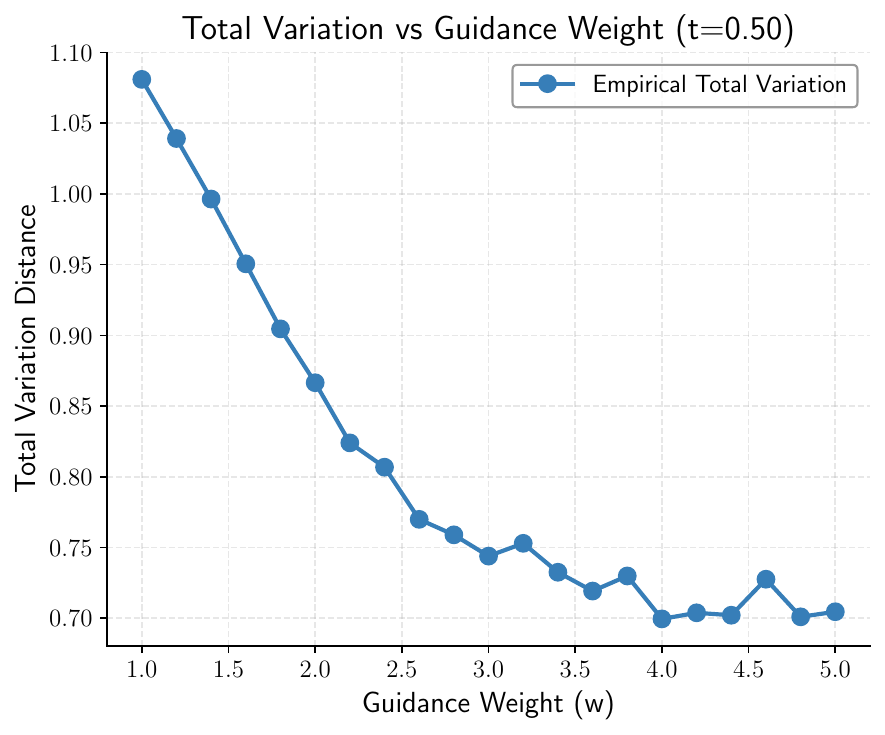}
        \caption{Total variation as a function of $w$.}
    \end{subfigure}
    
    \caption{
        In the first two plots, we illustrate how the effect of guidance differs depending on whether the class overlaps with the rest of the distribution: in disjoint regions, guidance has no effect, while in overlapping regions it redistributes mass. 
        The third plot shows how $\TV$ evolves with $w$.
        \vspace{-0.1in}}
    \label{fig:2d-threeplots}
\end{figure}

\noindent\textbf{Other Experiments.} We also conduct additional practical experiments: one in a higher-dimensional setting ($5$D), and another applying CFG-based discrete diffusion for conditional generation on MNIST. Due to space constraints, the full details of these experiments are provided in Appendix~\ref{append:numerical examples}.
\vspace{-0.1in}
\section{Conclusions}\label{sec:conclusion}
\vspace{-0.1in}
In this paper, we developed a rigorous framework to analyze the effect of CFG in masked discrete diffusion models, focusing on low-dimensional settings ($1$D and $2$D). We addressed how guidance reshapes the generated distribution and influences convergence of the reverse dynamics. Our results include explicit formulas showing that guidance amplifies class-specific regions and suppresses overlaps, with strength controlled by the guidance parameter $w$. We also prove that the $\TV$ distance decays double-exponentially in $w$ for large $w$. These findings offer theoretical insight into the role of guidance, bridging empirical practice and foundational analysis.

\noindent\textbf{Future Work.}
This work assumes idealized conditions with exact concrete scores and perfect numerical integration. A natural next step is to analyze how guidance interacts with the score error and discretization error. Key questions include whether guidance amplifies or mitigates these errors, how they propagate during sampling, and how their interaction with the guidance strength affects convergence and sample quality. Another important direction is to extend our analysis beyond low-dimensional settings to high-dimensional spaces, where geometry and multimodality pose additional challenges. Addressing these questions would deepen the theoretical foundations of guided discrete diffusion and improve its practical reliability.
\newpage
\bibliography{Ref}

\newpage
\appendix
\tableofcontents
\newpage
\section{Additional Related Work}\label{append:related work} 
\textbf{Diffusion Models in Continuous and Discrete Spaces.} Diffusion models were first developed for continuous data, where Gaussian noise is gradually added and then removed through a learned reverse process \citep{ho2020denoising, songscore}. While effective for images~\citep{dhariwal2021diffusion} and audio~\citep{kongdiffwave}, these models are less suited for inherently discrete data such as text or categorical variables. To address this, discrete diffusion models have been proposed, including D3PMs \citep{austin2021structured} and masked token diffusion models \citep{campbell2022continuous, shi2024simplified, ou2024your}, which model corruption through masking or categorical transitions.

Compared to the continuous setting, the discrete domain introduces additional challenges. The forward marginals are generally non-Gaussian and often lack analytical tractability. Score functions must be redefined, commonly through ratios of discrete logits~\citep{campbell2022continuous,lou2023discrete}. In masked discrete diffusion, the corruption process ensures analytic marginals and enables simpler likelihood training~\citep{shi2024simplified,ou2024your}, making it attractive for both theoretical analysis and practical applications.

\noindent\textbf{Diffusion Models with Guidance in Continuous Space.} To enhance controllability in generation, guidance techniques have become central to the success of continuous diffusion models. These methods condition the generative process on auxiliary information, such as class labels, text prompts, or segmentation maps. Two major paradigms have emerged: classifier guidance and classifier-free guidance (CFG). Classifier guidance, introduced by \citet{dhariwal2021diffusion}, conditions the sampling process by incorporating the gradient of a pretrained classifier into the reverse diffusion dynamics. This approach enables targeted generation (e.g., class-conditional image synthesis) without modifying the training of the diffusion model. However, it requires a separately trained, often large, and accurate classifier. To mitigate this dependency, classifier-free guidance was proposed by \citet{ho2021classifier}. In CFG, the diffusion model is trained jointly on conditional and unconditional data, allowing the sampling process to interpolate between guided and unguided generations by scaling the conditional score.

Recent theoretical studies have begun to analyze how guidance alters the reverse-time dynamics of diffusion models in continuous spaces~ \citep{bradley2024classifier,wu2024theoretical,chidambaram2024does}. These works focus on simplified settings. Assuming the conditional likelihood (hence the tilted distribution) Gaussian, \citet{bradley2024classifier} proved that the probability flow ODE with guidance does not sample from the correct tilted distribution, and CFG is equivalent to a special predictor-corrector scheme. \citet{wu2024theoretical} quantified the effect of guidance in Gaussian mixture models: larger guidance always reduces the differential entropy, hence generating more homogeneous samples, and larger guidance always increases the classification confidence of the sampled class. \citet{chidambaram2024does} analyzed the dynamics of the guided probability flow ODE for $1$D mixture models. Their results reflect that the guided ODE leverages the geometric information about the data distribution even if such information is absent in the classifier being used for guidance. These theoretical findings provide foundational understanding of the mechanisms behind guidance, though their applicability to high-dimensional or real-world scenarios remains limited.

\noindent\textbf{Diffusion Models with Guidance in Discrete Space.} In controllable generation for discrete data, one way was to apply guidance with continuous embedding of the discrete data~\citep{li2022diffusion,han2022ssd,lovelace2023latent,stark2024dirichlet,guo2024plug}. It was recently proposed by \citet{nisonoff2024unlocking} and \citet{sahoo2024simple} to apply guidance directly to discrete diffusion models. \citet{nisonoff2024unlocking} proposed to apply guidance on the reverse transition rate matrices while \citet{sahoo2024simple} proposed to apply guidance on the transition kernels of the reverse process. While these two formulations could lead to different sampled distributions, our paper studies the one introduced in \citet{nisonoff2024unlocking}, and leave the one in \citet{sahoo2024simple} and their comparison as a future work.  

To the best of our knowledge, this work provides the first theoretical analysis of how guidance influences the performance of discrete diffusion models. Comparing to existing studies on continuous diffusion models, our work is most closely related to that of \citet{chidambaram2024does}, as both analyze the sampling dynamics in low-dimensional settings. Leveraging the tractability of masked discrete diffusion, we derive explicit solutions for the reverse dynamics in 1D and 2D. This allows us to address questions discussed in \citet{bradley2024classifier} and \citet{wu2024theoretical}, but in the discrete domain. Specifically, we show that in $1$D, the guided sampling distribution matches the tilted distribution exactly, with discrepancies emerging only in $2$D and higher. Furthermore, we demonstrate that sample diversity decreases with increasing guidance strength, as the probability mass in overlapping regions vanishes. Notably, the total variation distance between sampled distribution and intermediate distribution along the reverse dynamics decays double-exponentially with guidance strength, potentially leading to numerical instability at large guidance values—a phenomenon also observed in \citet{chidambaram2024does} in the continuous case.

\section{Properties of Masked Discrete Diffusion Models without Guidance}\label{append:forward process}
\begin{lemma}[Diagonalization of $Q$]\label{lem:diag of Q} $Q= X \Lambda X^{-1}$ with $\Lambda= \mathrm{Diag}(-1,\cdots, -1,0)$ and  
\begin{align*}
    X = X^{-1} = \begin{pmatrix}
        -1 & -1  & \cdots & -1 & -1 & 0 \\
        1 & 0 &  \cdots & 0 & 0 & 0 \\
        0 & 1 & \cdots & 0 & 0 & 0 \\
        \vdots & \vdots & \ddots & \vdots & \vdots & \vdots  \\
        0 & 0 & \cdots & 1 & 0 & 0\\
        0 & 0  & \cdots & 0 & 1 & 1 
    \end{pmatrix}_{N\times N}
\end{align*}    
\end{lemma}
\begin{proof}[Proof of Proposition \ref{prop:forward density evolution}]
The solution to \eqref{eq:forward process} with initial distribution $\mu_0=\mu$ can be expressed as
\begin{align*}
    \mu_t & = \exp(t Q) \mu = \exp\big( t  \sum_{d=1}^D I_N \otimes \cdots \underbrace{ Q }_{d^{th}}\cdots \otimes I_N \big) \mu \\
    & = \prod_{d=1}^D \exp\big( t  I_N \otimes \cdots \underbrace{ Q }_{d^{th}}\cdots \otimes I_N \big) \mu \\
    & = \prod_{d=1}^D \exp\big( t  (X\otimes \cdots \otimes X) \big( I_N \otimes \cdots \underbrace{ \Lambda }_{d^{th}}\cdots \otimes I_N \big)  (X\otimes \cdots \otimes X)^{-1}\big)\mu \\
    & = \prod_{d=1}^D (X\otimes \cdots \otimes X)\exp\big( I_N \otimes \cdots \underbrace{ t \Lambda }_{d^{th}}\cdots \otimes I_N \big)(X\otimes \cdots \otimes X)^{-1} \mu \\
    & = (X\otimes \cdots \otimes X) \exp(t\Lambda)^{\otimes D} (X\otimes \cdots \otimes X)^{-1} \mu\\
    & = \big( X \exp(t\Lambda) X^{-1} \big)^{\otimes D} \mu ,
\end{align*}    
where the second identity uses the fact that $(I_N \otimes \cdots \underbrace{ Q }_{d^{th}}\cdots \otimes I_N)_d$ commute with each other. Then the statement follows from Lemma \ref{lem:diag of Q}.  
\end{proof}

\begin{proof}[Proof of Theorem \ref{thm:1d reverse no guidance}]
   With the expression of the distribution along the forward process, we can write the reverse transition rate matrix based on Proposition \ref{prop:reverse transition rate}. We have
\begin{align}\label{eq:reverse rate matrix 1d}
    \bQ_t  = \frac{e^{-t}}{1-e^{-t}} \bQ \coloneqq \frac{e^{-t}}{1-e^{-t}}\begin{pmatrix}
        0 & 0 & \cdots & 0 & p(1) \\
        0 & 0 & \cdots & 0 & p(2) \\
        \vdots & \vdots & \ddots & \vdots & \vdots  \\
        0 & 0 & \cdots & 0 & p(N-1) \\
        0 & 0 & \cdots &  0 & -1
    \end{pmatrix}_{N\times N}
\end{align}
The eigenvalues and eigenvectors of $\bQ$ are given by
\begin{align*}
    &\bar{\lambda}_1=\bar{\lambda}_2=\cdots = \bar{\lambda}_{N-1}=0,\quad \bar{\lambda}_N=-1,\\
    &\vec{u}_1=\begin{pmatrix}
        1 \\
        0\\
        \vdots\\
        0 \\
        0
    \end{pmatrix},
    \vec{u}_2=\begin{pmatrix}
        0 \\
        1\\
        \vdots\\
        0 \\
        0
    \end{pmatrix},
    \cdots,
    \vec{u}_{N-1}=\begin{pmatrix}
        0 \\
        0\\
        \vdots\\
        1 \\
        0 
    \end{pmatrix},
    \vec{u}_N=\begin{pmatrix}
        p(1) \\
        p(2)\\
        \vdots\\
        p(N-1) \\
        -1
    \end{pmatrix}
\end{align*}
The eigenvalue decomposition of $\bQ$ is given by $\bQ=\bar{X} \bar{D} \bar{X}^{-1}$ with $\bar{D} = \diag(0,0,\cdots, 0,-1)\in \mb{R}^{N\times N}$
\begin{align*}
    \bar{X} = \bar{X}^{-1}= \begin{pmatrix}
        1 & 0  & \cdots & 0 & p(1) \\
        0 & 1 &  \cdots & 0 & p(2) \\
        \vdots & \vdots & \ddots & \vdots & \vdots  \\
        0 & 0 & \cdots & 1 & p(N-1)\\
        0 & 0  & \cdots & 0 & -1 
    \end{pmatrix}_{N\times N}.
\end{align*}
A simple computation tells that
\begin{align*}
    \exp\big( \int_0^{T-t} \bQ_{T-s}\dee s \big) & = \exp\big( \int_t^{T} \frac{e^{-s}}{1-e^{-s}} \dee s \bQ \big)  = \bar{X} \exp\big( \ln(\frac{1-e^{-T}}{1-e^{-t}})\bar{D} \big) \bar{X}^{-1} \\
    &=  \begin{pmatrix}
        1 & 0  & \cdots & 0 & \big( 1-\frac{1-e^{-t}}{1-e^{-T}} \big)p(1) \\
        0 & 1 &  \cdots & 0 & \big( 1-\frac{1-e^{-t}}{1-e^{-T}} \big)p(2) \\
        \vdots & \vdots & \ddots & \vdots & \vdots  \\
        0 & 0 & \cdots & 1 & \big( 1-\frac{1-e^{-t}}{1-e^{-T}} \big)p(N-1)\\
        0 & 0  & \cdots & 0 & \frac{1-e^{-t}}{1-e^{-T}} 
    \end{pmatrix}_{N\times N}.
\end{align*}
Along the reverse sampling dynamics, we have $q_{t}= \exp\big( \int_0^{t} \bQ_{T-s}\dee s \big) q_0$, which implies
\begin{align}\label{eq:absorbing reverse density 1d}
    q_t(x) =\left\{
    \begin{aligned}
        & q_0(x) + (1-\frac{1-e^{-(T-t)}}{1-e^{-T}})p(x)q_0(N) , \quad & x=1,2,\cdots, N-1, \\
        & \frac{1-e^{-(T-t)}}{1-e^{-T}}q_0(N), \quad & x=N.
    \end{aligned}
    \right.
\end{align} 
Last, the theorem follows from plugging in $q_0=\delta_N$.
\end{proof}

\section{Properties of Masked Discrete Diffusion Models with CFG when $D=1$}\label{append:1d density} 
\begin{proof}[Proof of Theorem \ref{thm:1d reverse}] Notice that the reverse transition rate matrix $\hat{Q}_t^{z,w}=\mc{Z}^{z,w} \hat{Q}_t[p^{z,w}]\coloneqq \mc{Z} \hat{Q}_t$. Following the same computation in the proof of Theorem \ref{thm:1d reverse no guidance}, we have
\begin{align*}
    \exp\big( \int_0^{T-t} \mc{Z}\hat{Q}_{T-s}\dee s \big) & = \exp\big( \mc{Z} \int_t^{T} \frac{e^{-s}}{1-e^{-s}} \dee s \hat{Q} \big)  = \bar{X} \exp\big( \mc{Z}\ln(\frac{1-e^{-T}}{1-e^{-t}})\bar{D} \big) \bar{X}^{-1} \\
    &=  \begin{pmatrix}
        1 & 0  & \cdots & 0 & \big( 1-(\frac{1-e^{-t}}{1-e^{-T}})^{\mc{Z}} \big)p^{z,w}(1) \\
        0 & 1 &  \cdots & 0 & \big( 1-(\frac{1-e^{-t}}{1-e^{-T}})^{\mc{Z}} \big)p^{z,w}(2) \\
        \vdots & \vdots & \ddots & \vdots & \vdots  \\
        0 & 0 & \cdots & 1 & \big( 1-(\frac{1-e^{-t}}{1-e^{-T}})^{\mc{Z}} \big)p^{z,w}(N-1)\\
        0 & 0  & \cdots & 0 & (\frac{1-e^{-t}}{1-e^{-T}})^{\mc{Z}} 
    \end{pmatrix}_{N\times N}.
\end{align*}
Along the reverse sampling dynamics \eqref{eq:guided reverse process}, we have $q^{z,w}_{t}= \exp\big( \int_0^{t} \mc{Z}\hat{Q}_{T-s}\dee s \big) p^{z,w}_T$, which implies
\begin{align}\label{eq:absorbing ideal guided reverse density 1d}
    q^{z,w}_t(x) =\left\{
    \begin{aligned}
        & q^{z,w}_0(x) + \big(1-\frac{1-e^{-(T-t)}}{1-e^{-T}}\big)^{\mc{Z}}p^{z,w}(x)q^{z,w}_0(N) , \quad & x=1,2,\cdots, N-1, \\
        & \big(\frac{1-e^{-(T-t)}}{1-e^{-T}}\big)^{\mc{Z}}q^{z,w}_0(N), \quad & x=N.
    \end{aligned}
    \right.
\end{align}
Last, the theorem follows from plugging in $q_0^{z,w}=\delta_N$.
\end{proof}
\begin{proof}[Proof of Proposition \ref{prop:convergence rate}] The result directly follows from Theorem \ref{thm:1d reverse} and the formula $\TV(\mu_1,\mu_2)=\tfrac{1}{2}\sum_x |\mu_1(x)-\mu_2(x)|$.    
\end{proof}
\begin{proof}[Proof of Proposition \ref{prop:1d sampled distribution property}] According to Theorem \ref{thm:1d reverse}, in both cases, the sampled distribution is the same as the tilted distribution, i.e., $q_T^{z_1,w}=p^{z_1,w}$.

In case (1), it is obvious that $p^{z_1,w}=p(\cdot|z_1)$.

In case (2), we have $p^{z_1,w}(x)\propto (\tfrac{p(x|z_1)}{p(x)})^w p(x|z_1)$. Under Assumption \ref{assup:full distribution}, we have
\begin{align*}
\tfrac{p(x|z_1)}{p(x)} = \left\{
\begin{aligned}
    &\frac{p(x|z_1)}{a_1p(x|z_1)}, \qquad & x\in \mc{X}_1\setminus S_1 \\
   & \frac{p(x|z_1)}{\sum_k a_k p(x|z_k)} ,\qquad & x\in S_1, \\
   & 0 , \qquad &\text{otherwise}.
\end{aligned}
\right.    
\end{align*}
Then Proposition \ref{prop:1d sampled distribution property}-(2) is proved.   
\end{proof}
\begin{definition}[Local mean and covariance] For any probability distribution $\mu$ on $S$ and any subset $A\subset S$, the local mean and local covariance of $\mu$ on $A$ are defined respectively as
   \begin{align*}
       m_A(\mu)\coloneqq \sum_{x\in A} x \mu_A(x),\quad \Sigma_A(\mu)\coloneqq \sum_x (x-m_A(\mu))(x-m_A(x))^\intercal \mu_A(x),
   \end{align*} 
   where $\mu_A(x)\coloneqq \mu(x)/\sum_{y\in A}\mu(y)$ is the restriction of $\mu$ on $A$.
\end{definition}
\begin{lemma}\label{lem:local covariance preserve} Under the assumptions in Proposition \ref{prop:1d sampled distribution property}, for all $w>0$, $\Sigma_{\mc{X}_1\setminus S_1}(q_T^{z_1,w})=\Sigma_{\mc{X}_1\setminus S_1}(p(\cdot|z_1))$. 
\end{lemma}
\begin{proof}[Proof of Lemma \ref{lem:local covariance preserve}]
According to Proposition \ref{prop:1d sampled distribution property},
\begin{align*}
    m_{\mc{X}_1\setminus S_1}(q_T^{z_1,w}) = \sum_{x\in \mc{X}_1\setminus S_1} x p(x|z_1)/\sum_{y\in \mc{X}_1\setminus S_1} p(y|z_1)  = m_{\mc{X}_1\setminus S_1}(p(\cdot|z_1)),
\end{align*}
and
\begin{align*}
    \Sigma_{\mc{X}_1\setminus S_1}(q_T^{z_1,w})&=\sum_{x\in \mc{X}_1\setminus S_1} (x-m_{\mc{X}_1\setminus S_1}(q_T^{z_1,w}))(x-m_{\mc{X}_1\setminus S_1}(q_T^{z_1,w}))^\intercal p(x|z_1)/\sum_{y\in \mc{X}_1\setminus S_1} p(y|z_1) \\
    &=\sum_{x\in \mc{X}_1\setminus S_1} (x-m_{\mc{X}_1\setminus S_1}(p(\cdot|z_1))(x-m_{\mc{X}_1\setminus S_1}(p(\cdot|z_1))^\intercal p(x|z_1)/\sum_{y\in \mc{X}_1\setminus S_1} p(y|z_1)\\
    &=\Sigma_{\mc{X}_1\setminus S_1}(p(\cdot|z_1)).
\end{align*}
\end{proof}

\section{Properties of Masked Discrete Diffusion Models with CFG when $D=2$}\label{append:2d density}
\begin{proof}[Proof of Proposition \ref{prop:reverse rate matrix 2d}]
    For any $x,y\in S$ with $x_i=y_i\neq N$ and $x_j=N\neq y_j$, according to \eqref{eq:guided reverse diffusion matrix}, we have
    \begin{align*}
        \hat{Q}^{z,w}_t(y,x) & = \Bar{Q}^{z}_t (y,x)^{-w} \Bar{Q}_t(y,x)^{1+w} = \big( \frac{p_t(y)}{p_t(x)} \big)^{-w}\big( \frac{p_t(y|z)}{p_t(x|z)} \big)^{1+w} \\
        & = \big( \frac{e^{-2t}p(y) }{e^{-t}(1-e^{-t})p(y_i)} \big)^{-w}\big( \frac{e^{-2t}p(y|z) }{e^{-t}(1-e^{-t})p(y_i|z)} \big)^{1+w} \\
        & = \frac{e^{-t}}{1-e^{-t}} \frac{p(y)^{-w}p(y|z)^{1+w}}{p(y_i)^{-w}p(y_i|z)^{1+w}} \\
        & = \frac{e^{-t}}{1-e^{-t}} \frac{\mc{Z}^{z,w}p^{z,w}(y)}{p(y_i)^{-w}p(y_i|z)^{1+w}},
    \end{align*}
    where the third identity follows from Proposition \ref{prop:forward density evolution}, and the last identity follows from the definition of $p^{z,w}$. Next, following the same approach, for any $x,y\in S$ with $x_i=N\neq y_i$ and $x_j= y_j=N$, we have
    \begin{align*}
        \hat{Q}_t^{z,w}(y,x) &= \big( \frac{p_t(y)}{p_t(x)} \big)^{-w}\big( \frac{p_t(y|z)}{p_t(x|z)} \big)^{1+w} \\
        & = \big( \frac{e^{-t}(1-e^{-t})p(y_i)}{(1-e^{-t})^2 } \big)^{-w}\big( \frac{e^{-t}(1-e^{-t})p(y_i|z)}{(1-e^{-t})^2} \big)^{1+w} \\
        & = \frac{e^{-t}}{1-e^{-t}} {p(y_i)^{-w}p(y_i|z)^{1+w}}
    \end{align*}
    Last, the other cases for different $(y,x)$ follows from the definition of the transition rate matrix,0 \eqref{eq:reverse diffusion matrix} and \eqref{eq:conditional reverse diffusion matrix}.
\end{proof}

\begin{proof}[Proof of Theorem \ref{thm:2d reverse}] Our proof follows from the following steps.

\noindent\underline{Step 1: represent the reverse transition rate matrix blockwisely.} The matrix $Q^{z,w}$ in Proposition \ref{prop:reverse rate matrix 2d} can be represented blockwisely as
\begin{align*}
     \hat{Q}^{z,w}= \begin{pmatrix}
        \hat{R}_1^{z,w} & \cdots & \mathbf{0} & \hat{L}_1^{z,w} \\
        \vdots & \ddots & \vdots & \cdots \\
        \mathbf{0} & \cdots & \hat{R}_{N-1}^{z,w} & \hat{L}_{N-1}^{z,w} \\
        \mathbf{0} & \cdots & \mathbf{0} & \hat{M}^{z,w}-\sum_i \hat{L}_i^{z,w}
    \end{pmatrix},
\end{align*}
For all $i=1,2,\cdots N-1$,
\begin{align}\label{eq:absorbing practical reverse matrix 2d blocks}
    &\hat{R}_i^{z,w} \coloneqq \begin{pmatrix}
         0 & 0 & \cdots &  \big(\frac{p(i,1)}{\sum_l p(i,l)}\big)^{-w}\big(\frac{p(i,1|z)}{\sum_l p(i,l|z)}\big)^{1+w} \\
            0 & 0 & \cdots & \big(\frac{p(i,2)}{\sum_l p(i,l)}\big)^{-w}\big(\frac{p(i,2|z)}{\sum_l p(i,l|z)}\big)^{1+w} \\
            \vdots & \vdots & \ddots & \vdots \\
            0 & 0 & \cdots &  -\sum_{j}\big(\frac{p(i,j)}{\sum_l p(i,l)}\big)^{-w}\big(\frac{p(i,j|z)}{\sum_l p(i,l|z)}\big)^{1+w} 
    \end{pmatrix},\\
    &\hat{L}_i^{z,w} \coloneqq  \\
    &{\tiny{\begin{pmatrix}
         {\tiny\big(\frac{p(i,1)}{\sum_l p(l,1)}\big)^{-w}\big(\frac{p(i,1|z)}{\sum_l p(l,1|z)}\big)^{1+w}} & \cdots & 0 & 0 \\
            \vdots &  \ddots & \vdots & 0 \\
            0 & \cdots & {\tiny \big(\frac{p(i,N-1)}{\sum_l p(l,N-1)}\big)^{-w}\big(\frac{p(i,N-1|z)}{\sum_l p(l,N-1|z)}\big)^{1+w}}  & \vdots \\
            0 & 0 & \cdots &  {\tiny{(\sum_l p(i,l))^{-w}(\sum_l p(i,l|z))^{1+w} }}
    \end{pmatrix}}},\nonumber\\
   & \hat{M}^{z,w} \coloneqq    \begin{pmatrix}
         0 & 0 & \cdots & \big(\sum_l p(l,1)\big)^{-w}\big(\sum_l p(l,1|z)\big)^{1+w}  \\
            0 &  0 & \cdots & \big(\sum_l p(l,2)\big)^{-w}\big(\sum_l p(l,2|z)\big)^{1+w} \\
            \vdots & \vdots & \ddots  & \vdots \\
            0 & 0 & \cdots &   -\sum_{j}\big(\sum_l p(l,j)\big)^{-w}\big(\sum_l p(l,j|z)\big)^{1+w}\label{eq:absorbing practical reverse matrix 2d blocks 2}
            \end{pmatrix}  ,
\end{align}
where we used the definition of $\mc{Z}$ and marginal distributions.

\noindent\underline{Step 2: Eigenvalue decomposition for $\hat{Q}^{z,w}$} Since $\hat{Q}^{z,w}$ is upper triangular, its eigenvalues are diagonal entries. For all $i,j=1,2,\cdots, N-1$, define 
{\small{
\begin{align}\label{eq:coefficient c d definition}
    & c_i \coloneqq  \frac{\sum_l p(i,l)^{-w}p(i,l|z)^{1+w}}{(\sum_l p(i,l))^{-w}(\sum_l p(i,l|z))^{1+w}}, \quad d_j=\frac{\sum_l p(l,j)^{-w}p(l,j|z)^{1+w}}{\big(\sum_l p(l,j)\big)^{-w}\big(\sum_l p(l,j|z)\big)^{1+w}}, \\
    & c_N \coloneqq \frac{\sum_{l_1,l_2} p(l_1,l_2)^{-w}p(l_1,l_2|z)^{1+w}}{\sum_{l_1}\big(\sum_{l_2} p(l_1,l_2)\big)^{-w}\big(\sum_{l_2} p(l_1,l_2|z)\big)^{1+w}} ,\quad d_N\coloneqq \frac{\sum_{l_1,l_2} p(l_1,l_2)^{-w}p(l_1,l_2|z)^{1+w}}{\sum_{l_2}\big(\sum_{l_1} p(l_1,l_2)\big)^{-w}\big(\sum_{l_1} p(l_1,l_2|z)\big)^{1+w}}. 
\end{align}}}
Then the set of eigenvalues for $\hat{Q}^{z,w}$, denoted as $\{\lambda_{i,j}^{z,w}\}_{i,j\in [N]}$ can be represented as
\begin{align*}
    &\lambda_{i,1}^{z,w}=\cdots =\lambda_{i,N-1}^{z,w}=0, \lambda_{i,N}^{z,w}=-c_i,\quad i=1,2,\cdots, N-1,\\
    & \lambda^{z,w}_{N,j} = -d_j, \lambda^{z,w}_{N,N} = -\mc{Z}(1/c_N+1/d_N),\quad j=1,2,\cdots, N-1.
\end{align*}
The associated eigenvectors to $\lambda^{z,w}_{i,j}$, denoted as $\vec u = (\vec u_1,\cdots, \vec u_N)^\intercal$, satisfies 
  \begin{align*}
    \left\{
    \begin{aligned}
        &\hat{R}_l^{z,w}\vec{u}_l+\hat{L}_l^{z,w}\vec{u}_N = \lambda_{i,j}^{z,w}\vec{u}_l,\quad l=1,2,\cdot, N-1\\
        &\big(\hat{M}^{z,w}-\sum_l \hat{L}_l^{z,w}\big)\vec{u}_N = \lambda_{i,j}^{z,w}\vec{u}_N.
    \end{aligned}
    \right.
    \end{align*}
     The eigenvectors can be studied in two cases:
    \begin{itemize}
        \item [(1)] When $1\le i\le N-1$, we can pick $\vec u_N=\mathbf{0}$. Then for $l=1,2,\cdots, N-1$, $\hat{R}_l^{z,w}\vec u_l=\lambda_{i,j}^{z,w}\vec u_l$. For $l\neq i$, we pick $\vec u_l=\mathbf{0}$. For $l=i$, $\vec u_i$ is the eigenvector to $\hat{R}_i^{z,w}$ associated with the eigenvalue $\lambda_{i,j}^{z,w}$: for $j=1,2,\cdots, N-1$, we pick $\vec u_i = \vec{u}_{i,j}=\vec{e}_j$. For $j=N$, we pick $ \vec u_i =\vec{u}_{i,N}$ to be
        {\small
        \begin{align}\label{eq:eigenvector 1}
                   \big(  \frac{p(i,1)^{-w}p(i,1|z)^{1+w}}{\sum_l p(i,l)^{-w}p(i,l|z)^{1+w}},\cdots , \frac{p(i,N-1)^{-w}p(i,N-1|z)^{1+w}}{\sum_l p(i,l)^{-w}p(i,l|z)^{1+w}} ,-1  \big)^\intercal
        \end{align}}
        \item [(2)] When $i=N$, $\vec u_N\neq \mathbf{0}$. We need to solve $\big(\hat{M}^{z,w}-\sum_l \hat{L}_l^{z,w}\big)\vec{u}_N = \lambda_{i,j}^{z,w}\vec{u}_N$ first. For different $j$, we pick $\vec u_N=\vec u_{N,j}$ with $\vec u_{N,j} =\vec e_j$ for $j=1,\cdots N-1$ and for $j=N$, $\vec{u}_{N,j}=$
        {\small
        \begin{align}\label{eq:eigenvector 3}
                  \big( \frac{(\sum_l p(l,1))^{-w}(\sum_l p(l,1|z)^{1+w}}{-\lambda_{N,N}^{z,w}-d_1} , \cdots , \frac{(\sum_l p(l,N-1))^{-w}(\sum_l p(l,N-1|z))^{1+w}}{-\lambda_{N,N}^{z,w}-d_{N-1}} , -1 \big)^\intercal
         \end{align}}
        Next for each $j=1,\cdots, N$, we solve $\big(\hat{R}_l^{z,w}-\lambda_{N,j}^{z,w}I_N\big) \vec{u}_{lj} = -\hat{L}_l^{z,w}\vec{u}_{N,j}$ for all $l=1,2,\cdots, N-1$. We get  
        \begin{align}\label{eq:eigenvector 2}
            &  \vec{u}_{l,j} 
        =\left\{
        \begin{aligned}
           &-\frac{p(l,j)^{-w}p(l,j|z)^{1+w}}{\sum_{l'} p(l',j)^{-w}p(l',j|z)^{1+w}}\vec{e_j}    ,\quad j=1,\cdots, N-1,\\
           & \big( \vec{u}_{l,N}(1), \cdots \vec{u}_{l,N}(N-1), \vec{u}_{l,N}(N) \big)^\intercal,   j=N
        \end{aligned}
        \right.
      \end{align}
      with
      \begin{align*}
          \vec{u}_{l,N}(l') & = -\frac{1}{\lambda_{N,N}^{z,w}}(\frac{1}{\lambda_{N,N}^{z,w}+c_l }+\frac{1}{\lambda_{N,N}^{z,w}+d_{l'} })p(l,l')^{-w}p(l,l'|z)^{1+w}, \\
          \vec{u}_{l,N}(N) & =-\frac{1}{c_l(\lambda_{N,N}^{z,w}+c_l )}\sum_{l'} p(l,l')^{-w}p(l,l'|z)^{1+w}.
      \end{align*}
    \end{itemize}
Collect all the eigen information above, we diagonalize $\hat{Q}^{z,w}$ blockwisely: $\hat{Q}^{z,w}=\hat{X}^{z,w}\hat{D}^{z,w}(\hat{X}^{z,w})^{-1}$ s.t.
    \begin{align*}
        \hat{D}^{z,w} &= \text{Diag}(\hat{D}^{z,w}_1,\cdots, \hat{D}^{z,w}_{N-1}, \hat{D}^{z,w}_N),\ \hat{D}_i^{z,w}=\text{Diag}(\lambda^{z,w}_{i,1},\cdots\lambda^{z,w}_{i,N-1},\lambda^{z,w}_{i,N})\ \text{for each }i\\
        \hat{X}^{z,w} &= \begin{pmatrix}
            \hat{X}^{z,w}_1 & \mathbf{O} & \cdots & \mathbf{O} & -\hat{Y}_1^{z,w} \\
            \mathbf{O} &  \hat{X}^{z,w}_2 & \cdots & \mathbf{O} & -\hat{Y}_2^{z,w} \\
            \vdots & \vdots & \ddots & \vdots & \vdots \\
            \mathbf{O} & \mathbf{O} & \cdots & \hat{X}^{z,w}_{N-1} & -\hat{Y}_{N-1}^{z,w} \\
            \mathbf{O} & \mathbf{O} & \cdots & \mathbf{O} & \hat{X}_N^{z,w}
        \end{pmatrix},\\
        (\hat{X}^{z,w})^{-1} &= \begin{pmatrix}
            \hat{X}^{z,w}_1 & \mathbf{O} & \cdots & \mathbf{O} & \hat{X}^{z,w}_1\hat{Y}_1^{z,w}\hat{X}_N^{z,w} \\
            \mathbf{O} &  \hat{X}^{z,w}_2 & \cdots & \mathbf{O} & \hat{X}^{z,w}_2 \hat{Y}_2^{z,w}\hat{X}_N^{z,w} \\
            \vdots & \vdots & \ddots & \vdots & \vdots \\
            \mathbf{O} & \mathbf{O} & \cdots & \hat{X}^{z,w}_{N-1} & \hat{X}^{z,w}_{N-1} \hat{Y}_{N-1}^{z,w}\hat{X}_N^{z,w} \\
            \mathbf{O} & \mathbf{O} & \cdots & \mathbf{O} & \hat{X}_N^{z,w}
        \end{pmatrix},
    \end{align*}
where for each $i=1,\cdots, N-1$,
 \begin{align*}
        \hat{D}^{z,w}_i &= \text{Diag}(0,\cdots,0,-c_i),\\
        \hat{X}^{z,w}_i &= \begin{pmatrix}
            \vert & \cdots & \vert & \vert \\
            \vec{e}_1 & \cdots & \vec{e}_{N-1} & \vec{u}_{i,N} \\
            \vert & \cdots & \vert & \vert
        \end{pmatrix}=(\hat{X}^{z,w}_i)^{-1} \text{with } \vec{u}_{i,N} \ \text{defined in \eqref{eq:eigenvector 1}}\\
      \hat{Y}_i^{z,w} & = -0 \begin{pmatrix}
          \vert & \cdots & \vert & \vert \\
            \vec{u}_{i,1} & \cdots & \vec{u}_{i,N-1} & \vec{u}_{i,N} \\
            \vert & \cdots & \vert & \vert
       \end{pmatrix} \text{with } \{\vec{u}_{i,j}\}_{j=1}^N \ \text{defined in \eqref{eq:eigenvector 2}},
    \end{align*}
    and for $i=N$,
    \begin{align*}
        \hat{D}^{z,w}_N &=\text{Diag}(-d_1,\cdots , -d_{N-1}, -\mc{Z}^{z,w}/c_N-\mc{Z}^{z,w}/d_N  ),\\
       \hat{X}_N^{z,w} & =  \begin{pmatrix}
            \vert & \cdots & \vert & \vert \\
            \vec{e}_1 & \cdots & \vec{e}_{N-1} & \vec{u}_{N,N} \\
            \vert & \cdots & \vert & \vert
        \end{pmatrix}=( \hat{X}_N^{z,w})^{-1} \text{with } \vec{u}_{N,N} \text{ defined in \eqref{eq:eigenvector 3}}.
    \end{align*}
\noindent\underline{Step 3: solve the equation \eqref{eq:guided reverse process} explicitly} The solution to \eqref{eq:guided reverse process} can be computed using the formula $q_t^{z,w} =\exp\big( \int_0^{t} \hat{Q}_{T-s}^{z,w} \dee s \big) q_0^{z,w}$, where the matrix $\exp\big( \int_0^{t} \hat{Q}_{T-s}^{z,w} \dee s \big) $ is computed using the eigenvalue decomposition in \textbf{Step 2}. More specifically,
\begin{align*}
    &\quad \exp\big( \int_0^{t} \hat{Q}_{T-s}^{z,w} \dee s \big) = \exp\big( \int_0^{t} \frac{e^{-(T-s)}}{1-e^{-(T-s)}} \dee s \hat{Q}^{z,w} \big) =\hat{X}^{z,w} \exp\big( \ln (\frac{1-e^{-T}}{1-e^{-(T-t)}}) \hat{D}^{z,w} \big) (\hat{X}^{z,w})^{-1} \\
    &={\small \begin{pmatrix}
            \hat{X}^{z,w}_1 & \mathbf{O} & \cdots & \mathbf{O} & -\hat{Y}_1^{z,w} \\
            \mathbf{O} &  \hat{X}^{z,w}_2 & \cdots & \mathbf{O} & -\hat{Y}_2^{z,w} \\
            \vdots & \vdots & \ddots & \vdots & \vdots \\
            \mathbf{O} & \mathbf{O} & \cdots & \hat{X}^{z,w}_{N-1} & -\hat{Y}_{N-1}^{z,w} \\
            \mathbf{O} & \mathbf{O} & \cdots & \mathbf{O} & \hat{X}^{z,w}_N
        \end{pmatrix}
        \text{Diag}\big( (\frac{1-e^{-T}}{1-e^{-(T-t)}})^{\hat{D}^{z,w}_1},\cdots, (\frac{1-e^{-T}}{1-e^{-(T-t)}})^{\hat{D}^{z,w}_N0 } \big)}\\
        &\quad {\small \begin{pmatrix}
            \hat{X}^{z,w}_1 & \mathbf{O} & \cdots & \mathbf{O} & \hat{X}^{z,w}_1\hat{Y}_1^{z,w}\hat{X}^{z,w}_{N} \\
            \mathbf{O} &  \hat{X}^{z,w}_2 & \cdots & \mathbf{O} & \hat{X}^{z,w}_2\hat{Y}_2^{z,w}\hat{X}^{z,w}_{N} \\
            \vdots & \vdots & \ddots & \vdots & \vdots \\
            \mathbf{O} & \mathbf{O} & \cdots & \hat{X}^{z,w}_{N-1} & \hat{X}^{z,w}_{N-1}\hat{Y}_{N-1}^{z,w}\hat{X}^{z,w}_{N} \\
            \mathbf{O} & \mathbf{O} & \cdots & \mathbf{O} & \hat{X}^{z,w}_{N}
        \end{pmatrix}}\\
    & =     {\small \begin{pmatrix}
            \hat{X}^{z,w}_1 (\frac{1-e^{-T}}{1-e^{-(T-t)}})^{\hat{D}_1^{z,w}} \hat{X}^{z,w}_1 &  \cdots & \mathbf{O} &  \hat{M}^{z,w}_1 \\
            \vdots & \ddots & \vdots & \vdots \\
            \mathbf{O} & \cdots &  \hat{X}^{z,w}_{N-1} (\frac{1-e^{-T}}{1-e^{-(T-t)}})^{\hat{D}_{N-1}^{z,w}} \hat{X}^{z,w}_{N-1} & \hat{M}^{z,w}_{N-1} \\
            \mathbf{O}  & \cdots & \mathbf{O} & \hat{X}^{z,w}_{N} (\frac{1-e^{-T}}{1-e^{-(T-t)}})^{\hat{D}_{N}^{z,w}} \hat{X}^{z,w}_{N}
        \end{pmatrix}}.
\end{align*}
For each $i=1,2\cdots, N-1$
\begin{align*}
    &\hat{X}^{z,w}_i (\frac{1-e^{-T}}{1-e^{-(T-t)}})^{\hat{D}_1^{z,w}} \hat{X}^{z,w}_i = \begin{pmatrix}
            1 & \cdots & 0 & \big( 1-(\frac{1-e^{-(T-t)}}{1-e^{-T}})^{c_i} \big) \frac{p^{z,w}(i,1)}{\sum_l p^{z,w}(i,l)} \\
            \vdots & \ddots & \vdots & \vdots \\
            0 & \cdots & 1 & \big( 1-(\frac{1-e^{-(T-t)}}{1-e^{-T}})^{c_i} \big) \frac{p^{z,w}(i,N-1)}{\sum_l p^{z,w}(i,l)} \\
            0  & \cdots & 0 & (\frac{1-e^{-(T-t)}}{1-e^{-T}})^{c_i}
        \end{pmatrix},\\
      &\hat{M}^{z,w}_i \coloneqq   \hat{X}^{z,w}_i (\frac{1-e^{-T}}{1-e^{-(T-t)}})^{\hat{D}_i^{z,w}}\hat{X}^{z,w}_i\hat{Y}_i^{z,w}\hat{X}^{z,w}_{N}-\hat{Y}_i^{z,w}(\frac{1-e^{-T}}{1-e^{-(T-t)}})^{\hat{D}_N^{z,w}}\hat{X}^{z,w}_{N} \\
      & = \begin{pmatrix}
            \big(1-(\frac{1-e^{-t}}{1-e^{-T}})^{d_1}\big)\frac{p^{z,w}(i,1)}{\sum_l p^{z,w}(l,1)} & \cdots & 0 & \beta_{i,1} \mc{Z} p^{z,w}(i,1) \\
            \vdots & \ddots & \vdots & \vdots \\
            0 & \cdots & \big(1-(\frac{1-e^{-t}}{1-e^{-T}})^{d_{N-1}}\big)\frac{p^{z,w}(i,N-1)}{\sum_l p^{z,w}(l,N-1)} & \beta_{i,N-1} \mc{Z} p^{z,w}(i,N-1) \\
            0  & \cdots & 0 & \beta_{i,N}{\mc{Z} }\sum_l p^{z,w}(i,l)
        \end{pmatrix},
\end{align*}
where for each $i,j=1,2,\cdots, N-1$,
\begin{align}\label{eq:coeff beta}
    \beta_{i,j} \coloneqq & -\frac{1}{c_i(\lambda_{NN}^{z,w}+c_i)}\big( 1-(\frac{1-e^{-(T-t)}}{1-e^{-T}})^{c_i} \big)-\frac{1}{d_j(\lambda_{NN}^{z,w}+d_i)}\big( 1-(\frac{1-e^{-(T-t)}}{1-e^{-T}})^{d_j} \big) \nonumber \\
    & -\frac{1}{\lambda_{NN}^{z,w}}\big( \frac{1}{\lambda_{NN}^{z,w}+c_i} + \frac{1}{\lambda_{NN}^{z,w}+d_j} \big) \big( 1-(\frac{1-e^{-(T-t)}}{1-e^{-T}})^{-\lambda_{NN}^{z,w}} \big) , \\
    \beta_{i,N}\coloneqq &-\frac{1}{c_i(\lambda_{NN}^{z,w}+c_i)} \big( (\frac{1-e^{-(T-t)}}{1-e^{-T}})^{c_i} -(\frac{1-e^{-(T-t)}}{1-e^{-T}})^{-\lambda_{NN}^{z,w}} \big).
\end{align}
For $i=N$, we have 
\begin{align*}
     &\hat{X}^{z,w}_N (\frac{1-e^{-T}}{1-e^{-(T-t)}})^{\hat{D}_N^{z,w}} \hat{X}^{z,w}_N \\
     &\quad=  \begin{pmatrix}
            (\frac{1-e^{-(T-t)}}{1-e^{-T}})^{d_1} & \cdots & 0 & \beta_{N,1} \mc{Z}\sum_l p^{z,w}(l,1) \\
            \vdots & \ddots & \vdots & \vdots \\
            0 & \cdots & (\frac{1-e^{-(T-t)}}{1-e^{-T}})^{d_{N-1}} & \beta_{N,N-1}\mc{Z}\sum_l p^{z,w}(l,N-1) \\
            0  & \cdots & 0 & \big(\frac{1-e^{-(T-t)}}{1-e^{-T}}\big)^{-\lambda_{NN}^{z,w}}
        \end{pmatrix}
\end{align*}
where for each $j=1,2,\cdots, N-1$,
\begin{align}\label{eq:coofficient beta 2}
    \beta_{N,j}\coloneqq -\frac{1}{d_j(\lambda_{NN}^{z,w}+d_j)}  \big( (\frac{1-e^{-(T-t)}}{1-e^{-T}})^{d_j}- (\frac{1-e^{-(T-t)}}{1-e^{-T}})^{-\lambda_{NN}^{z,w}} \big).
\end{align}
Now, we can apply the initial condition $q_0^{z,w}=\delta_{NN}$ to compute $q_t^{z,w}$.
\begin{align*}
    &q_t^{z,w}=\exp\big( \int_0^{t} \hat{Q}_{T-s}^{z,w} \dee s \big) q_0^{z,w} \\
    & =  {\small \begin{pmatrix}
            \hat{X}^{z,w}_1 (\frac{1-e^{-T}}{1-e^{-(T-t)}})^{\hat{D}_1^{z,w}} \hat{X}^{z,w}_1 &  \cdots & \mathbf{O} &  \hat{M}^{z,w}_1 \\
            \vdots & \ddots & \vdots & \vdots \\
            \mathbf{O} & \cdots &  \hat{X}^{z,w}_{N-1} (\frac{1-e^{-T}}{1-e^{-(T-t)}})^{\hat{D}_{N-1}^{z,w}} \hat{X}^{z,w}_{N-1} & \hat{M}^{z,w}_{N-1} \\
            \mathbf{O}  & \cdots & \mathbf{O} & \hat{X}^{z,w}_{N} (\frac{1-e^{-T}}{1-e^{-(T-t)}})^{\hat{D}_{N}^{z,w}} \hat{X}^{z,w}_{N}
        \end{pmatrix}}\begin{pmatrix}
            \mathbf{0} \\
            \vdots \\
            \mathbf{0} \\
            \vec e _N
        \end{pmatrix}\\
        & = \begin{pmatrix}
            \hat{M}_1^{z,w}(:,N) \\
            \vdots \\
            \hat{M}_{N-1}^{z,w}(:,N) \\
            (\hat{X}^{z,w}_{N} (\frac{1-e^{-T}}{1-e^{-(T-t)}})^{\hat{D}_{N}^{z,w}} \hat{X}^{z,w}_{N})(:,N)
        \end{pmatrix}.
\end{align*}
Therefore, for all $i,j=1,2\cdots, N-1$,
\begin{align*}
    q_t^{z,w} (i,j) = \hat{M}_i^{z,w}(j,N) = \beta_{i,j} \mc{Z}p^{z,w}(i,j).
\end{align*}
For $j=N,i=1,2,\cdots, N-1$,
\begin{align*}
     q_t^{z,w} (i,N) = \hat{M}_i^{z,w}(N,N) = \beta_{i,N} \mc{Z} \sum_l p^{z,w}(i,l).
\end{align*}
For $i=N,j=1,2,\cdots, N-1$,
\begin{align*}
    q_t^{z,w}(N,j) = \beta_{N,j} \mc{Z} \sum_l p^{z,w}(l,j).
\end{align*}
Last, for $i=j=N$,
\begin{align*}
    q_t^{z,w}(N,N) = (\frac{1-e^{-(T-t)}}{1-e^{-T}})^{-\lambda_{NN}^{z,w}} .
\end{align*}
Last, Theorem \ref{thm:2d reverse} follows from the following definition of $\alpha_t\coloneqq \beta_{x_1,x_2}$ for all $x\in \{1,2,\cdots, N\}^2$.
\end{proof}
\subsection{Convergence rates for $D=2$}\label{sec:convergence rate 2d}
\begin{proof}[Proof of Proposition \ref{prop:convergence rate 2d}] For simplicity, we denote $\lambda\coloneqq \lambda_{N,N}^{z,w}$. According to Theorem \ref{thm:2d reverse} and Remark \ref{rem:sampled distribution 2d}, the total variation distance can be computed as
{\small
\begin{align*}
    &\quad\TV(q_t^{z,w},q_T^{z,w})\\
    & = \frac{1}{2}\sum_{x\in S} |q_t^{z,w}(x)-q_T^{z,w}(x)| \\
    & = \frac{1}{2}\sum_{x\neq (N,N)}|\alpha_t(x)-\alpha_T(x)|\mc{Z}p^{z,w}(x) + \frac{1}{2}|\alpha_t(N,N)-\alpha_T(N,N)| \\
    & = \frac{1}{2}\mc{Z}\sum_{x_1,x_2\neq N} p^{z,w}(x) \bigg( |\frac{1}{c_{x_1}+\lambda} (\frac{1}{c_{x_1}}r(t)^{c_{x_1}} + \frac{1}{\lambda}r(t)^{-\lambda}) | + |\frac{1}{d_{x_2}+\lambda} (\frac{1}{d_{x_2}}r(t)^{d_{x_2}} + \frac{1}{\lambda}r(t)^{-\lambda}) | \bigg) \\
    &\quad  + \frac{1}{2}\mc{Z}\sum_{x_1\neq N} p^{z,w}(x_1) | \frac{1}{c_{x_1}(\lambda+c_{x_1})}\big( r(t)^{c_{x_1}}-r(t)^{-\lambda} \big) | \\
    &\quad  + \frac{1}{2}\mc{Z}\sum_{x_2\neq N} p^{z,w}(x_2) | \frac{1}{d_{x_2}(\lambda+d_{x_2})}\big( r(t)^{d_{x_2}}-r(t)^{-\lambda} \big) |  + \frac{1}{2} r(t)^{-\lambda}\\
    &\coloneqq \mathrm{\rom{1}}+\mathrm{\rom{2}}+\mathrm{\rom{3}}+\mathrm{\rom{4}},
\end{align*}
}
where $r(t)\coloneqq \tfrac{1-e^{-(T-t)}}{1-e^{-T}}\in (0,1)$. Next, we bound each term respectively. 

For $\mathrm{\rom{1}}$, we bound the two terms inside using the following properties of function $h_1: y\in [1,\infty) \mapsto y^{-1}r(t)^y $: $h_1'(y)<0$ and $h_1''(y)>0$ for all $y$. Notice that $c_l,d_l\ge 1$ and $-\lambda= \tfrac{\mc{Z}}{c_N}+\tfrac{\mc{Z}}{d_N}=\sum_{l_1} p(l_1)^{-w}p(l_1|z)^{1+w}+\sum_{l_2} p(l_2)^{-w}p(l_2|z)^{1+w}=\exp(w\mc{D}_{1+w}(p_1(\cdot|z)|p_1(\cdot)))+\exp(w\mc{D}_{1+w}(p_2(\cdot|z)|p_2(\cdot)))\ge 2$, where we use $\mu_i$ to represent the $i^{th}$ marginal of $\mu$. Therefore, we have
\begin{align*}
    |\frac{1}{c_{x_1}+\lambda} (\frac{1}{c_{x_1}}r(t)^{c_{x_1}} + \frac{1}{\lambda}r(t)^{-\lambda}) | & = \big| \frac{h(c_{x_1})-h(-\lambda)}{c_{x_1}-(-\lambda)} \big| = -h'(c_{x_1}^*)= \frac{1}{c_{x_1}^*}r(t)^{c_{x_1}^*}\big(  \frac{1}{c_{x_1}^*}-\ln r(t)\big) \\
    |\frac{1}{d_{x_2}+\lambda} (\frac{1}{d_{x_2}}r(t)^{d_{x_2}} + \frac{1}{\lambda}r(t)^{-\lambda}) | & = \big| \frac{h(d_{x_2})-h(-\lambda)}{d_{x_2}-(-\lambda)} \big| = -h'(d_{x_2}^*)= \frac{1}{d_{x_2}^*}r(t)^{d_{x_2}^*}\big(  \frac{1}{d_{x_2}^*}-\ln r(t)\big),
\end{align*}
where $c_{x_1}^*$ is between $c_{x_1}$ and $-\lambda$, $d_{x_2}^*$ is between $d_{x_2}$ and $-\lambda$.  

For $\mathrm{\rom{2}}$ and $\mathrm{\rom{3}}$, we bound the two terms using the property of the function $h_2:y \in[1,\infty)\mapsto r(t)^y$: $h_2'(y)<0$ and $h_2''(y)>0$ for all $y$. Again, due to the fact that $c_l,d_l\ge 1$ for all $l$ and $-\lambda\ge 2$, we have
\begin{align*}
    | \frac{1}{c_{x_1}(\lambda+c_{x_1})}\big( r(t)^{c_{x_1}}-r(t)^{-\lambda} \big) | & =\frac{1}{c_{x_1}}\big| \frac{h_2(c_{x_1})-h_2(-\lambda)}{c_{x_1}-(-\lambda)} \big| = -\frac{1}{c_{x_1}} h_2'(c_{x_1}') = -\frac{1}{c_{x_1}}r(t)^{c_{x_1}'}\ln r(t) \\
    | \frac{1}{d_{x_2}(\lambda+d_{x_2})}\big( r(t)^{d_{x_2}}-r(t)^{-\lambda} \big) | & =\frac{1}{d_{x_2}}\big| \frac{h_2(d_{x_2})-h_2(-\lambda)}{d_{x_2}-(-\lambda)} \big| = -\frac{1}{d_{x_2}} h_2'(d_{x_2}') = -\frac{1}{d_{x_2}}r(t)^{d_{x_2}'}\ln r(t),
\end{align*}
where $c_{x_1}'$ is between $c_{x_1}$ and $-\lambda$, $d_{x_2}'$ is between $d_{x_2}$ and $-\lambda$.

Last, according to the expression of $c_l,d_l$ in \eqref{eq:coefficient c d definition}, we have
    \begin{align*}
        c_l &= \sum_{l'} \big( \frac{p(x_2=l'|x_1=l,z)}{p(x_2=l'|x_1=l)} \big)^{w+1} p(x_2=l'|x_1=l)=\exp\bigg( w \D_{1+w}(p(\cdot|x_1=l,z) |p(|x_1=l) )  \bigg),\\
        d_l &= \sum_{l'} \big( \frac{p(x_1=l'|x_2=l,z)}{p(x_1=l'|x_2=l)} \big)^{w+1} p(x_1=l'|x_2=l)=\exp\bigg( w \D_{1+w}(p(\cdot|x_2=l,z) |p(|x_2=l) )  \bigg).
    \end{align*}
For $w\gg 1$, since $c_{x_1}^*,c_{x_1}'$ are between between $c_{x_1}$ and $-\lambda$ and $\ln c_{x_1}=\Theta(w),\ln (-\lambda)=\Theta(w)$, we have $c_{x_1}^*=\Theta(w),c_{x_1}'=\Theta(w)$ for all $x_1$. For the same reason, $d_{x_2}^*=\Theta(w),d_{x_2}'=\Theta(w)$ for all $x_2$. Therefore, if we focus on the order of $w$ for $w\gg 1$ and preserve the leading order terms in $\TV(q_t^{z,w},q_T^{z,w})$, we have
\begin{align*}
    \TV(q_t^{z,w},q_T^{z,w}) &=  \mc{Z} \exp(-\Theta(w)) r(t)^{\exp(\Theta(w))} + r(t)^{\exp(\Theta(w))} \\
    &= \exp(\Theta(w))\exp(-\Theta(w)) r(t)^{\exp(\Theta(w))} + r(t)^{\exp(\Theta(w))},
\end{align*}
where the second identity follows from Remark \ref{rem:1d convergence rate}.
\end{proof}
\subsection{Sampled distributions for $D=2$}\label{append:explicit expression asymp w 2d}
\begin{proof}[Proof of Proposition \ref{prop:2d sampled distribution property}] According to \eqref{eq:sampled distribution 2d} and Assumption \ref{assup:full distribution}, we have
\begin{align*}
    q_T^{z_1,w}(x) &\propto (1/c_{x_1}+1/d_{x_2}) p(x)^{-w} p(x|z_1)^{1+w} \\
    & \propto \left( \underbrace{\big( \frac{a_1 p(x_1|z_1)}{\sum_k a_k p(x_1|z_k)}  \big)^w}_{\text{\rom{1}}} +\underbrace{\big( \frac{a_1 p(x_2|z_1)}{\sum_k a_k p(x_2|z_k)}  \big)^w}_{\text{\rom{2}}} \right) \underbrace{\big(\frac{a_1 p(x|z_1)}{\sum_k a_k p(x|z_k)}\big)^w}_{\text{\rom{3}}} p(x|z_1).
\end{align*}
Each of the terms \rom{1}, \rom{2} and \rom{3} is within the range $[0,1]$ and exponentially dependent to $w$. Therefore, the values of \rom{1}, \rom{2} and \rom{3} affect the sampled distribution significantly when $w$ is large. By evaluating \rom{1}, \rom{2} and \rom{3} in different regions depending on relations between the marginal supports, we express $q_T^{z_1,w}$ as presented in Proposition \ref{prop:2d sampled distribution property}. The last statement in Proposition \ref{prop:2d sampled distribution property} follows from the \textbf{discussion on $q^{z_1,\infty}(\cdot|z_1)$} in this section.
\end{proof}
\noindent\textbf{Effect of guidance on sampled distributions.} According to Proposition \ref{prop:2d sampled distribution property}-(2), $q_T^{z_1,w}$ is defined with different weight-adjustment in $5$ different type of regions. For simplicity, we denote them as 
\begin{align*}
& \mc{R}_1\coloneqq \{x | x\in \mc{X}_1, x_1\in \mc{X}_{1,1}\setminus S_{1,1}, x_2\in \mc{X}_{1,2}\setminus S_{1,2}\} ,\\
&\mc{R}_{2,i}\coloneqq \{x | x\in \mc{X}_1, x_i\in S_{1,i}, x_{\setminus i}\in \mc{X}_{1,\setminus i}\setminus S_{1,\setminus i}\},\qquad i=1,2,\\
&\mc{R}_3\coloneqq \{ x | x\in \mc{X}_1\setminus S_1, x_1\in S_{1,1}, x_2\in S_{1,2}  \},\\
&\mc{R}_4 \coloneqq S_1 .    
\end{align*} 
The above sets reflect different level of ``privacy'' of class $z_1$. $\mc{R}_4$ is the shared region with other classes. $\mc{R}_{1},\mc{R}_{2,i},\mc{R}_3$ are not shared with other classes. But $\mc{R}_3$ has both marginals shared with other classes and $\mc{R}_{2,i}$ has one of the marginals shared with other classes. $\mc{R}_1$ is the most private set in class $z_1$, with no intersection with other classes even for marginals. If we denote the associated weights (before normalization) on different regions by $A^{z_1,w}$ with the corresponding sub-index:
\begin{align*}
    &A^{z_1,w}=2,\\
    &A^{z_1,w}_{2,i}= 1+ (\frac{a_1 p(x_i|z_1)}{\sum_{k\in I_{1,i}} a_k p(x_i|z_k) })^{w}, \qquad i=1,2, \\
    &A^{z_1,w}_3 =  \sum_{i=1}^2(\frac{a_1 p(x_i|z_1)}{\sum_{k\in I_{1,i}} a_k p(x_i|z_k) })^w , \\
    & A^{z_1,w}_4 = \big( \sum_{i=1}^2(\frac{a_1 p(x_i|z_1)}{\sum_{k\in I_{1,i}} a_k p(x_i|z_k) } )^w\big) \big(\frac{a_1 p(x|z_1)}{\sum_{k\in I_{1}} a_k p(x|z_k)}\big)^w,
\end{align*}
we can notice that for all $w\ge 0$, $A^{z_1,w}_1\ge A^{z_1,2}_{2,i} \ge A^{z_1,w}_3\ge A^{z_1,w}_4$. This reflects that \textit{the sampled distribution from the discrete diffusion with CFG can leverage the geometric information of the full data distribution: the sampled distribution puts larger weights on more private regions of class $z_1$}. We conjecture that the above fact is also true in high dimension:
\begin{conjecture}\label{conject:high dimension} For any $D\ge 2$, discrete diffusion with CFG leverages the geometric information from the full data distribution. More specifically, under Assumption \ref{assup:full distribution}, the sampled distribution $q_T^{z_1,w}$ adapts the class distribution $p(\cdot|z_1)$ by putting larger weights on more private regions of class $z_1$, where those regions with different privacy are defined based on the support sets and their marginals. 
\end{conjecture}

\noindent\textbf{Discussion on $q^{z_1,\infty}(\cdot|z_1)$.} Now we look at the structure of the sampled distribution as $w\to\infty$ in further detail. 
According to the expression of weights $A^{z_1,w}$, since $1\in I_{1,i}$ for $i=1,2$ and $1\in I_1$, the rational factors inside the parentheses is in $(0,1]$. In particular, if $S_1\neq \emptyset$, i.e., class $z_1$ has intersected domain with other classes, $|I_1|\ge 2$. Hence $\tfrac{a_1 p(x|z_1)}{\sum_{k\in I_{1}} a_k p(x|z_k)}\in (0,1)$. Therefore, as $w\to \infty$, we have
\begin{align*}
    A^{z_1,\infty}_1 =2 ,\quad A^{z_1,\infty}_{2,i}\in \{ 1,2 \}, \quad  A^{z_1,\infty}_{3}\in \{0, 1,2 \},\quad A^{z_1,\infty}_{4}=0.
\end{align*}
Then, we have $q^{z_1,\infty}(\cdot|z_1)|_{A_4^{z_1,\infty}}=0$, i.e., $\mathrm{Supp}(q^{z_1,\infty}(\cdot|z_1))\subset \mc{X}_1\setminus S_1$. 

It is worth noting that it is possible that some sets among $\mc{R}_1,\mc{R}_{2,i},\mc{R}_3$ could be empty. Therefore, for a general data distribution $p$ satisfying Assumption \ref{assup:full distribution}, in order to derive $q^{z_1,\infty}(\cdot|z_1)$ completely, we need to first identity whether $\mc{R}_1,\mc{R}_{2,i},\mc{R}_3$ are non-empty or not, and then compute the associated limiting weights on the non-empty regions. In the following, we will use a simple example to illustrate this procedure.

\noindent\underline{\textit{An example with $D=2, N=5$.}} We consider the data distribution $p$ is a mixture of two classes with equal weights: $p(x)=\tfrac{1}{2}p(x|z_1)+\tfrac{1}{2}p(x|z_2)$ for all $x\in \{1,2,3,4,5\}^2$ with $5$ being the masked state. The heat maps for $p(\cdot|z_1), p(\cdot|z_2)$ and $p$ are given in Figure \ref{fig:heat map setting}.
\begin{figure}[h]
    \centering    \includegraphics[width=\textwidth]{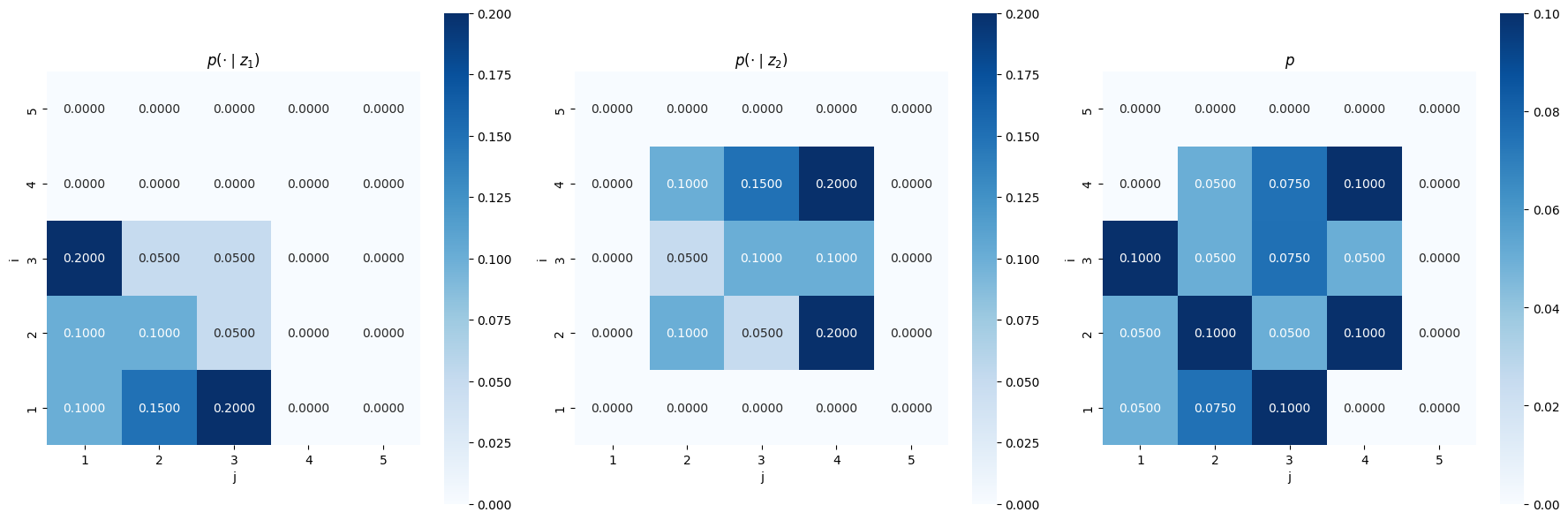}
    \caption{heat maps for $p(\cdot|z_1), p(\cdot|z_2)$ and $p$.}
    \label{fig:heat map setting}
\end{figure}
We can distinguish the regions with different level of privacy based on our formulas. As shown in Figure \ref{fig:regions}, we notice that $\mc{R}_3= \emptyset$ and $\mc{R}_1,\mc{R}_{2,1},\mc{R}_{2,2},\mc{R}_4$ are identified with different colors.
\begin{figure}[h]
    \centering
    \includegraphics[width=0.5\linewidth]{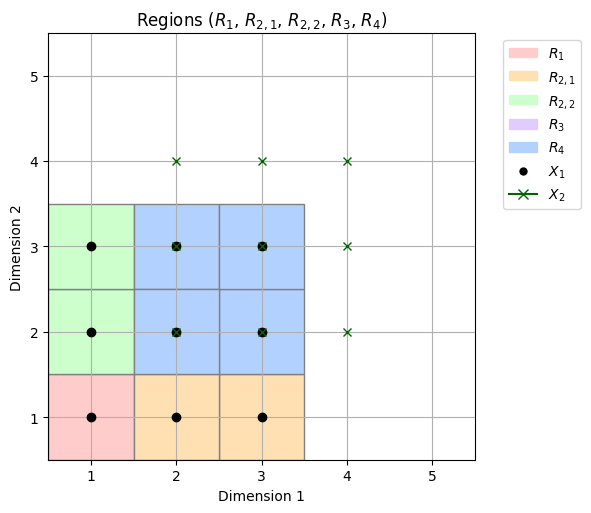}
    \caption{identification of different regions.}
    \label{fig:regions}
\end{figure}
Based on the information of $p$, we can compute the limiting weights as $w\to\infty$. We have
\begin{align*}
     A^{z_1,\infty}_1 =2 ,\quad A^{z_1,\infty}_{2,1}=1_{x_1=1}, \quad  A^{z_1,\infty}_{2,2}=1_{x_2=1},\quad A^{z_1,\infty}_{4}=0.
\end{align*}
Therefore, the sampled distribution $q_T^{z_1,\infty}(\cdot|z_1)$ adapts $p(\cdot|z_1)$ by putting these weights on the 4 regions respectively, i.e.,
\begin{align*}
    q_{T}^{z_1,\infty}(x|z_1)\propto \left\{
    \begin{aligned}
        & 2 p(x|z_1) , \qquad & x\in \mc{R}_1=\{(1,1)\} ,\\
        &  p(x|z_1), \qquad & x\in \mc{R}_{2,1}=\{(2,1),(3,1))\},\\
        &  p(x|z_1), \qquad & x\in \mc{R}_{2,2}=\{(1,2),(1,3))\},\\
        & 0,\qquad &\text{otherwise},
    \end{aligned}\right.
\end{align*}
which implies that $q_T^{z_1,\infty}(1,1|z_1)= q_T^{z_1,\infty}(1,3|z_1)=q_T^{z_1,\infty}(3,1|z_1)=4/17$, $q_T^{z_1,\infty}(1,2|z_1)=q_T^{z_1,\infty}(2,1|z_1)=3/17$ and $q_T^{z_1,\infty}(x_1,x_2|z_1)=0$ otherwise. In Figure \ref{fig:sampled distributions guidance}, we present the heatmaps for the class distribution of $z_1$, the tilted distributions and the sampled distributions with $w=1,5,15$. We can observe the following facts that match our theory.
\begin{itemize}
    \item [(1)] the sampled distribution deviates from the tilted distribution for all $w>0$.
    \item [(2)] the effects of guidance differ in different regions: as $w$ increases, the probability mass decreases in $S_1=\{(2,2),(2,3),(3,2), (3,3)\}$; the probability mass increases in regions $\mc{R}_{2,1}=\{(2,1),(3,1)\}$ and $\mc{R}_{2,2}=\{(1,2),(1,3)\}$ at the same rate; the probability mass increases in the region $\mc{R}_1=\{1,1\}$ at the largest rate. 
    \item [(3)] for large guidance ($w=15$), the sampled distribution $q_T^{z_1,w}$ can be approximately understood as $q_T^{z_1,\infty}(\cdot|z_1)$. The last plot in Figure \ref{fig:sampled distributions guidance} matches our computation for $q_T^{z_1,\infty}(\cdot|z_1)$.
    \item [(4)] for small guidance ($w=1$), the effect of guidance is also small. The sampled distribution $q_T^{z_1,w}$ deviates a little bit from the target distribution $p(\cdot|z_1)$ in the way we described in (2). 
\end{itemize}
In practice, people observe that the optimal guidance is usually positive but small (of order $\Theta(1)$). Our theory and numerical observations bring insights in understanding the optimal guidance. Roughly speaking, if we can show that the effects of guidance presented above actually compensate the effect of score approximation, by quantifying the inductive bias in learning the scores, we can rigorously analyze the optimal guidance in the CFG setting. This will be left as an interesting future work to explore.
\begin{figure}[H]
    \centering
    \includegraphics[width=\linewidth]{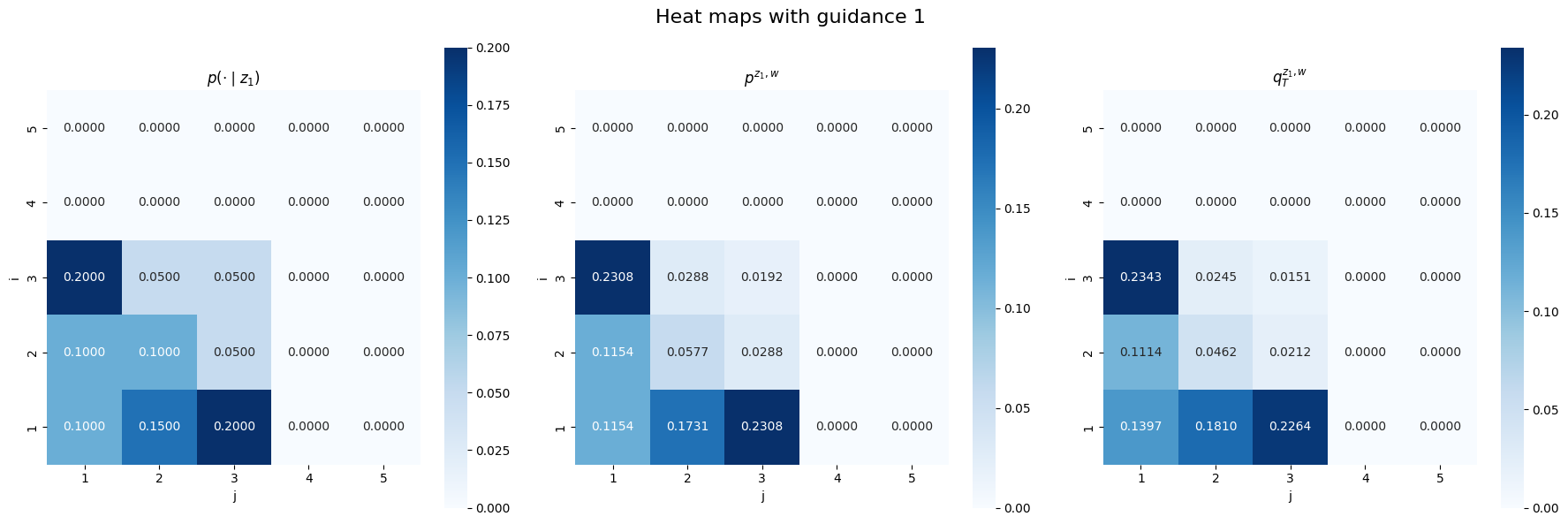}
    \includegraphics[width=\linewidth]{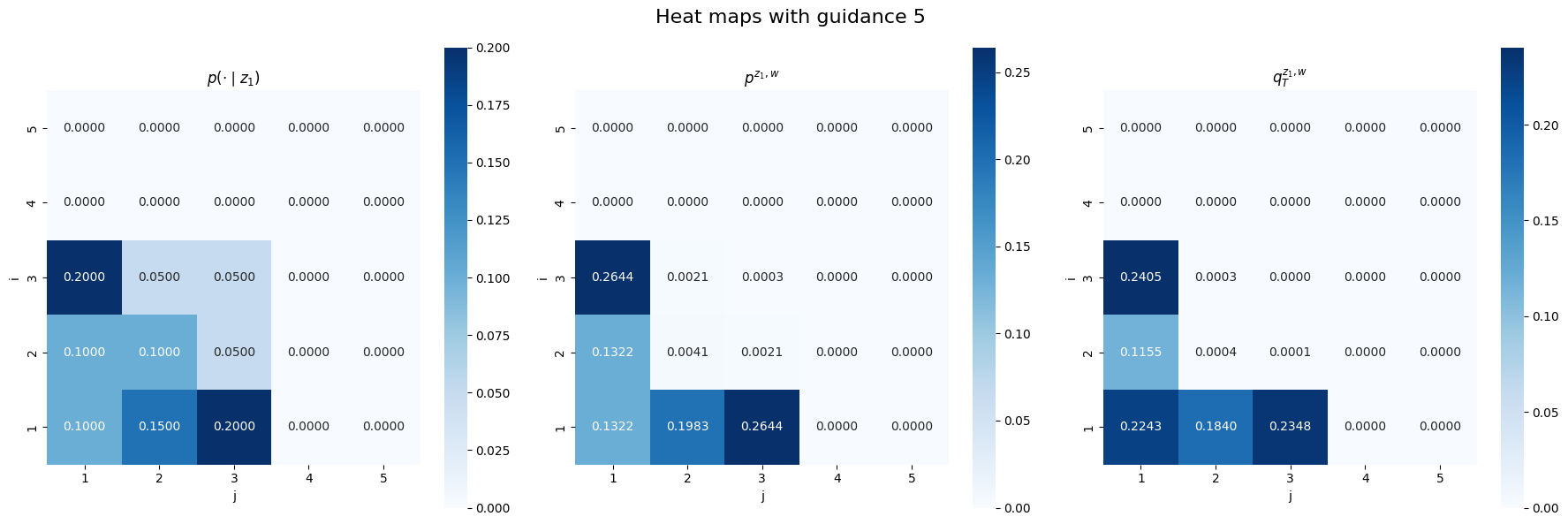}
    \includegraphics[width=\linewidth]{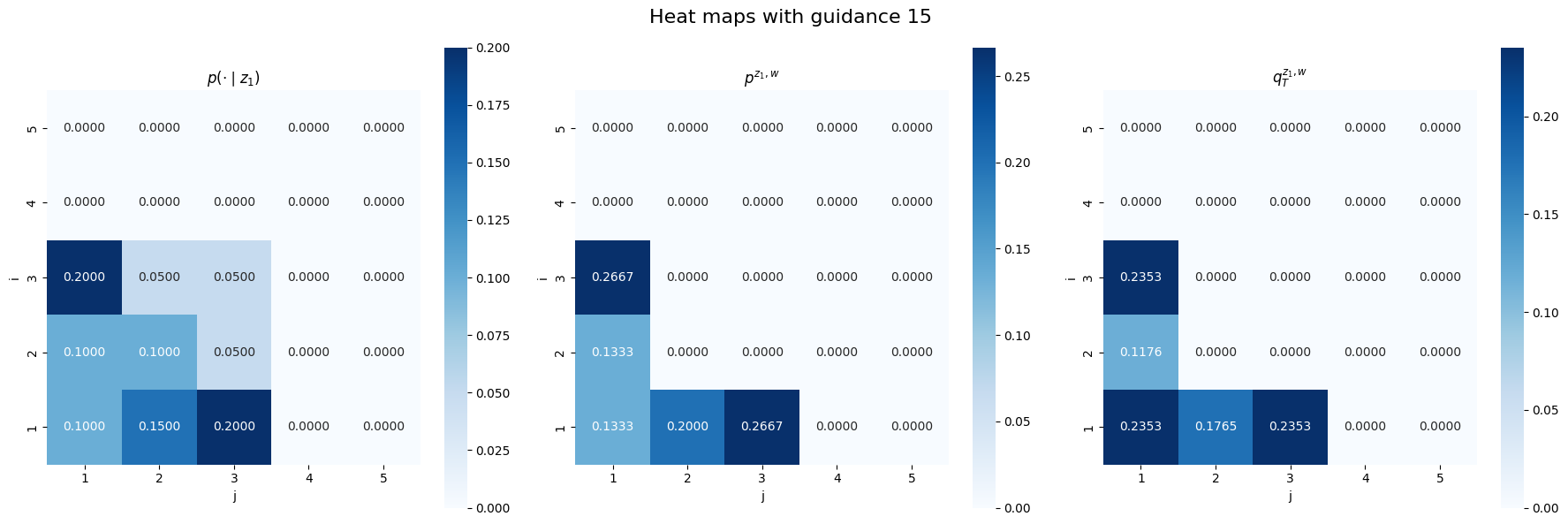}
    \caption{distributions under different guidance strengths: $w=1,5,15$. The first column presents the class distribution of $z_1$. The second column presents the tilted distributions. The third column presents the sampled distributions which are obtained using exact evaluations of scores and integrals. }
    \label{fig:sampled distributions guidance}
\end{figure}
\section{Numerical Experiments}\label{append:numerical examples}
\subsection{Details on $1$D experiment}
We consider each cluster to be defined by the following vector:
\[ (0.1, 0.2, 0.4, 0.2, 0.1)\]
We consider two classes, each containing two of the clusters above. We consider a mixture of both classes with equal weight assigned to each class. 

\noindent\textbf{Disjoint Example:} We plot the class conditional and full probability distributions for the disjoint example in Figure \ref{fig:disjoint-example}.
\begin{figure}[H]
    \centering   
    \begin{subfigure}[b]{0.3\textwidth} \includegraphics[width=\linewidth]{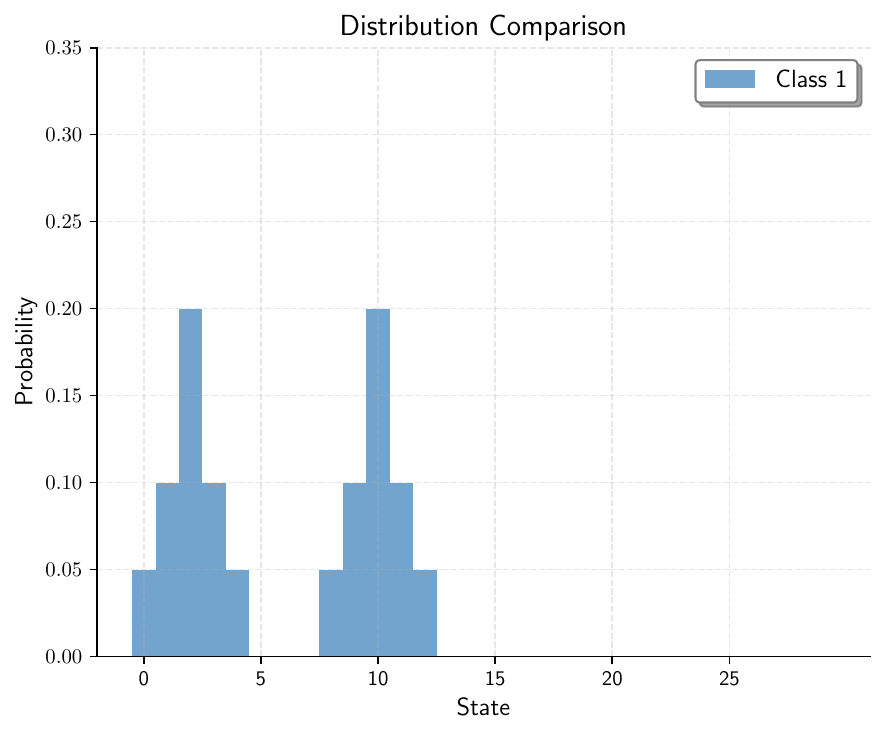}
        \caption{Histogram corresponding to class $1$}
    \end{subfigure}
    \hfill
    \begin{subfigure}[b]{0.3\textwidth}        \includegraphics[width=\linewidth]{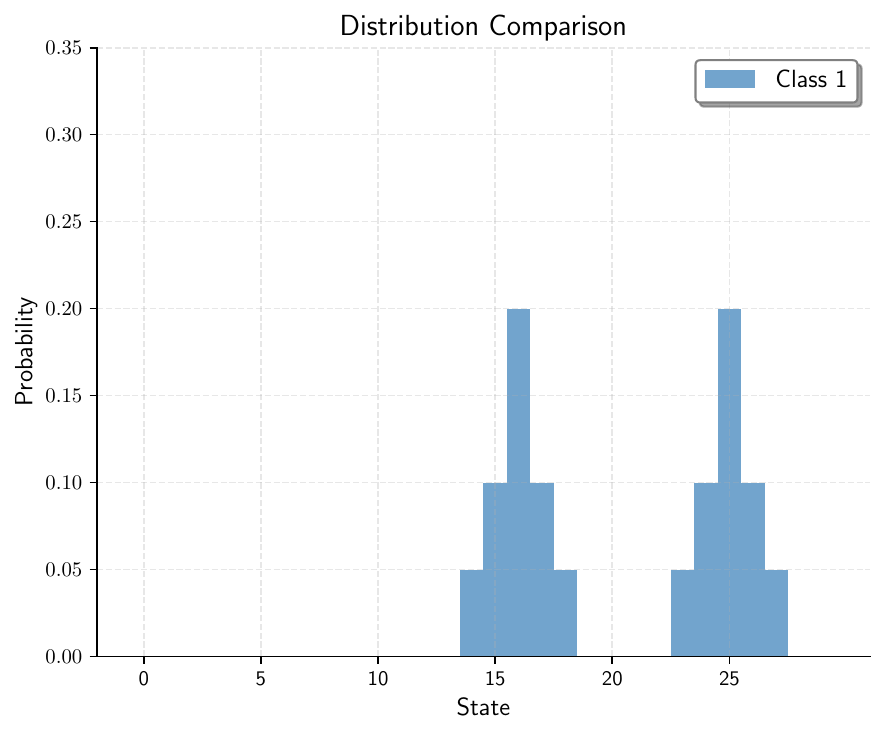}
        \caption{Histogram corresponding to class $2$}
    \end{subfigure}
    \hfill
    \begin{subfigure}[b]{0.3\textwidth}      \includegraphics[width=\linewidth]{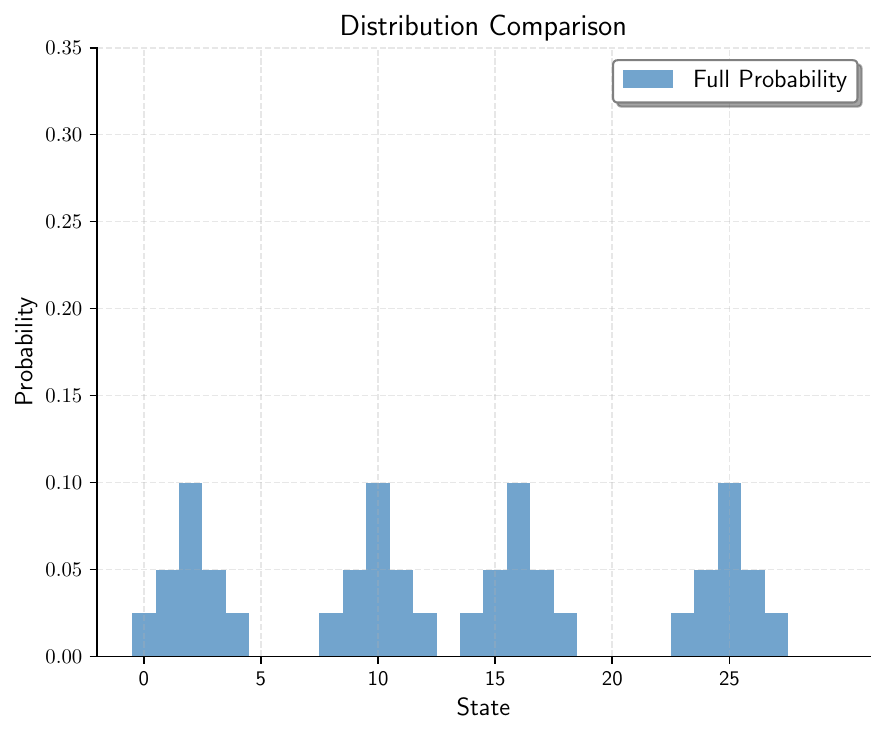}
        \caption{Histogram corresponding to the full probability}
    \end{subfigure}
    \hfill
    \caption{Histograms corresponding to the disjoint example.}
    \label{fig:disjoint-example}
\end{figure}

\noindent\textbf{Intersection Example:} We pull the classes together to create a region of intersection. We plot the class conditional and full probability distributions for the intersection example in Figure \ref{fig:intersection-example}.
\begin{figure}[H]
    \centering   
    \begin{subfigure}[t]{0.3\textwidth}     \includegraphics[width=\linewidth]{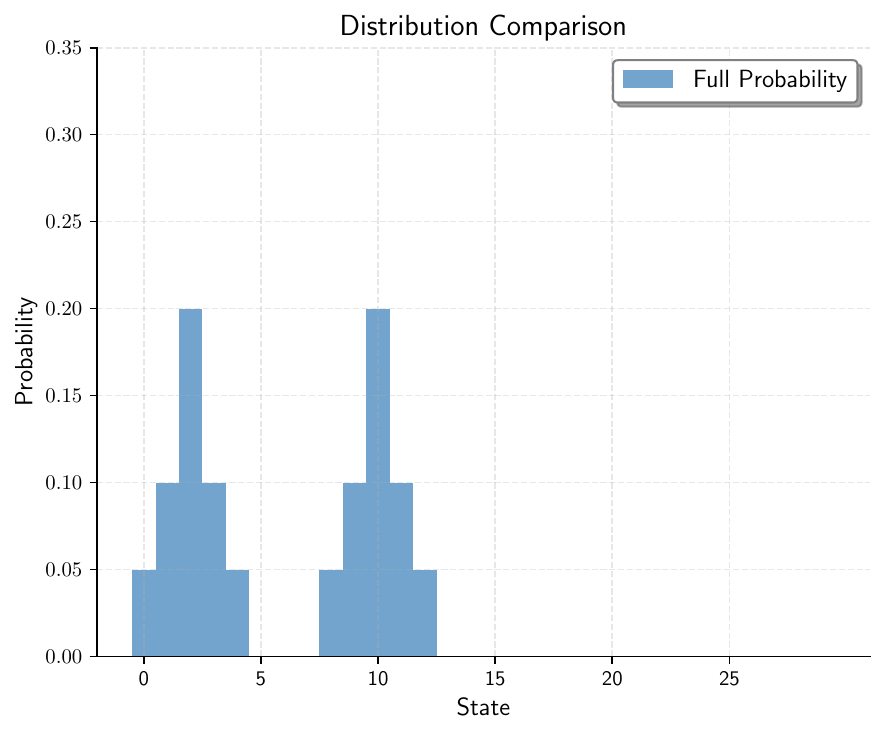}
        \caption{Histogram corresponding to class $1$}
    \end{subfigure}
    \hfill
    \begin{subfigure}[t]{0.3\textwidth}
        \includegraphics[width=\linewidth]{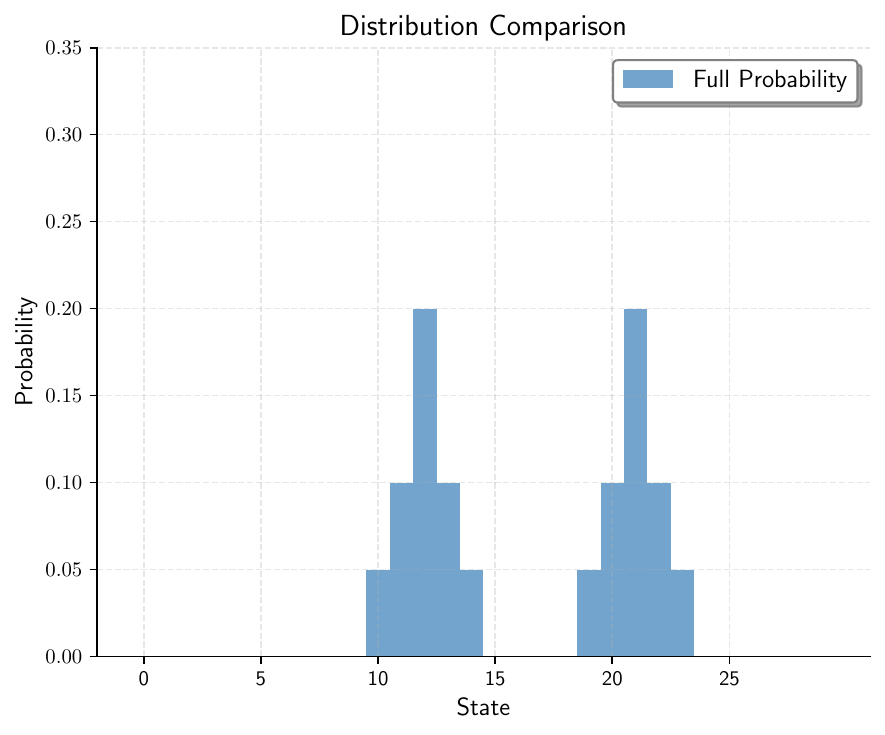}
        \caption{Histogram corresponding to class $2$}
    \end{subfigure}
    \hfill
    \begin{subfigure}[t]{0.3\textwidth}
        \includegraphics[width=\linewidth]{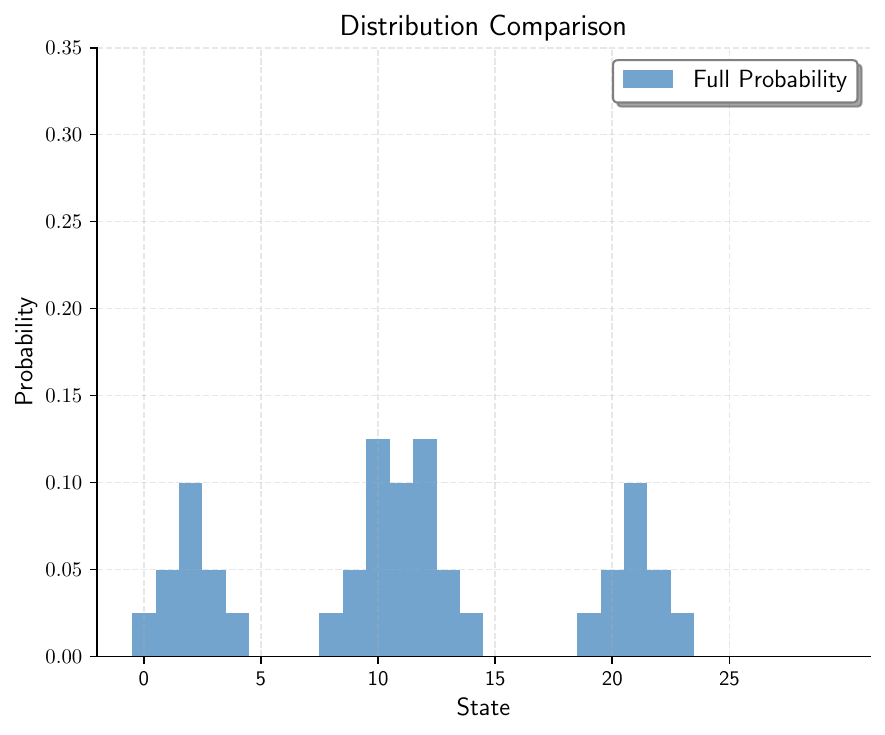}
        \caption{Histogram corresponding to the full probability}
    \end{subfigure}
    \hfill
    \caption{Histograms corresponding to the intersection example.}
    \label{fig:intersection-example}
\end{figure} 

\subsection{Details on $2$D experiment}
\textbf{Disjoint Example:} We plot the class conditional and full probability distributions for the disjoint example in Figure \ref{fig:2d-disjoint-example}.
\begin{figure}[H]
    \centering
    
    \begin{subfigure}[t]{0.3\textwidth}
        \includegraphics[width=\linewidth]{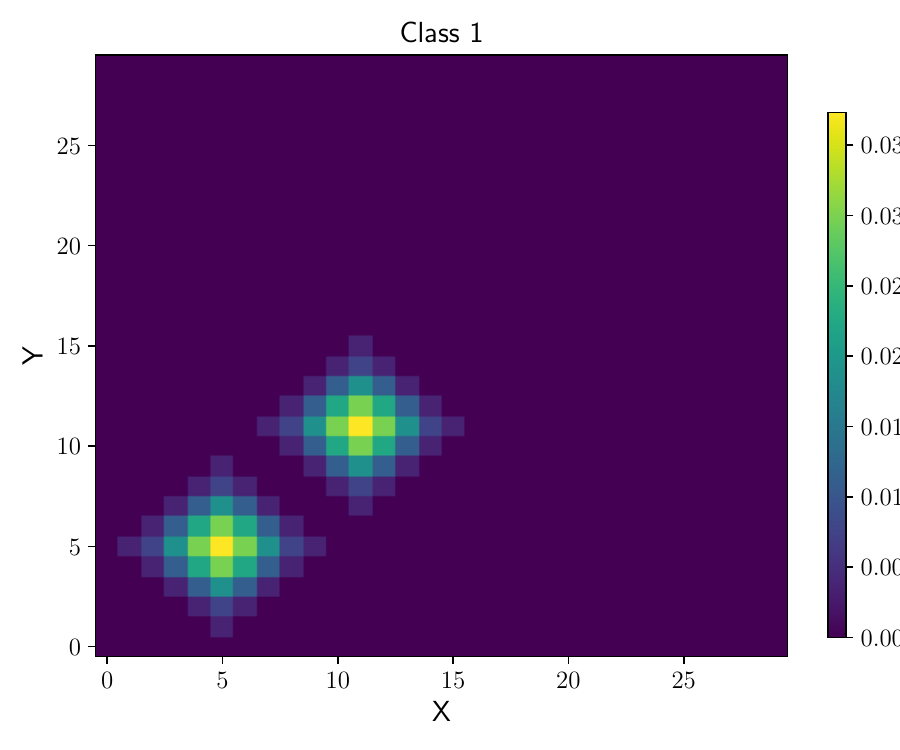}
        \caption{Heat plot corresponding to class $1$}
    \end{subfigure}
    \hfill
    \begin{subfigure}[t]{0.3\textwidth}
        \includegraphics[width=\linewidth]{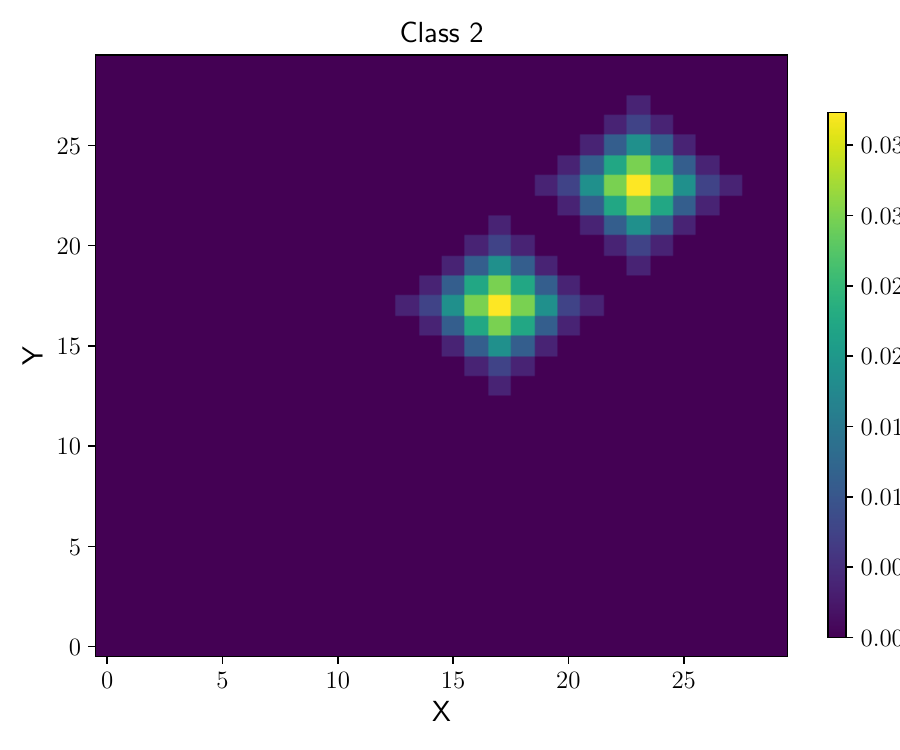}
        \caption{Heat  plot corresponding to class $2$}
    \end{subfigure}
    \hfill
    \begin{subfigure}[t]{0.3\textwidth}
        \includegraphics[width=\linewidth]{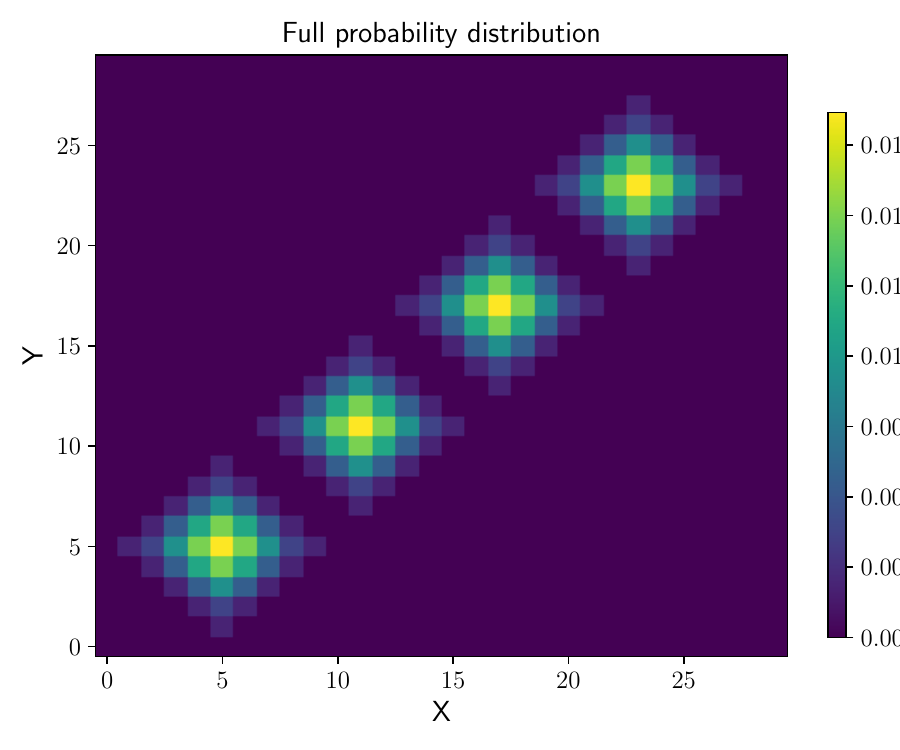}
        \caption{Heat plot corresponding to the full probability}
    \end{subfigure}
    \hfill
    \caption{Heat plot corresponding to the disjoint example.}
    \label{fig:2d-disjoint-example}
\end{figure} 

\noindent\textbf{Intersection Example:} We pull the classes together to create a region of intersection. We plot the class conditional and full probability distributions for the intersection example in Figure \ref{fig:2d-intersection-example}.
\begin{figure}[h!]
    \centering
    \begin{subfigure}[t]{0.3\textwidth}
        \includegraphics[width=\linewidth]{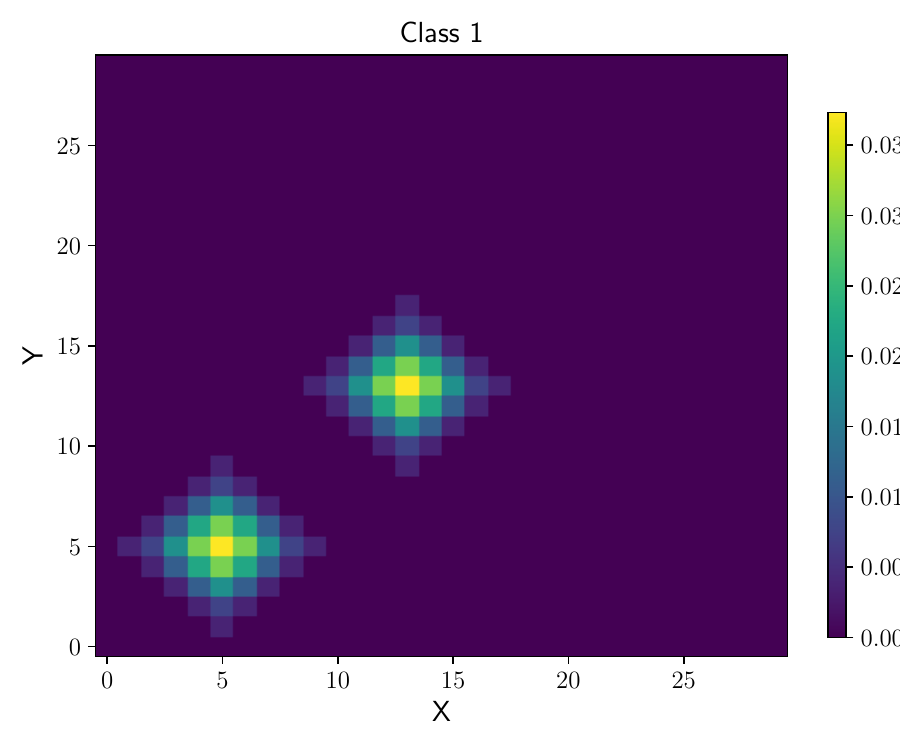}
        \caption{Heat plot corresponding to class $1$}
    \end{subfigure}
    \hfill
    \begin{subfigure}[t]{0.3\textwidth}
        \includegraphics[width=\linewidth]{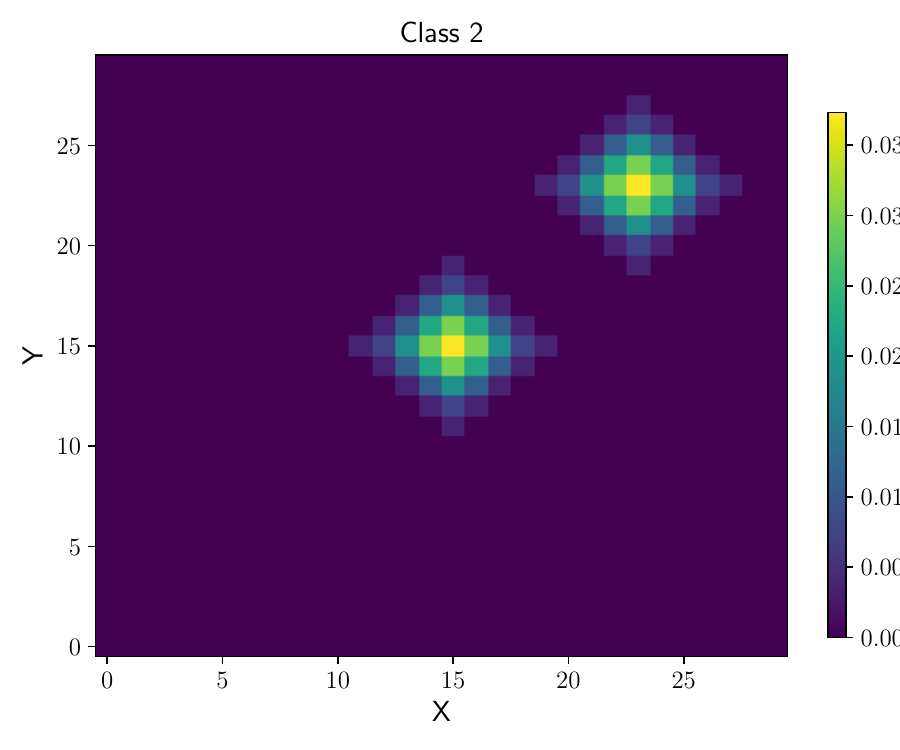}
        \caption{Heat  plot corresponding to class $2$}
    \end{subfigure}
    \hfill
    \begin{subfigure}[t]{0.3\textwidth}
        \includegraphics[width=\linewidth]{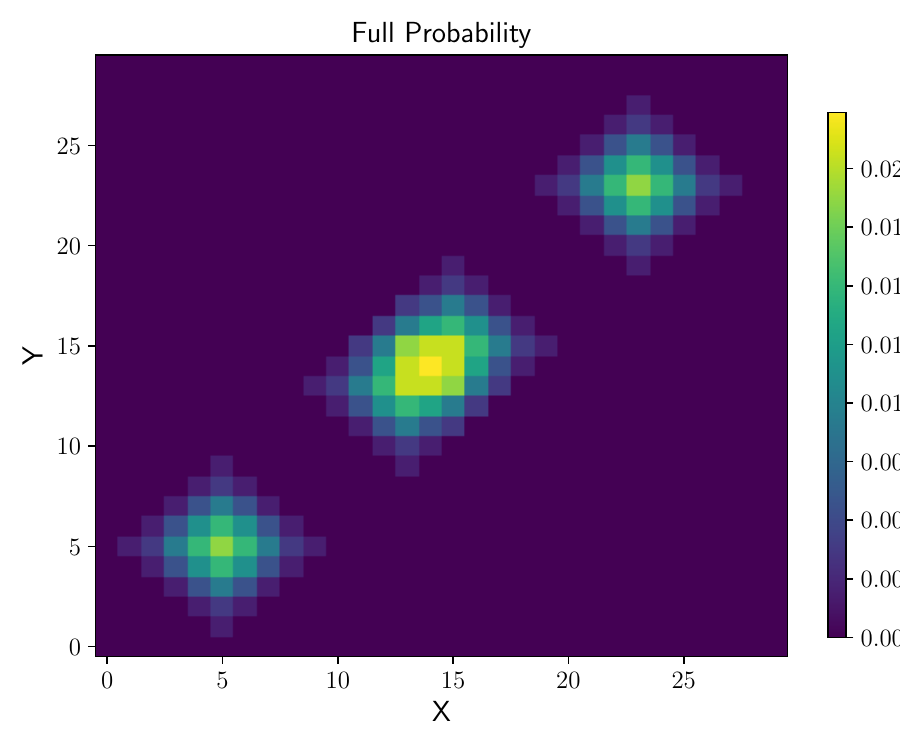}
        \caption{Heat plot corresponding to the full probability}
    \end{subfigure}
    \hfill
    \caption{Heat plot corresponding to the intersection example.}
    \label{fig:2d-intersection-example}
\end{figure} 
\subsection{Experiments in $5$D}
\textbf{Defining the distribution}. We define a 5-dimensional Gaussian mixture distribution with three components to induce structured overlaps along specific dimensions. The probability density function is given by:

\[
p({x}) = \sum_{k=1}^3 \pi_k \, \mathcal{N}({x} \mid {\mu}_k, {\Sigma}_k),
\]

where:
\begin{itemize}
  \item The mixture weights are \( {\pi} = [0.4, 0.3, 0.3] \).
  \item The component means are:
  \[
  \begin{aligned}
  {\mu}_1 &= [0,\, 0,\, 0,\, 0,\, 0]^\intercal, \\
  {\mu}_2 &= [2,\, 2,\, 0,\, -1,\, -1]^\intercal, \\
  {\mu}_3 &= [-2,\, -2,\, 0,\, 1,\, 1]^\intercal.
  \end{aligned}
  \]
  \item The covariance matrices are diagonal and given by:
  \[
  \begin{aligned}
  {\Sigma}_1 &= \operatorname{diag}\left(\frac{0.8}{0.4},\, \frac{0.8}{0.4},\, \frac{0.5}{0.4},\, \frac{0.5}{0.4},\, \frac{0.5}{0.4} \right) ,\\
  {\Sigma}_2 &= {\Sigma}_3 = \operatorname{diag}\left(\frac{0.8}{0.8},\, \frac{0.8}{0.8},\, \frac{0.5}{0.8},\, \frac{0.5}{0.8},\, \frac{0.5}{0.8} \right).
  \end{aligned}
  \]
\end{itemize}
This construction ensures that all components overlap along dimensions $3–5$ while differing significantly along dimensions $1$ and $2$, allowing for structured ambiguity in a subspace of the input. We then generate a discrete distribution by looking at a grid of $10$ points per side on the interval $[-3,3]$. After evaluating on these grids, we generate a tensor and normalize to create the distribution in our discrete space. 

\noindent\textbf{Experiment setting.} We generate $10$K samples and plot several marginals for each class distribution, each of the associated conditional generated distributions with guidance $w=1$, $w=3$, and each of the associated unconditional generated distributions. Our results show that as we increase the guidance strength, the probability mass in the intersection region decreases in all the marginal plots. These numerical results support our Conjecture \ref{conject:high dimension} for $D\ge 2$ in Section \ref{append:explicit expression asymp w 2d}. 

\noindent\textbf{Numerical results.} We first generate $10$K samples and plot several marginals for each class on Figures \ref{fig:5d-class-0}, \ref{fig:5d-class-1}, \ref{fig:5d-class-2} and the unconditional distribution in \ref{fig:5d-class--1}. 

\begin{figure}[H]
    \centering
    \includegraphics[width=\linewidth]{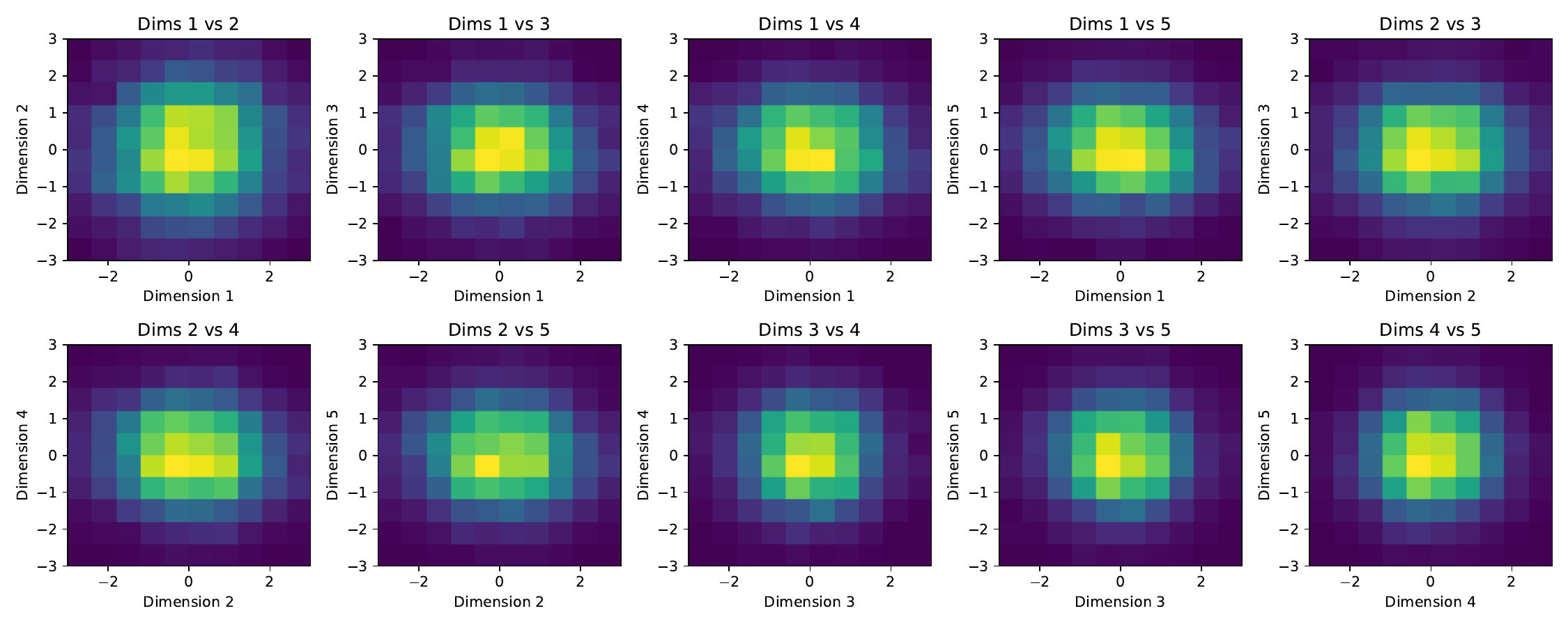}
    \caption{Class 0}
    \label{fig:5d-class-0}
\end{figure} 
\vspace{-.5cm}
\begin{figure}[H]
    \centering
    \includegraphics[width=\linewidth]{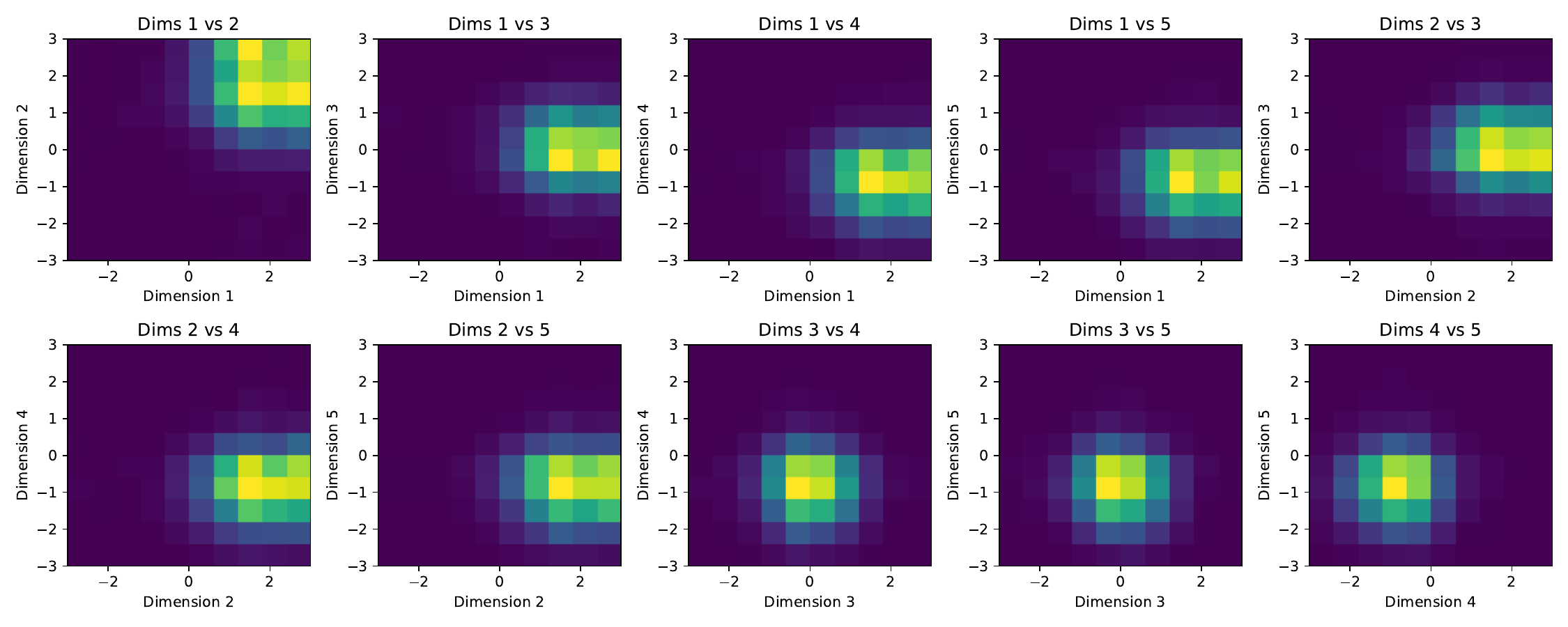}
    \caption{Class 1}
    \label{fig:5d-class-1}
\end{figure} 
\vspace{-.5cm}
\begin{figure}[H]
    \centering
    \includegraphics[width=\linewidth]{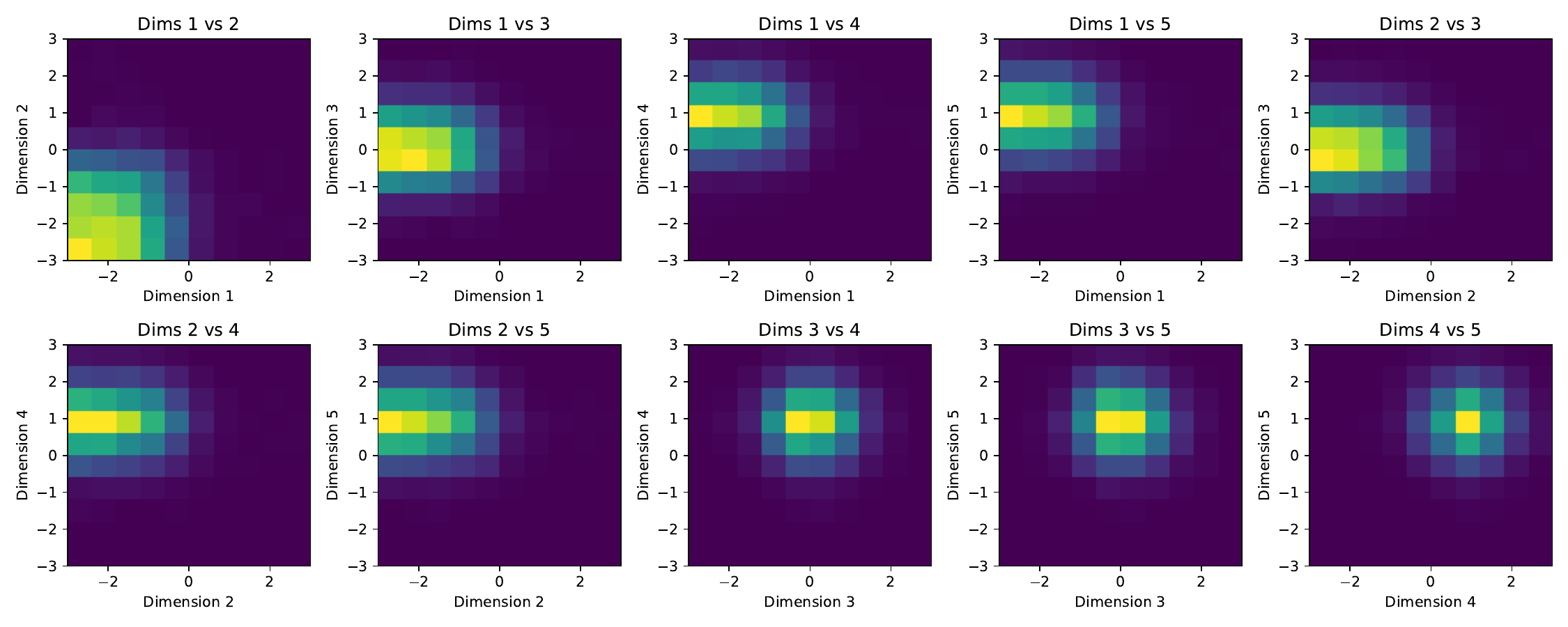}
    \caption{Class 2}
    \label{fig:5d-class-2}
\end{figure} 
\vspace{-.5cm}
\begin{figure}[H]
    \centering
    \includegraphics[width=\linewidth]{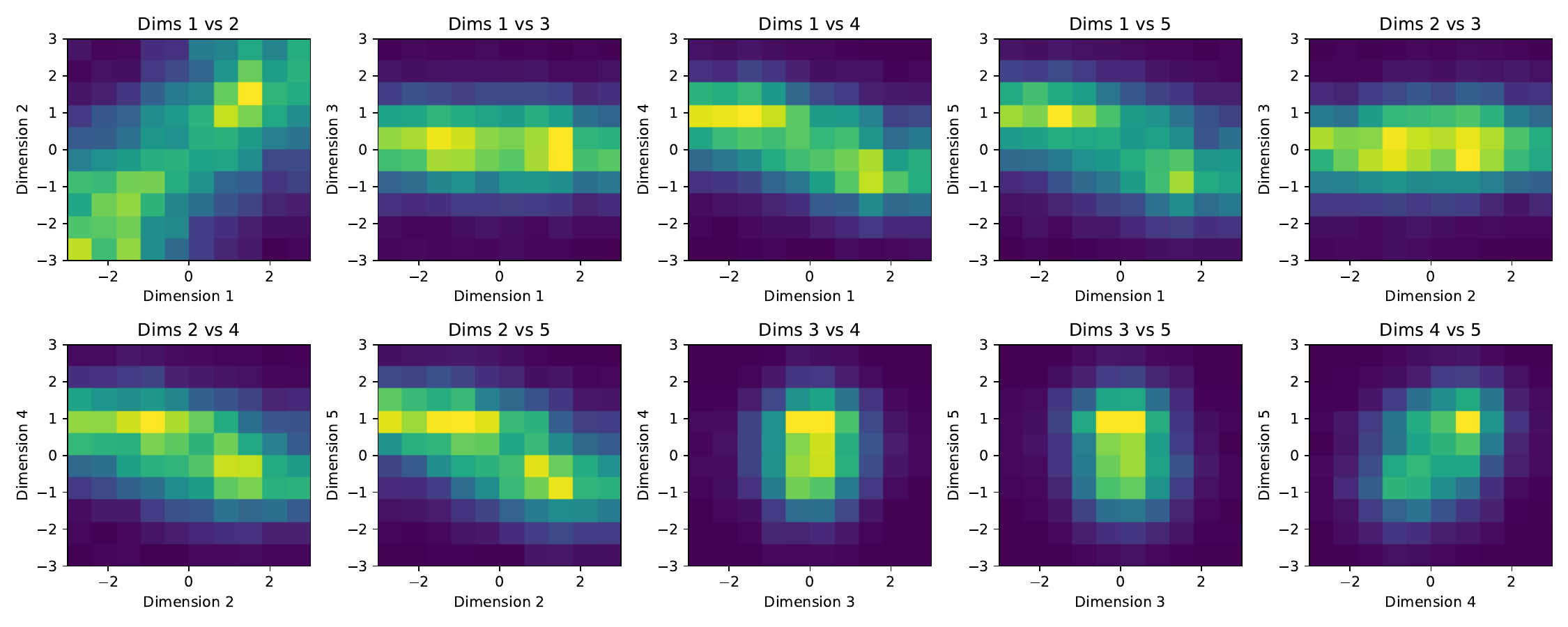}
    \caption{Unconditional Generation}
    \label{fig:5d-class--1}
\end{figure} 
We now generate some samples using guidance $w=1$ in Figures \ref{fig:5d-class-0-guid-2}, \ref{fig:5d-class-1-guid-2}, \ref{fig:5d-class-2-guid-2} and \ref{fig:5d-class--1-guid-2}. 

\begin{figure}[H]
    \centering
    \includegraphics[width=\linewidth]{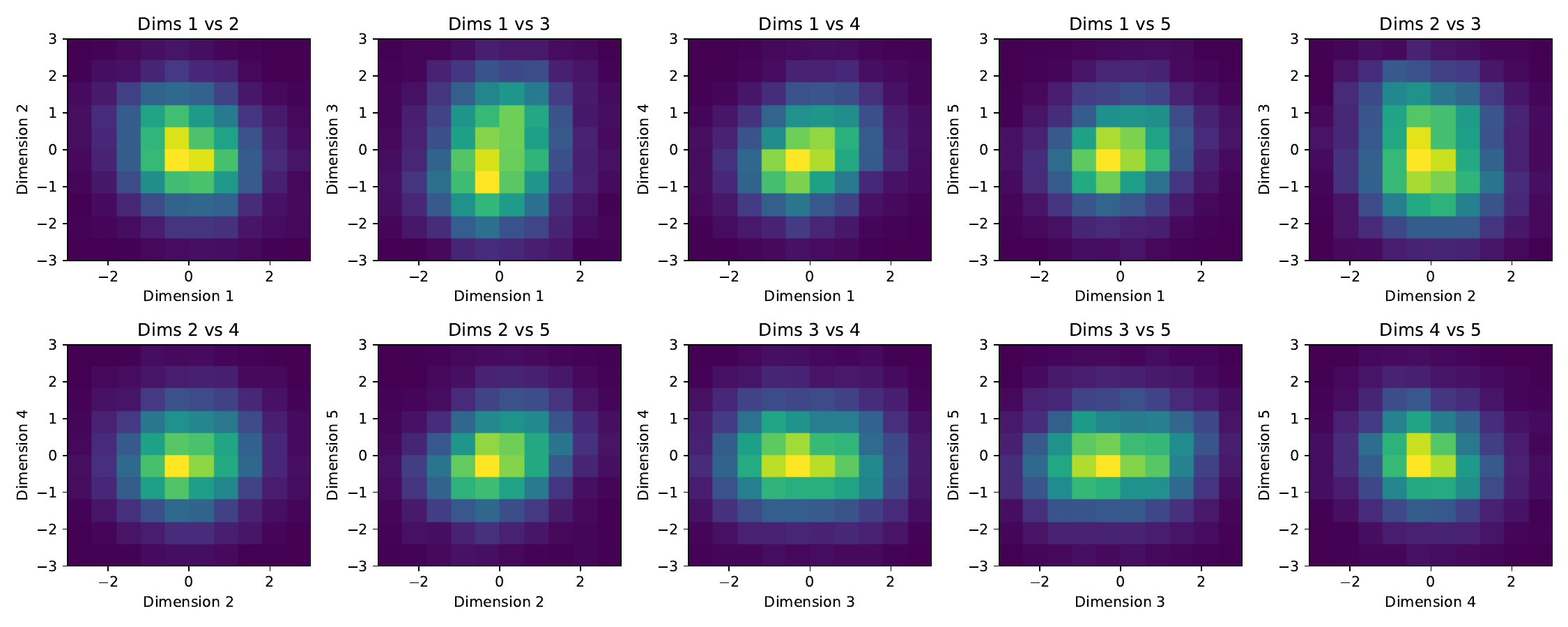}
    \caption{Class 0 with $w=1$}
    \label{fig:5d-class-0-guid-2}
\end{figure} 

\begin{figure}[H]
    \centering
    \includegraphics[width=\linewidth]{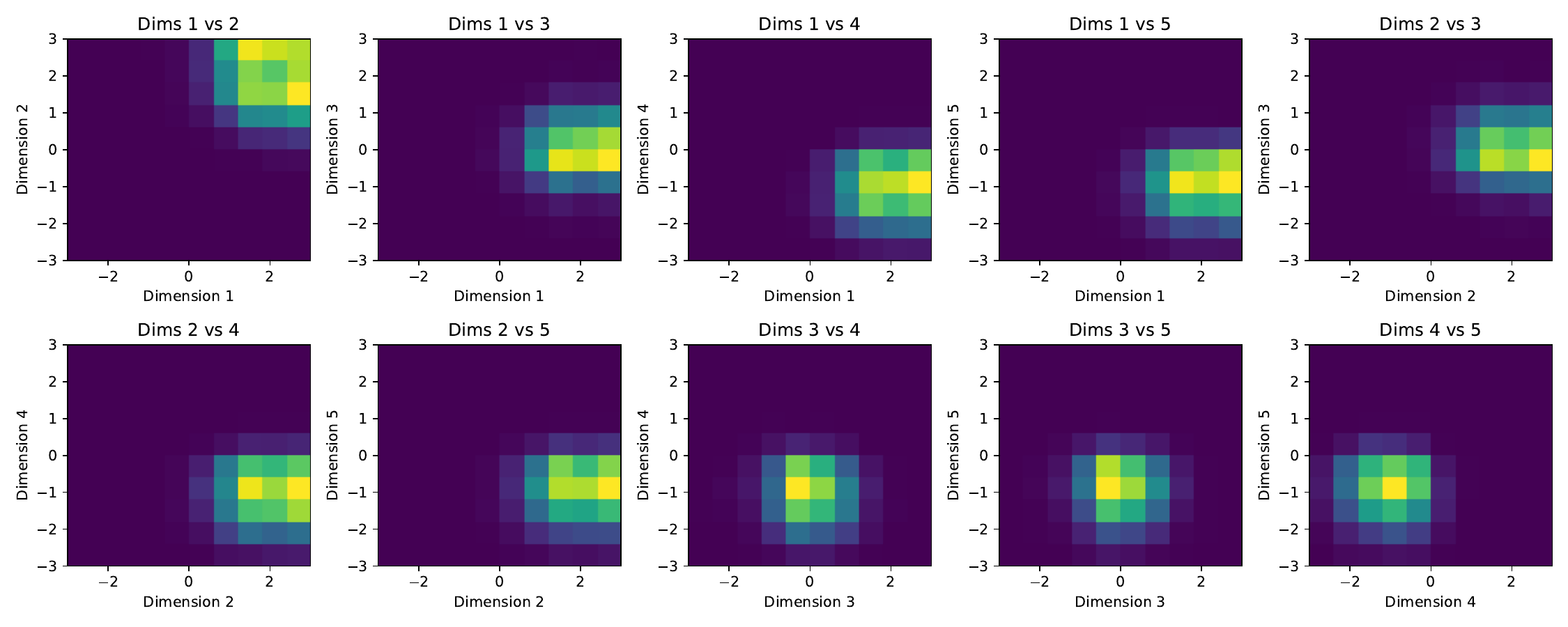}
    \caption{Class 1 with $w=1$}
    \label{fig:5d-class-1-guid-2}
\end{figure} 
\begin{figure}[H]
    \centering
    \includegraphics[width=\linewidth]{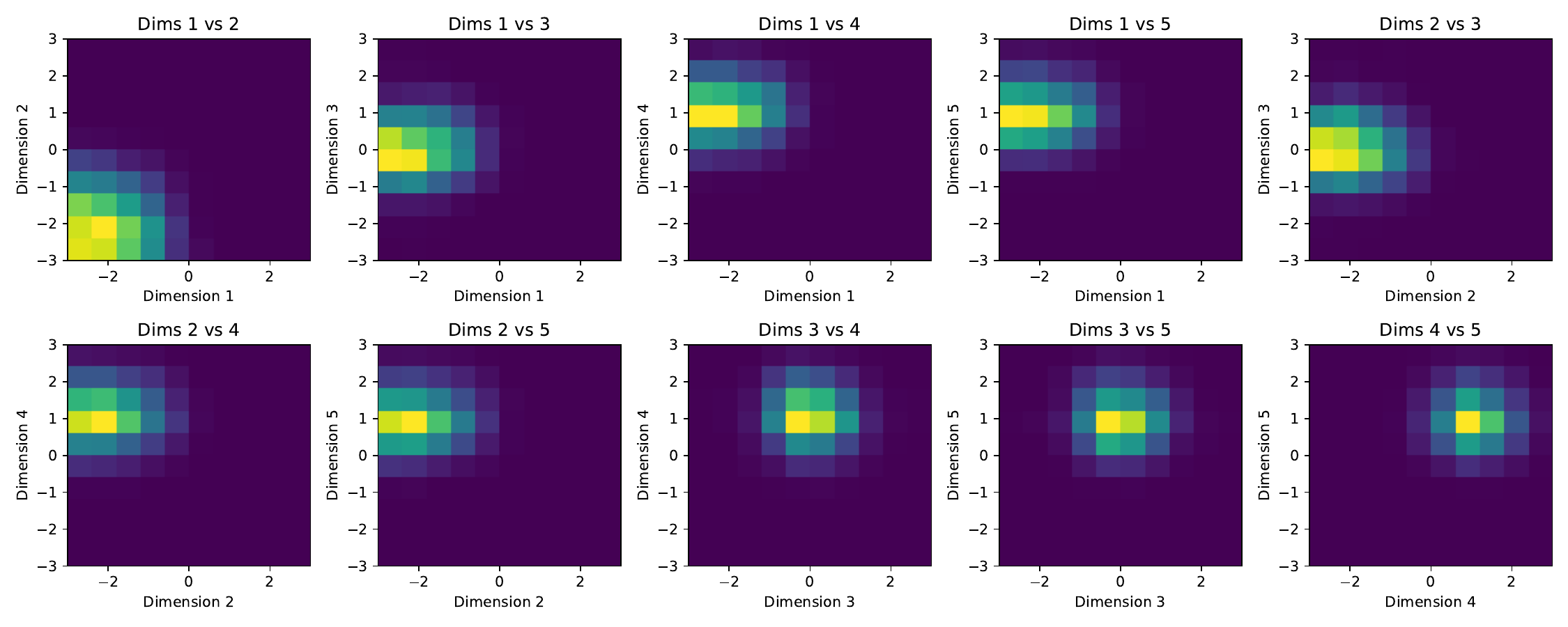}
    \caption{Class 2 with $w=1$}
    \label{fig:5d-class-2-guid-2}
\end{figure} 
\begin{figure}[H]
    \centering
    \includegraphics[width=\linewidth]{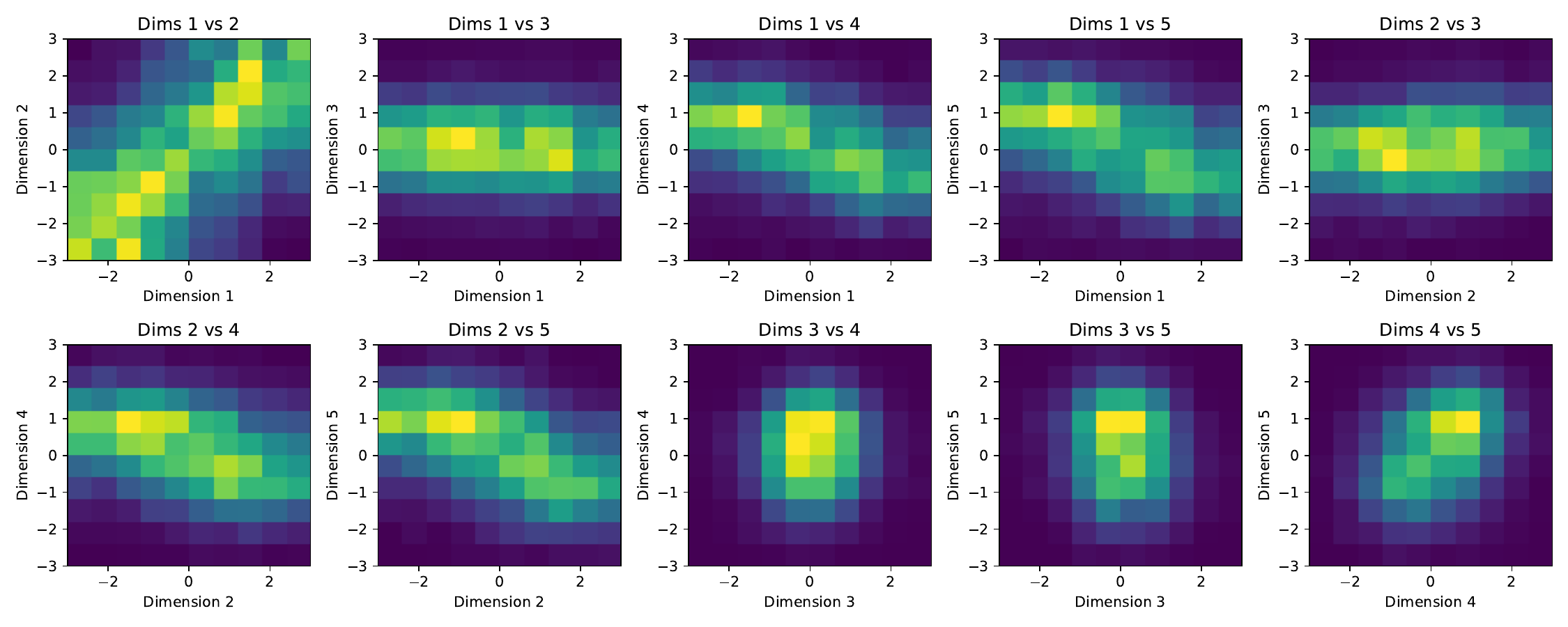}
    \caption{Unconditional Generation with $w = 1$}
    \label{fig:5d-class--1-guid-2}
\end{figure} 

\newpage
We now generate some samples using guidance $w=3$ in Figures \ref{fig:5d-class-0-guid-4}, \ref{fig:5d-class-1-guid-4}, \ref{fig:5d-class-2-guid-4} and \ref{fig:5d-class--1-guid-4}. Observe how probability mass decreases in the intersection region.

\begin{figure}[H]
    \centering
    \includegraphics[width=\linewidth]{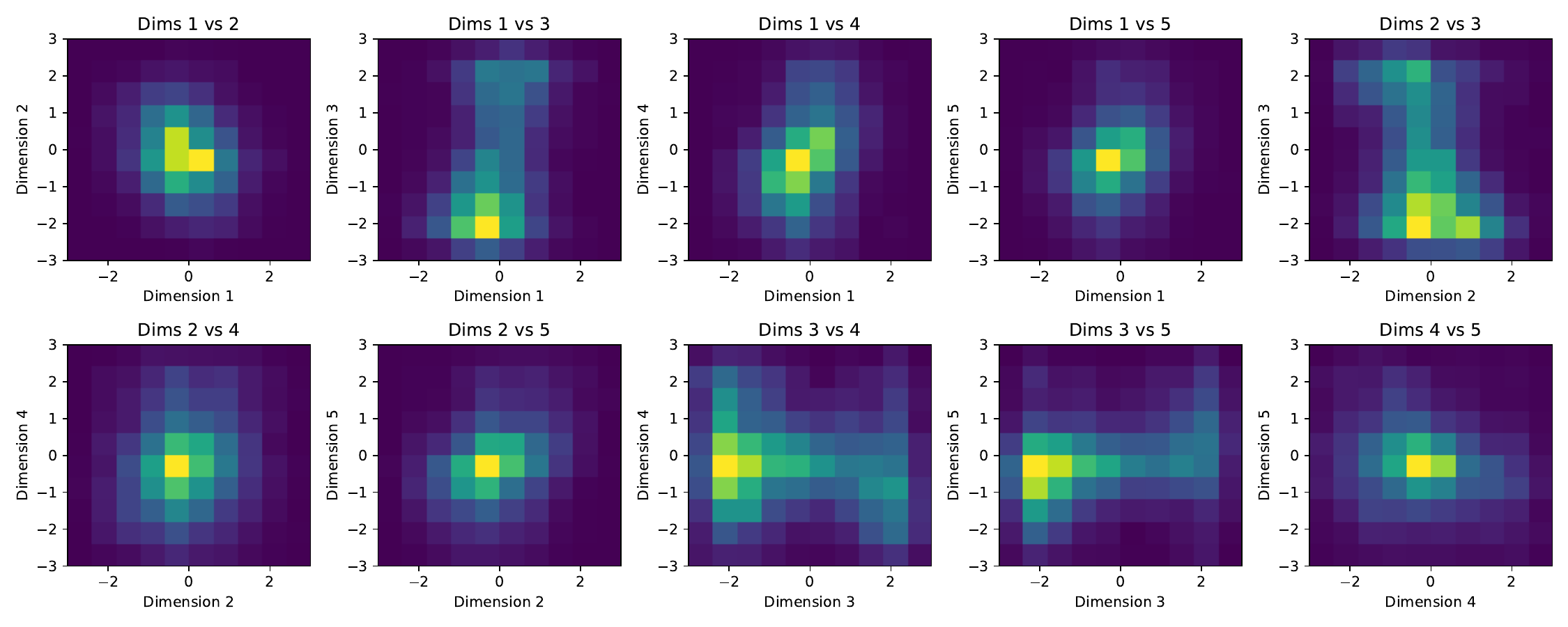}
    \caption{Class 0 with $w=3$}
    \label{fig:5d-class-0-guid-4}
\end{figure} 

\begin{figure}[H]
    \centering
    \includegraphics[width=\linewidth]{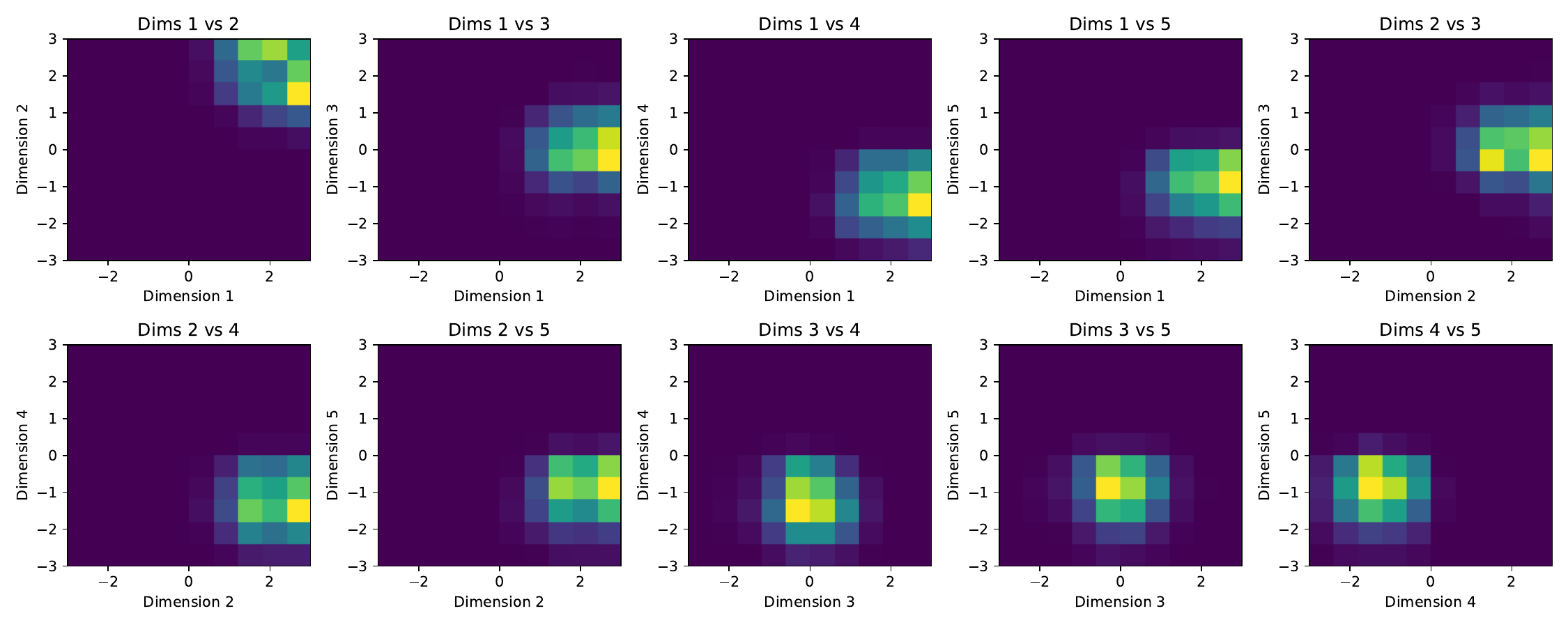}
    \caption{Class 1 with $w=3$}
    \label{fig:5d-class-1-guid-4}
\end{figure} 
\begin{figure}[H]
    \centering
    \includegraphics[width=\linewidth]{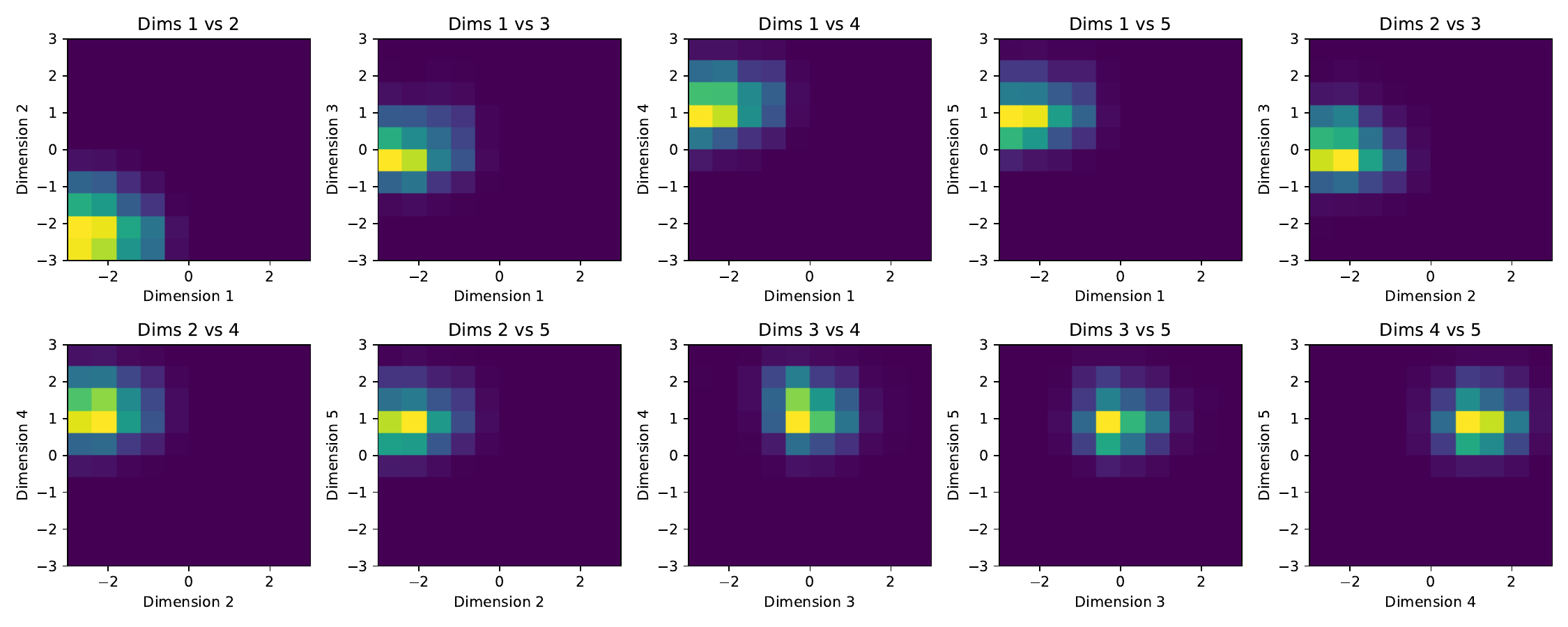}
    \caption{Class 2 with $w=3$}
    \label{fig:5d-class-2-guid-4}
\end{figure} 
\begin{figure}[H]
    \centering
    \includegraphics[width=\linewidth]{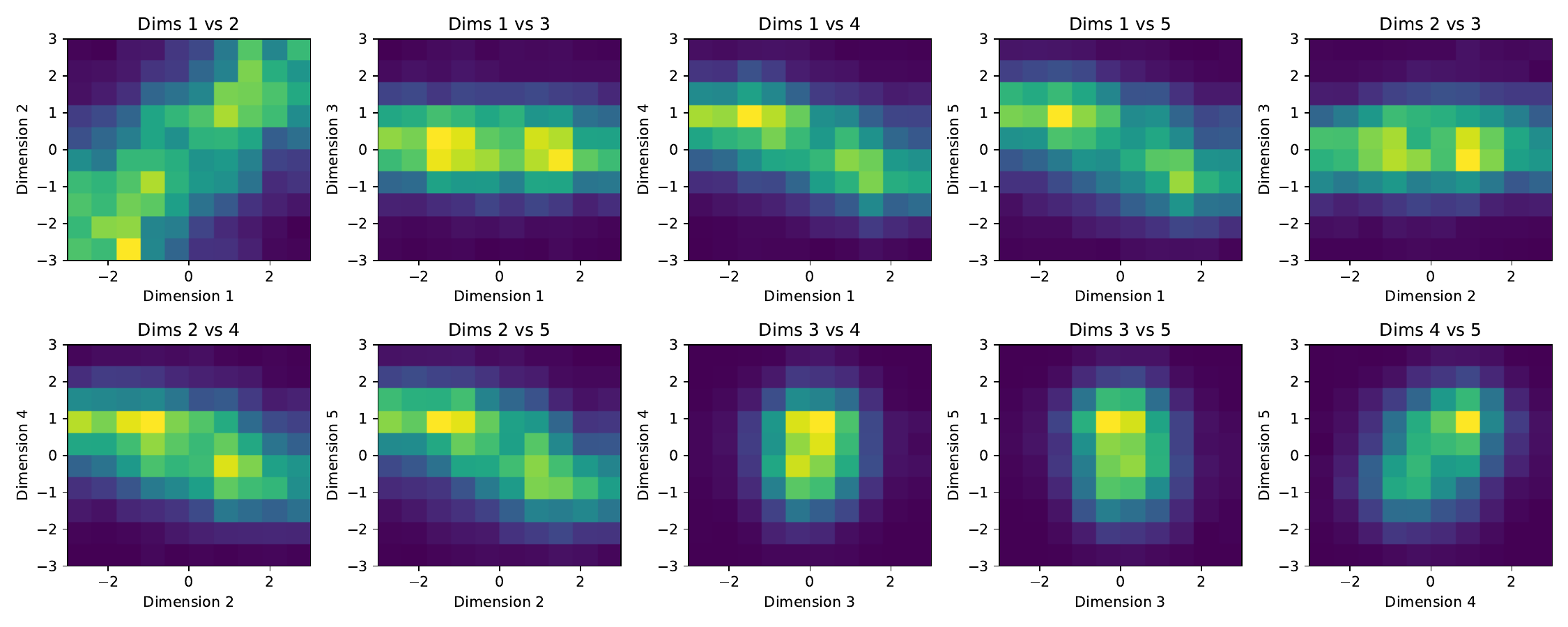}
    \caption{Unconditional Generation with $w = 3$}
    \label{fig:5d-class--1-guid-4}
\end{figure} 
\newpage 
\subsection{Experiments on MNIST}
We demonstrate that our findings apply in high dimensional problems and practical settings. We trained a U-ViT network \citep{bao2023all} for $100$K iterations using the Adam optimizer with $1e-4$ learning rate. The hyperparameters for the network can be found in Table \ref{tab:model-config}.

\begin{table}[h!]
  \centering
  \caption{Model Configuration}
  \label{tab:model-config}
  \begin{tabular}{>{\bfseries}l l}
    \toprule
    Parameter & Value \\
    \midrule
    img\_size        & 28 \\
    in\_chans        & 1 \\
    patch\_size      & 2 \\
    embed\_dim       & 512 \\
    depth            & 12 \\
    num\_heads       & 8 \\
    mlp\_ratio       & 4 \\
    qkv\_bias        & False \\
    mlp\_time\_embed & False \\
    labels\_dim      & 11 \\
    \bottomrule
  \end{tabular}
\end{table}

We demonstrate that guidance eliminates the region of intersection. To do so, we consider two examples. In the first one, we sample digit $8$ without guidance, with guidance, and using the class of digit $3$ as the guiding distribution. The results show that samples of $8$ that resemble $3$ disappear when using guidance, this effect is even more pronounced when we use $3$ to guide the generation of $8$. This can be observed in Figure \ref{fig:8conditioned3}. To further demonstrate our point, we repeat the same experiment using $7$ conditioned on $1$. With the numerical results in Figure \ref{fig:8conditioned3} and Figure \ref{fig:7conditioned1}, it becomes clear that even in practical settings, the theoretical results move on to higher dimensions.

\begin{figure}[H]
    \centering
    
    \begin{subfigure}[t]{0.3\textwidth}
        \includegraphics[width=\linewidth]{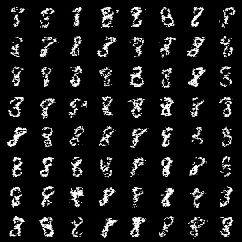}
        \caption{Generating samples of $8$ with no guidance}
    \end{subfigure}
    \hfill
    \begin{subfigure}[t]{0.3\textwidth}
        \includegraphics[width=\linewidth]{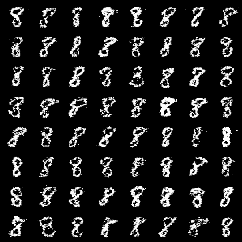}
        \caption{Generating samples of $8$ using guidance with $w=1$}
    \end{subfigure}
    \hfill
    \begin{subfigure}[t]{0.3\textwidth}
        \includegraphics[width=\linewidth]{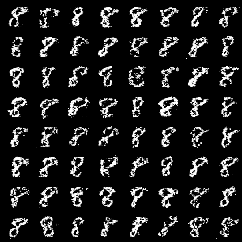}
        \caption{Generating samples of $8$ but using the class of number $3$ as the guiding distribution using $w=1$}
    \end{subfigure}
    \caption{Applying guidance reduces the are of intersection between classes. Notice how number $8$'s that look similar to a $3$ disappear.}
    \label{fig:8conditioned3}
\end{figure}

\begin{figure}[H]
    \centering
    
    \begin{subfigure}[t]{0.3\textwidth}
        \includegraphics[width=\linewidth]{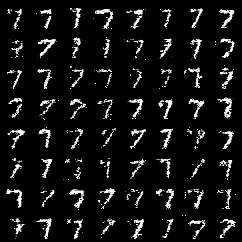}
        \caption{Generating samples of $7$ with no guidance}
    \end{subfigure}
    \hfill
    \begin{subfigure}[t]{0.3\textwidth}
        \includegraphics[width=\linewidth]{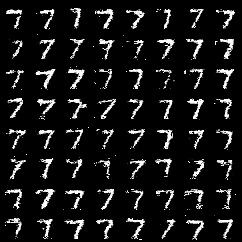}
        \caption{Generating samples of $7$ using guidance with $w=2$}
    \end{subfigure}
    \hfill
    \begin{subfigure}[t]{0.3\textwidth}
        \includegraphics[width=\linewidth]{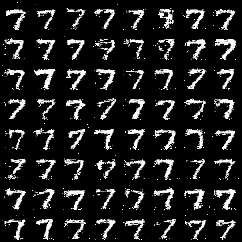}
        \caption{Generating samples of $7$ but using the class of number $1$ as the guiding distribution using $w=2$}
    \end{subfigure}
    \caption{Applying guidance reduces the area of intersection between classes. Notice how number $7$'s that resemble a $1$ disappear.}
    \label{fig:7conditioned1}
    \vspace{-.6cm}
\end{figure}

\end{document}